\documentclass[final,12pt]{colt2022} 

\usepackage{amsfonts}
\usepackage{comment}
\usepackage{mathrsfs}
\usepackage{booktabs}       
\usepackage{multirow}
\usepackage{tablefootnote}

\newcommand{\aaa}{\mathfrak{a}}

\newcommand{\al}{\mathbf{A}}
\newcommand{\B}{\mathcal B}
\newcommand{\Bg}{\B_{\nabla\ell}}
\newcommand{\Bgh}{\hat{\B}_{\nabla\ell}}
\newcommand{\Bl}{{\B}_\ell}
\newcommand{\Blh}{\hat{\B}_\ell}
\newcommand{\Btl}{\B_{\tilde\ell}}
\newcommand{\Bcmi}{\B_{\text{CMI}}}
\newcommand{\Btlh}{\hat{\B}_{\tilde\ell}}

\newcommand{\dd}{\mathrm{d}}
\newcommand{\df}{\mathfrak{D}}

\newcommand{\kk}{{k'}}
\newcommand{\KL}{\mathrm{K\!L}}
\newcommand{\E}{\mathbb{E}}
\newcommand{\Ell}{\mathscr{L}}
\newcommand{\ee}{\varepsilon}

\newcommand{\G}{\mathcal G}
\newcommand{\indep}{\perp \!\!\! \perp}
\newcommand{\LI}{\mathrm{L}}
\newcommand{\N}{\mathcal N}
\newcommand{\Nn}{\mathbb N}
\newcommand{\one}{\mathbf{1}}

\newcommand{\Pp}{\mathbb{P}}
\newcommand{\PP}{\mu}
\newcommand{\PR}{\mathscr{P}}

\newcommand{\QQ}{\nu}
\newcommand{\R}{\mathbb{R}}
\newcommand{\Rr}{\mathcal{R}}
\newcommand{\Sc}{\mathcal{S}}
\newcommand{\Supp}{\mathrm{Supp}}

\newcommand{\TV}{\mathrm{T\!V}}

\newcommand{\uuu}{\mathfrak{u}}
\newcommand{\U}{\mathcal{U}}

\newcommand{\vvv}{\mathfrak{v}}
\newcommand{\V}{\mathbb{V}}
\newcommand{\W}{\mathcal{W}}
\newcommand{\Wass}{\mathfrak{W}}
\newcommand{\Wassun}{\Wass_{1\!\text{D}}}
\newcommand{\X}{\mathcal{X}}

\newcommand{\Z}{\mathcal{Z}}

\DeclareMathOperator*{\erf}{erf}
\DeclareMathOperator*{\Id}{Id}

\newtheorem{assumptions}{Assumptions}

\newtheorem{manpropin}{Proposition}
\newenvironment{manualprop}[1]{%
  \renewcommand\themanpropin{#1}%
  \manpropin
}{\endmanpropin}

\newtheorem{manlemmin}{Lemma}
\newenvironment{manuallemma}[1]{%
  \renewcommand\themanlemmin{#1}%
  \manlemmin
}{\endmanlemmin}

\newenvironment{proofsketch}{%
  \renewcommand{\proofname}{Proof's sketch}\proof}{\endproof}

\title[Chained Generalisation Bounds]{Chained Generalisation Bounds}
\usepackage{times}
 \coltauthor{\Name{Eugenio Clerico} \Email{clerico@stats.ox.ac.uk}\\
  \Name{Amitis Shidani} \Email{shidani@stats.ox.ac.uk}\\
  \Name{George Deligiannidis} \Email{deligian@stats.ox.ac.uk}\\
  \Name{Arnaud Doucet} \Email{doucet@stats.ox.ac.uk}\\
  \addr Department of Statistics, University of Oxford, UK.
}

\begin{document}

\maketitle

\begin{abstract}%
This work discusses how to derive upper bounds for the expected generalisation error of supervised learning algorithms by means of the chaining technique. By developing a general theoretical framework, we establish a duality between generalisation bounds based on the regularity of the loss function, and their chained counterparts, which can be obtained by lifting the regularity assumption from the loss onto its gradient. This allows us to re-derive the chaining mutual information bound from the literature, and to obtain novel chained information-theoretic generalisation bounds, based on the Wasserstein distance and other probability metrics. We show on some toy examples that the chained generalisation bound can be significantly tighter than its standard counterpart, particularly when the distribution of the hypotheses selected by the algorithm is very concentrated. 
\end{abstract}


\section{Introduction}
In the supervised setting, a learning algorithm is a procedure that takes a training dataset as input and returns a hypothesis (\textit{e.g.}, regression coefficients, weights of a neural network, etc.). Ideally, the learned hypothesis should perform well on both the input dataset and new data, which were not used for the training. There is hence interest in providing generalisation bounds, namely upper bounds on the algorithm's gap in performance for seen and unseen instances.

The first generalisation bounds were based on characterisations of the hypothesis space's complexity, such as the VC dimension or the Rademacher complexity \citep{Bousquet2004, vapnik00,shalevBook2014}. However, due to their algorithm-independent nature, these bounds must hold even for the worst algorithm on the given hypothesis space. Consequently, they are often inadequate for modern over-parameterised neural networks, with the complexity measure usually scaling exponentially with the architecture's depth \citep{neuralnetAnthony, zhang2017understanding, belkin2018understand}.  

To address this issue, recent approaches aim at providing algorithm-dependent generalisation bounds. The underlying intuition is that if the output hypothesis is less dependent on the input dataset, it would be less prone to overfitting, and so generalises better. Among the results building on this idea, there are bounds based on uniform stability \citep{bousquet2002} and differential privacy \citep{DworkDP2014}, PAC-Bayesian bounds \citep{guedj2019primer, McAllester98somepac-bayesian, mcallester}, and information-theoretic bounds. 

In this paper, we shall mainly focus on the information-theoretic framework, where the learning algorithm is seen as a noisy channel connecting the input dataset and the chosen hypothesis. \cite{russo2019much} and \cite{Xu2017InformationtheoreticAO} were the first to introduce this approach. They upper-bounded the expected generalisation error via the Mutual Information (MI) between the input sample and the learnt hypothesis. This bound is simple and can be applied to a broad class of learning algorithms. However, a major drawback is that it becomes infinite if the choice of the hypothesis is deterministic in the input. Motivated by this problem, several strategies have been proposed. 

\cite{Bu2019} gave an individual-sample MI bound, while \cite{steinke2020} introduced a conditional version of the MI, which is always finite. \cite{Galvez2020OnRS}, \cite{haghifam2020}, and \cite{hellstrom2020} extended and merged these results. Alternatively, different measures of algorithmic stability can replace the MI: \cite{lopez2018wass}, \cite{Wang2019Wass}, and \cite{borja2021tighter} proposed bounds based on the Wasserstein distance, while others focused on total variation, $f$-divergences, and lautum information \citep{Wang2019Wass, borja2021tighter, esposito2020generalization, lautuminf}.

Adopting a different perspective, \cite{asadi2018chaining} observed that several information-theoretic bounds fail to exploit the dependencies between hypotheses. They hence proposed to combine the original MI bound with the chaining method, a powerful tool from high dimensional probability originally aimed at upper-bounding the expected supremum of random processes. First introduced by Kolmogorov (see \cite{vanHandel}), the chaining technique has been successfully extended and developed \citep{DUDLEY, Talagrand2005MMbook, Talagrand2014MMbook}. In their Chaining Mutual Information (CMI) bound, \cite{asadi2018chaining} take finer and finer discretisations of the hypothesis space and rewrite the generalisation error as a telescopic sum, whose terms can be controlled by exploiting the dependencies between the hypotheses. Subsequently, \cite{Asadi2020ChainingMC} adapted the CMI technique to the architecture of deep neural nets, while \cite{zhou2022stochastic} introduced bounds based on a stochastic version of chaining. However, it is worth mentioning that previous works had already applied the chaining method to algorithm-dependent bounds. For instance, \cite{genericPACBayes} combined the generic chaining from \cite{Talagrand2005MMbook} with the PAC-Bayesian approach.

As a final comment, it must be noted that the generalisation bounds from the information-theoretic literature are hard to evaluate in practice, involving expectations with respect to the unknown sample distribution. Nevertheless, they provide useful intuition on the mechanism of the learning process and, as a result, they represent a very active research area. Moreover, recent works have built on them to derive computable analytical bounds for specific algorithms, such as Langevin dynamics, stochastic gradient Langevin dynamics, and stochastic gradient descent \citep{Bu2019, negrea19, haghifam2020, Galvez2020OnRS, neu21}.

\subsection{Our contributions}
The CMI bound is an interesting multi-scale reformulation of the original MI result by \cite{russo2019much}. However, in the information-theoretic literature on generalisation bounds, the chaining method has been coupled only with the MI \citep{asadi2018chaining,Asadi2020ChainingMC,zhou2022stochastic}. Two questions then naturally arise. \textit{Is it possible to derive chained versions of other kinds of generalisation bounds? Can these chained bounds be tighter than their original counterparts?} 

In the present work, we establish a duality that reads as follows. \textit{Each bound, based on (a certain notion of) regularity of the loss function, corresponds to a chained bound that can be obtained by lifting the regularity condition from the loss to its gradient.} To make sense of this, we first introduce a general framework, standardising the main step in the proof of several information-theoretic bounds from the literature. We then discuss how to extend this framework leveraging the chaining technique, and we provide a simple method to derive novel chained generalisation bounds. We show indeed that in our framework each unchained bound corresponds to a chained one (see Theorems \ref{thm:genstd} and \ref{thm:gench}), in a way reflecting the connection between the MI and CMI results.

The framework introduced in this work encompasses several information-theoretic \textit{backward-channel}\footnote{In the information-theoretic literature, the \textit{forward-channel} connects the sample to the hypothesis, while the \textit{backward-channel} goes the other way. Chaining on the hypotheses combines naturally with the \textit{backward-channel}.} bounds, and allows us to derive their chained counterparts. However, due to space limitations, many explicit results are deferred to Appendix \ref{app:bounds} (see Table \ref{table:boundsapp}) and in the main text we focus on four bounds to concretely illustrate how our framework works: the MI bound from \cite{russo2019much} and the CMI bound from \cite{asadi2018chaining} serve as a motivation for our general result, while as an application of our framework we derive a novel Wasserstein bound (see Proposition \ref{prop:Wassch}), which is the chained counterpart of a bound from \cite{lopez2018wass}.

Moreover, we discuss some possible extensions of our work. On the one hand, our information-theoretic framework can be restated with weaker regularity assumptions on both the loss and the hypothesis space. On the other hand, we present an additional bound that does not fit our theoretical framework but can still be derived using essentially the same technical machinery. It is a chained PAC-Bayesian generalisation result, which has the interesting features of being finite even for deterministic algorithms and not requiring the loss to be bounded by a small constant.

As a final remark, there is no generic answer on whether the chained bounds are tighter than their unchained counterparts. However, the chaining technique turns out to be particularly effective when the hypotheses' distribution is very concentrated. In fact, many of the standard bounds do not exploit this feature, the most pathological case being the MI bound, which can even be infinite. In contrast, the chained bounds can be significantly tighter, intrinsically leveraging the dependencies between different hypotheses. We illustrate this phenomenon through some simple toy examples.
\section{Preliminaries}\label{sec:prel}
Let the input space $(\X, d_\X)$ be a separable complete metric space, and $\Sigma_\X$ the corresponding Borel $\sigma$-algebra. We define $\Sc=\X^m$ and consider a metric $d_\Sc$ inducing the product $\sigma$-algebra $\Sigma_\Sc=\Sigma_\X^{\otimes m}$. We denote the training dataset as $s=\{x_1,\dots,x_m\}\in\Sc$. Let $\Pp_X$ be a probability measure on $\X$ and $X$ a random variable with law $\Pp_X$. $S=\{X_1, \dots, X_m\}\in\Sc$ denotes the random training sample, with law $\Pp_S$. We will always assume that the marginal $\Pp_{X_i}=\Pp_X$, for each index $i$. This is of course the case if the $X_i$ are i.i.d.\ ($\Pp_S=\Pp_X^{\otimes m}$). We will suppose that the hypothesis space $\W$ is a closed subset of $\R^d$, endowed with its Borel $\sigma$-algebra $\Sigma_\W$. A learning algorithm consists in a Markov kernel that maps each $s\in\Sc$ to a probability measure $\Pp_{W|S=s}$ on $\W$. In turn, this defines a joint probability $\Pp_{W, S}$ on $\W\times\Sc$. We denote as $\Pp_W$ and $\Pp_S$ the marginal distributions of $\Pp_{W, S}$, and we let $s\mapsto\Pp_{W|S=s}$ and $w\mapsto\Pp_{S|W=w}$ be regular conditional probabilities\footnote{The existence of these is ensured by the fact that $\Sc$ and $\W$ are Polish spaces, cf.\ Theorem 10.2.2 in \cite{dudley_2002}.}. 

In the supervised framework, the goal is to approximate a map $x\mapsto f_\star(x)$ by making use of the information contained in the training sample $s$ (the value of $f^\star(x_i)$ is known for each $x_i\in s$). Each hypothesis $w$ represents a parameterised mapping $x\mapsto f_w(x)$, and the training process consists in tuning $w$, so as to approximate $f^\star$. The loss $\ell:\W\times\X\to\R$, allows to assess how far each $f_w(x)$ is from $f^\star(x)$. We will always assume that $\ell(w, \cdot)\in L^1(\Pp_X)$. Define the empirical and the true loss 
$$\Ell_s(w) = \frac{1}{m}\sum_{i=1}^m\ell(w, x_i)\,;\qquad\qquad \Ell_\X(w) = \E_{\Pp_X}[\ell(w, X)]\,. $$
We call generalisation error the difference $g_s(w) = \Ell_\X(w)- \Ell_s(w)$. In this work, we are essentially interested in upper-bounding its expected value $\G = \E_{\Pp_{W, S}}[g_S(W)]$. 

The equality $\E_{\Pp_{W, S}}[\Ell_\X(W)] = \E_{\Pp_{W\otimes S}}[\Ell_S(W)]$, where $\Pp_{W\otimes S}=\Pp_W\otimes\Pp_S$, follows from $\Ell_\X(w)=\E_{\Pp_S}[\Ell_S(w)]$ and is the starting point of several information-theoretic bounds. Indeed,
$$\G = \E_{\Pp_{W\otimes S}}[\Ell_S(W)] - \E_{\Pp_{W, S}}[\Ell_S(W)]\,$$
can be upper-bounded in terms of how ``far apart'' $\Pp_{W, S}$ and $\Pp_{W\otimes S}$ are.

\subsection{Further notation and conventions}
The following notation will be used throughout the rest of the paper. $(\Z, \Sigma_\Z)$ denotes a generic separable complete metric space, endowed with the Borel $\sigma$-algebra induced by its metric $d_\Z$. We endow $\PR$, the space of all the probability measures on $(\Z, \Sigma_\Z)$, with the topology of the weak convergence, and we denote the corresponding Borel $\sigma$-algebra as $\Sigma_\PR$. For two coupled random variables $Z, Z'$ on $\Z$, we write  $\Pp_{Z\otimes Z'}$ for the independent coupling $\Pp_Z\otimes \Pp_{Z'}$. For $v,v'\in\R^q$ (for a generic $q\in\mathbb N$) we write $\|v\|$ and $v\cdot v'$ for the Euclidean norm and the standard dot product in $\R^q$ respectively. For a random vector $V\in\R^q$, we write that $V\in L^1$ if $\E_{\Pp_V}[\|V\|]<+\infty$. When we need to specify the integrability of $V$ with respect to a particular law $\mu$, we explicitly write $V\in L^1(\mu)$, that is $\E_\mu[\|V\|]<+\infty$. Finally, $\xi$ denotes an arbitrary non-negative real number.

\section{General framework}\label{sec:fram}
\subsection{Bounds based on the regularity of the loss}\label{sec:framstd}
Both the standard MI and Wasserstein bounds from \cite{russo2019much} and \cite{lopez2018wass} (see Propositions \ref{prop:MI} and \ref{prop:Wass} in Section \ref{sec:ex} for the explicit statements) build on some regularity condition on the dependence of $\ell$ in $x$, holding uniformly on $\W$. As this is a common assumption for various \textit{backward-channel} bounds in the literature, we will now introduce a unified abstract framework, which allows us to re-derive several information-theoretic bounds, such as many of those based on MI, Wasserstein distances, and other probability metrics. Due to the limited space, in the main text we only give a few concrete applications of our framework (see Section \ref{sec:ex}). A wide range of additional explicit examples, listed in Table \ref{table:boundsapp}, can be found in Appendix \ref{app:bounds}. 
\begin{definition}[$\df$-regularity]\label{def:Dreg} Let $\df$ be a measurable\footnote{The measurability wrt $\Sigma_\PR$ is a technical assumption that is required in order to ensure that expressions, such as $\int_\W\df(\Pp_S, \Pp_{S|W=w})\dd\Pp_W(w)$ in Theorem \ref{thm:genstd}, make sense. The reader can be assured that it holds whenever $\df$ is a measurable function of an $f$-divergence, or the Wasserstein distance. We refer to Appendix \ref{app:tech} for more details.} map $\PR\times\PR\to[0,+\infty]$. Fix $\mu\in\PR$ and $\xi\geq 0$. We say that $f:\Z\to\R$ has regularity $\Rr_\df(\xi)$, with respect to $\mu$, if $f\in L^1(\mu)$ and, for every $\nu\in\PR$ such that $\Supp(\nu)\subseteq\Supp(\mu)$ and $f\in L^1(\nu)$,
$$|\E_\mu[f(Z)]-\E_\nu[f(Z)]|\leq \xi\,\df(\mu, \nu)\,.$$
We can extend the definition to functions taking values in $\R^q$, for $q>1$. We say that $F:\Z\to\R^q$ has regularity $\Rr_\df(\xi)$ (wrt $\mu$) if $z\mapsto v\cdot F(w)$ has regularity $\Rr_\df(\xi\|v\|)$ (wrt $\mu$), for all $v\in\R^q$.
\end{definition}
The concept of $\df$-regularity is intrinsically connected to the choice of the measure $\mu\in\PR$, in the sense that $f$ might be $\Rr_\df(\xi)$ regular with respect to $\mu$, but not with respect to some other $\mu'\in\PR$. For two simple concrete examples of $\df$-regularity, we refer to Lemma \ref{lemma:dfreg}, in Section \ref{sec:ex}. 


Now, let $\Z=\Sc$ and recall that $\W$ is a closed subset of $\R^d$, with Borel $\sigma$-algebra $\Sigma_\W$. On the product space $(\W\times\Sc, \Sigma_\W\otimes\Sigma_\Sc)$, we consider a probability measure $\Pp_{W, S}$, with marginals $\Pp_W$ and $\Pp_S$. Recall that since $\Sc$ is a Polish space, $w\mapsto \Pp_{S|W=w}$ defines a regular conditional probability (cf.\ Theorem 10.2.2 in \cite{dudley_2002}). The next result, which follows easily from the definition of regularity, is a powerful tool to derive generalisation bounds. 
\begin{theorem}\label{thm:genstd}
Assume that $s\mapsto \Ell_s(w)$ has regularity $\Rr_\df(\xi)$ wrt $\Pp_S$, $\forall w\in\W$. Then we have
$$|\G|=|\E_{\Pp_{W\otimes S}}[\Ell_S(W)]-\E_{\Pp_{W, S}}[\Ell_S(W)]|\leq \xi\,\E_{\Pp_W}[\df(\Pp_S, \Pp_{S|W})]\,,$$
where $\E_{\Pp_W}[\df(\Pp_S, \Pp_{S|W})]=\int_\W\df(\Pp_S, \Pp_{S|W=w})\,\dd\Pp_W(w)$.\footnote{Note that $w\mapsto\df(\Pp_S, \Pp_{S|W=w})$ is measurable, since both $w\mapsto \Pp_{S|W=w}$ and $(\mu, \nu)\mapsto\df(\mu, \nu)$ are Borel measurable (see Appendix \ref{app:tech}).}
\end{theorem}

By specialising the concept of $\df$-regularity, we can leverage the framework introduced so far and obtain generalisation bounds based on various probability divergences (cf.\ Table \ref{table:boundsapp}). Moreover, individual-sample bounds such as those from \cite{Bu2019} can fit in our framework, as well as bounds based on the random sub-sampling from a super-sample, in the same spirit of the conditional MI bound from \cite{steinke2020}. We refer the reader to Appendix \ref{app:bounds} for a more detailed discussion of these results.
\subsection{Bounds based on the regularity of the loss's gradient}\label{sec:framch}
The bounds based on the chaining technique, such as the CMI bound from \cite{asadi2018chaining} (see Proposition \ref{prop:MIchain0} in Section \ref{sec:ex}), do not fit naturally in the framework presented so far. We are thus motivated to find an alternative setting that naturally gives rise to chained bounds, thus establishing new generalisation results.

As a starting point, let us notice that the main idea behind the CMI bound is to lift the regularity assumption from  $x\mapsto\ell(w, x)$ onto $x\mapsto (\ell(w, x)-\ell(w',x))$. A natural guess is that this approach could provide chained bounds also in our general framework, and this is indeed the case (cf.\ Theorem \ref{thm:genchapp} in Appendix \ref{app:extfram}). However, if $\ell$ is regular enough we can focus on the gradient $\nabla_w\ell(w, x)$ instead. Since this leads to more intuitive and compact statements, we chose to consider this case in the main text. 
\begin{assumptions}\label{ass}\\
$\bullet$ \ The set $\W\subset\R^d$ is convex, compact, and with non-empty interior.\\
$\bullet$ \ The function $w\mapsto\ell(w,x)$ is of class $C^1$ on $\W$, $\Pp_X$-a.s.\\
$\bullet$ \ We have $\sup_{(w,x)\in\W\times\X}|\ell(w, x)|<+\infty$ and $\sup_{(w,x)\in\W\times\X}\|\nabla_w\ell(w, x)\|<+\infty$.
\end{assumptions}
Let us stress once more that the above assumptions are not necessary in order to obtain the duality chained-unchained bounds. In Appendix \ref{app:extfram} we discuss a more general setting: $\W$ can be non-convex and with empty interior, $\ell$ continuous on $\W$ ($\Pp_X$-a.s.) and only bounded in expectation.

The chained bounds involve a sequence of finer and finer discretisations of the hypotheses' space, which can be formalised as follows.  
\begin{definition}[Nets and refining sequences of nets]
Given $\ee>0$, we define an $\ee$-projection on $\W$ as a measurable mapping $\pi:\W\to\W$ such that $\pi(\W)$ has finitely many elements and, for all $w\in\W$, $\|\pi(w)-w\|\leq \ee$. The image $\pi(\W)$ is called an $\ee$-net on $\W$.\\
Consider a positive, vanishing, decreasing sequence $\{\ee_k\}_{n\in\Nn}$, and assume that there is a $w_0\in\W$ such that $\|w-w_0\|\leq\ee_0$ for each $w\in\W$. We call $\{\pi_k(\W)\}_{n\in\Nn}$ an $\{\ee_k\}$-refining sequence of nets if $\pi_0(\W) = \{w_0\}$ and, for all $k\geq 1$, we have that $\pi_k$ is a $\ee_k$-projection and $\pi_{k-1}\circ\pi_k=\pi_{k-1}$.
\end{definition}
To simplify the notation, for all $w\in\W$ we let $w_k = \pi_k(w)$, and similarly $W_k = \pi_k(W)$ and $\W_k=\pi_k(\W)$. Note that for all $k$, $w_{k'}$ is determined by $w_k$ whenever $k'\leq k$, as $w_{k'}=\pi_{k'}(w_k)$. Moreover, for all $k\geq 1$, $\|w_k- w_{k-1}\| = \|w_k-\pi_{k-1}(w_k)\|\leq\ee_{k-1}$.

The next theorem is the main result of this work. Together with Theorem \ref{thm:genstd}, it establishes the duality between chained and unchained generalisation bounds, which can essentially be obtained by lifting the regularity from the loss onto its gradient. 
\begin{theorem}\label{thm:gench}
Assume \ref{ass} and that $s\mapsto \nabla_w \Ell_s(w)$ has regularity $\Rr_\df(\xi)$ wrt $\Pp_S$, $\forall w\in\W$. Then, for any $\{\ee_k\}$-refining sequence of nets on $\W$,
$$|\G|=|\E_{\Pp_{W\otimes S}}[\Ell_S(W)]-\E_{\Pp_{W, S}}[\Ell_S(W)]|\leq \xi\sum_{k=1}^\infty \ee_{k-1}\E_{\Pp_W}[\df(\Pp_S, \Pp_{S|W_k})]\,,$$
where $\E_{\Pp_W}[\df(\Pp_S, \Pp_{S|W_k})]=\int_\W\df\big(\Pp_S, \Pp_{S|W\in \pi_k^{-1}(w)}\big)\,\dd\Pp_W(w)$.
\end{theorem}
\begin{proofsketch}
Here is a sketch of the proof; see Appendix \ref{app:proofch} for the details. Following the standard chaining argument, we control $\Ell_s(w)$ by the telescopic sum $\sum_{k\geq 1} (\Ell_s(w_k)-\Ell_s(w_{k-1}))$. The upper bound will then follow from the fact that the $\Rr_\df(\xi)$ regularity of $s\mapsto\nabla_w \Ell_s(w)$ implies the $\Rr_{\df}(\ee_{k-1}\xi)$ regularity of $w\mapsto (\Ell_s(w_k)-\Ell_s(w_{k-1}))$.
\end{proofsketch}

Both Theorem \ref{thm:genstd} and \ref{thm:gench} are stated under uniform regularity conditions, in the sense that the value of the regularity's parameter $\xi$ has to be the same for all $w\in\W$. However, we can still achieve generalisation bounds under less strict assumptions. In Appendix \ref{app:extdf} we discuss the case of a measurable map $w\mapsto\xi_w$, such that, for some $p\in[1,+\infty]$, $\xi_W$ is bounded in $L^p(\Pp_W)$ (or $L^p(\Pp_{W_k})$, $\forall k\in\Nn$). Note that choosing $p=+\infty$ brings back the uniform condition.

In a similar spirit, one might try to relax the definition of $\ee$-net, by mimicking the stochastic chaining idea from \cite{zhou2022stochastic}. We defer this approach to future work.

\section{A few concrete examples: MI and Wasserstein bounds}\label{sec:ex}
In the current section we give a few concrete applications of the abstract framework that we have presented so far. We recover some simple generalisation bounds from the literature and establish a novel chained bound, based on the Wasserstein distance. 

First, we need to state a few standard definitions. 
\begin{definition}[Subgaussianity]\label{def:SG} A real random variable $Z\in L^1$ is $\xi$-SubGaussian ($\xi$-SG) if
$$\log\E_{\Pp_Z}[e^{\lambda Z}]\leq \lambda\E_{\Pp_Z}[Z] + \tfrac{\xi^2\lambda^2}{2}\,,\qquad\forall\lambda>0\,.$$
A random vector $V\in\R^q$ is $\xi$-SG if, for all $v\in\R^q$, $V\cdot v$ is $(\|v\|\xi)$-SG. Finally, a stochastic process $\{F_w\}_{w\in\W}$ is $\xi$-SG if, for every pair $(w, w')\in\W^2$, $F_w-F_{w'}$ is a $(\|w-w'\|\xi)$-SG random variable.
\end{definition}
Note that any bounded random variable $Z\in[a, b]$ is $\frac{b-a}{2}$-SG. Moreover, a normally distributed random variable $Z\sim\N(0, \xi)$ is $\xi$-SG.
\begin{definition}[Lipschitzianity] A function $f:\Z\to\R^q$ is $\xi$-Lipschitz on $\Z$ if, for all $z, z'\in\Z$, $$\|f(z)-f(z')\|\leq \xi d_\Z(z,z')\,.$$
\end{definition}
\begin{definition}[Kullback--Leibler divergence and mutual information] Let $\PP$ and $\QQ$ be two probability measures on $\Z$. We define the Kullback--Leibler divergence 
$$\KL(\QQ\|\PP) = \begin{cases}\E_\QQ[\log{\dd\QQ}/{\dd\PP}] &\text{if $\QQ\ll\PP$;}\\+\infty &\text{otherwise.}\end{cases}$$
For two coupled random variables $Z, Z'$, the Mutual Information (MI) is defined as $$I(Z;Z') = \KL(\Pp_{Z, Z'}\|\Pp_{Z\otimes Z'}).\,$$
\end{definition}
The $\KL$ divergence is non-negative, with $\KL(\QQ\|\PP)=0$ if, and only if, $\PP=\QQ$. Similarly the MI is always non-negative, and null if, and only if, $Z\indep Z'$.
\begin{definition}[Wasserstein distance] Given two distributions $\PP$ and $\QQ$ on $\Z$ and fixed $p\geq 1$, their $p$-Wasserstein distance $\Wass_p$ is defined as
$$\Wass_p(\PP, \QQ) = \inf_{\pi\in\Pi[\PP, \QQ]}\E_{(Z,Z')\sim\pi}[d_\Z(Z,Z')^p]^{1/p}\,,$$
where $\Pi[\PP, \QQ]$ is the set of all probability measures, on $(\Z^2, \Sigma_\Z\otimes\Sigma_\Z)$, with marginals $\PP$ and $\QQ$.
\end{definition}
It can be shown that for $p>p'$ we have $\Wass_p(\PP, \QQ)\geq \Wass_{p'}(\PP, \QQ)$, so that in particular $\Wass_1$ is the weakest. For this reason, henceforth we will focus on $\Wass_1$, which we will simply denote $\Wass$. 

Using the concepts that we have just introduced, we can give two simple and concrete examples of $\df$-regularity.
\begin{lemma}\label{lemma:dfreg}
Let $\df_1:(\mu, \nu) \mapsto \sqrt{2\KL(\nu\|\mu)}$ and $\df_2:(\mu, \nu) \mapsto \Wass(\mu, \nu)$. Consider a measurable map $f:\Z\to\R^q$ (with $q\geq 1$). If $f(Z)$ is $\xi$-SG for $Z\sim\mu\in\PR$, then $f$ has regularity $\Rr_{\df_1}(\xi)$ wrt $\mu$. If $f$ is $\xi$-Lipschitz on $\Z$, then $f$ has regularity $\Rr_{\df_2}(\xi)$, wrt any $\mu\in\PR$ such that $f\in L^1(\mu)$.
\end{lemma}
\subsection{Standard MI and Wasserstein bounds}
We state two simple generalisation bounds that were previously mentioned in the introduction. The proofs that we give leverage the abstract framework of Section \ref{sec:framstd}. The first result \citep{russo2019much, Xu2017InformationtheoreticAO} is an upperbound on $\G$ based on the mutual information between $W$ and $S$.
\begin{proposition}[Standard MI bound]\label{prop:MI}
Let $\Pp_S=\Pp_X^{\otimes m}$. If $\ell(w, X)$ is $\xi$-SG, $\forall w\in\W$, then
$$|\G|\leq \xi\sqrt{\frac{2I(W;S)}{m}}\,.$$
\end{proposition}
\begin{proof}
First, since $\Pp_S=\Pp_X^{\otimes m}$, $\Ell_S(w)$ is the average of $m$ independent $\xi$-SG random variables, so it is $(\xi/\sqrt m)$-SG. In particular, with $\df: (\mu, \nu)\mapsto \sqrt{2\KL(\nu\|\mu)}$, Lemma \ref{lemma:dfreg} shows that $s\mapsto \Ell_s(w)$ has regularity $\Rr_\df(\xi/\sqrt m)$, $\forall w\in\W$. We conclude by Theorem \ref{thm:genstd} and Jensen's inequality.
\end{proof}
The next bound is from \cite{lopez2018wass} and is close in spirit to the previous one, as again it tries to measure how much information about $S$ is enclosed in $W$. However, now the MI is replaced by an expected Wasserstein distance. In order to get an explicit dependence on  $1/\sqrt m$, we assume that the metric $d_\Sc$ on $\Sc$ is related to the one on $\X$ via 
\begin{equation}\label{eq:defmetric}
    d_\Sc(s, s') = \left(\sum_{i=1}^m d_\X(x_i, x'_i)^2\right)^{1/2}\,,
\end{equation}
where $s=\{x_1,\dots, x_m\}$ and $s'=\{x'_1, \dots, x'_m\}$. Note that we do not need $\Pp_S=\Pp_X^{\otimes m}$.
\begin{proposition}[Standard Wasserstein bound]\label{prop:Wass}
Suppose that $d_\X$ and $d_\Sc$ are related by \eqref{eq:defmetric}. If, $\forall w\in\W$, $x\mapsto\ell(w, x)$ is $\xi$-Lipschitz on $\X$, then
$$|\G|\leq \frac{\xi}{\sqrt{m}}\,\E_{\Pp_W}[\Wass(\Pp_{S}, \Pp_{S|W})]\,.$$
\end{proposition}
\begin{proof}
First notice that
$$d_\Sc(s, s') = \left(\sum_{i=1}^m d_\X(x_i, x'_i)^2\right)^{1/2} \geq \frac{1}{\sqrt m}\sum_{i=1}^m d_\X(x_i, x'_i)\,,$$
where we used the Cauchy-Schwartz inequality. Consequently, $s\mapsto \Ell_s(w)$ is $(\xi/\sqrt m)$-Lipschitz $\forall w\in\W$. In particular, if we let $\df: (\mu, \nu)\mapsto \Wass(\mu,\nu)$, then $s\mapsto \Ell_s(w)$ has regularity $\Rr_\df(\xi/\sqrt m)$ by Lemma \ref{lemma:dfreg}. We conclude by Theorem \ref{thm:genstd}.
\end{proof}
\subsection{Chained MI and Wasserstein bounds}
As we mentioned in the introduction, one of the main issues with the standard MI bound is that it can easily be vacuous, as it is the case when the learning algorithm defines a deterministic map $\Sc\to\W$. To address this issue, \cite{asadi2018chaining} proposed to build on the chaining technique and established the bound below. The setting here is quite different from the one of the standard MI bound, as the process's subgaussianity takes into account the dependencies between different hypotheses. Letting $\{\ee_k\}_{k\in\Nn}$ be a vanishing decreasing positive sequence, we consider an $\{\ee_k\}$-refining sequence of nets $\{\W_k\}_{k\in\Nn}=\{\pi_k(\W)\}_{k\in\Nn}$ and recall that $W_k=\pi_k(W)$. 
\begin{proposition}[CMI bound]\label{prop:MIchain0}
Let $\Pp_S=\Pp_X^{\otimes m}$ and $\W$ be a compact set, with an $\{\ee_k\}$-refining sequence of nets defined on it. Suppose that $w\mapsto\ell(w, x)$ is continuous, for $\Pp_X$-almost every $x$,\footnote{Note that in \cite{asadi2018chaining} the result is stated under a weaker assumption of separability of the process. To avoid introducing further definitions and technicalities in the proofs, we decided to focus on the case of a.s.\ continuity.}\ and that $\{\ell(w, X)\}_{w\in\W}$ is a $\xi$-SG stochastic process. Then we have
$$|\G|\leq \frac{\xi}{\sqrt m}\sum_{k=1}^\infty\ee_{k-1}\sqrt{2I(W_k;S)}\,.$$
\end{proposition}
We provide a proof of Proposition \ref{prop:MIchain0} within the extended general framework of Appendix \ref{app:extfram}, while here we establish a similar result, under the more restrictive assumptions \ref{ass}.

Leveraging the machinery developed in Section \ref{sec:framch}, we can expect that lifting the subgaussianity from $\ell$ to $\nabla_w\ell$ we can find a chained version of the MI bound in Proposition \ref{prop:MI}. Perhaps unsurprisingly, we simply re-obtain the CMI bound of Proposition \ref{prop:MIchain0}.
\begin{proposition}\label{prop:MIchain}
Let $\Pp_S=\Pp_X^{\otimes m}$ and assume \ref{ass}. If $\nabla_w\ell(w, X)$ is $\xi$-SG, $\forall w\in\W$, we have that for any $\{\ee_k\}$-refining sequence of nets on $\W$
$$|\G|\leq \frac{\xi}{\sqrt m}\sum_{k=1}^\infty\ee_{k-1}\sqrt{2I(W_k;S)}\,.$$
\end{proposition}
\begin{proof}
As in the proof of Proposition \ref{prop:MI}, we have that $\nabla_w \Ell_S(w)$ is $(\xi/\sqrt m)$-SG, $\forall w\in\W$. In particular, by Lemma \ref{lemma:dfreg} we have that $s\mapsto \nabla_w \Ell_s(w)$ has regularity $\Rr_\df(\xi/\sqrt m)$, $\forall w\in\W$, where $\df:(\mu,\nu)\mapsto\sqrt{2\KL(\nu\|\mu)}$. Hence, we conclude by Theorem \ref{thm:gench} and Jensen's inequality.
\end{proof}
The next lemma shows that, under the assumptions \ref{ass}, Propositions \ref{prop:MIchain0} and \ref{prop:MIchain} are equivalent.
\begin{lemma}\label{lem:equivSGP}
Under the assumptions \ref{ass}, the stochastic process $(\ell(w, X))_{w\in \W}$ is $\xi$-SG if, and only if, $\nabla_w\ell(w, X)$ is a $\xi$-SG vector for all $w\in\W$.
\end{lemma}

Once again, the main point of the abstract framework presented so far is to underline a duality: to each bound based on the $\df$-regularity of the loss corresponds a chained bound based on the $\df$-regularity of its gradient. We can hence apply this idea to the standard Wasserstein bound of Proposition \ref{prop:Wass} and obtain its chained counterpart, which is a novel result. 
\begin{proposition}[Chained Wasserstein bound]\label{prop:Wassch}
Let $d_\X$ and $d_\Sc$ be related by \eqref{eq:defmetric}. Under the assumptions \ref{ass}, suppose that $x\mapsto\nabla_w\ell(w, x)$ is $\xi$-Lipschitz on $\X$, $\forall w\in\W$. Then, for any $\{\ee_k\}$-refining sequence of nets on $\W$,
$$|\G| \leq \frac{\xi}{\sqrt{m}}\sum_{k=1}^\infty \varepsilon_{k-1}\E_{\Pp_{W}}[\Wass(\Pp_{S},\Pp_{S|W_k})]\,.$$
\end{proposition}
\begin{proof}
Let $\df:(\mu,\nu)\mapsto\Wass(\mu,\nu)$. Proceeding as in the proof of Proposition \ref{prop:Wass}, we have that $\nabla_w \Ell_S(w)$ is $(\xi/\sqrt m)$-Lipschitz, $\forall w\in\W$. In particular, by Lemma \ref{lemma:dfreg}, $s\mapsto \nabla_w \Ell_s(w)$ has regularity $\Rr_\df(\xi/\sqrt m)$, $\forall w\in\W$. Hence, we conclude by Theorem \ref{thm:gench}.
\end{proof}

We conclude by recalling once more that, in our framework, any bound based on the regularity of $\ell$ gives rise to a chained bound. We refer to Table \ref{table:boundsapp} in the appendix for several explicit examples.

\section{A chained PAC-Bayesian bound}
The framework introduced in Section \ref{sec:fram} focuses on the \textit{backward-channel} information-theoretic setting. However, the chaining ideas behind Theorem \ref{thm:gench} can fit in a broader context. As an example, we discuss here a PAC-Bayesian result. Although \cite{genericPACBayes} have already combined the PAC-Bayesian approach with the chaining technique, their use of an auxiliary sample and of the average distance between nets makes their bounds conceptually different from ours.

The PAC-Bayesian bounds are algorithmic-dependent upper bounds on the expected generalisation error $\E_{\Pp_{W|S}}[g_S(W)]$ of stochastic classifiers \citep{McAllester98somepac-bayesian}, holding with high probability on the random draw of the training sample $S$ (see \cite{guedj2019primer} and \cite{alquier2021user} for recent introductory overviews). They share the same underlying idea with the information-theoretic bounds: the less $\Pp_{W|S}$ is dependent on $S$, the better the algorithm generalises. However, in the PAC-Bayesian setting we compare $\Pp_{W|S}$ not with the marginal $\Pp_W$, but rather with a fixed probability measure $\Pp^*_W$, which can be chosen arbitrarily but without making use of the training sample $S$.

We state here a very simple classical PAC-Bayesian result from \cite{catonipreprint}.
\begin{proposition}\label{prop:PAC}
Assume that $\ell$ is bounded in $[-\xi, \xi]$. Let $\Pp^*_W$ be a fixed probability measure on $\W$, chosen independently of $S$. Fix $\delta\in(0,1)$ and $\lambda>0$. Then, with probability $\Pp_S=\Pp_X^{\otimes m}$ larger than $1-\delta$ on the draw of $S$, we have
$$\E_{\Pp_{W|S}}[g_S(W)]\leq \frac{\xi}{\sqrt{2m}}\left(\lambda + \frac{\KL(\Pp_{W|S}\|\Pp^*_W)+\log\frac{1}{\delta}}{\lambda}\right)\,.$$
\end{proposition}
A chained version of the above can be obtained by lifting the boundedness hypothesis from $\ell$ to $\nabla_w\ell$. This is quite peculiar, as most PAC-Bayesian bounds hold for bounded loss functions $\ell\subseteq[-\xi, \xi]$.
\begin{proposition}\label{prop:chPAC}
Under the assumptions \ref{ass}, consider a $\{\ee_k\}$-refining sequence of nets on $\W$ and assume that $\nabla_w\ell$ is bounded in $[-\xi, \xi]$. Let $\Pp^*_W$ be a fixed probability measure on $\W$, chosen independently of $S$. Fix two sequences $\{\delta_k\}_{k\in\mathbb N}$ and $\{\lambda_k\}_{k\in\mathbb N}$, such that $\delta_k\in(0,1)$ and $\lambda_k>0$ for all $k$. Assume that $\sum_{k\in\Nn}\delta_k=\delta \in(0,1)$. Then, with probability $\Pp_S=\Pp_X^{\otimes m}$ larger than $1-\delta$ on the draw of $S$, we have
\begin{align*}
    \E_{\Pp_{W|S}}[g_S(W)]\leq\frac{\xi}{\sqrt{2m}}\left( 2\,\sqrt{\log\frac{1}{\delta_0}}+\sum_{k=1}^\infty\ee_{k-1} \left(\lambda_k + \frac{\KL(\Pp_{W_k|S}\|\Pp^*_{W_k})+\log\frac{1}{\delta_k}}{\lambda_k}\right)\right)\,.
\end{align*}
\end{proposition}

The PAC-Bayesian bound in Proposition \ref{prop:PAC} is infinite for a deterministic algorithm (that is when $\Pp_{W|S=s}$ is a Dirac delta for all $s\in\Sc$). Remarkably, for suitable coefficients $\lambda_k$, $\delta_k$, and $\ee_k$, the chained bound of Proposition \ref{prop:chPAC} is always finite, since all the terms $\KL(\Pp_{W_k|S}\|\Pp^*_{W_k})$ are bounded by $\log|\W_k|$. However, the best choice of the parameters $\lambda$ and $\lambda_k$ is delicate, as it cannot depend on $S$ (and hence on the $\KL$ term). We refer to Appendix \ref{app:PAC} for further discussion on this last point. 
\section{Comparison of chained and unchained bounds}\label{sec:comp}
Having established the duality, we are left with the Hamletic question: \textit{chained or unchained, what is the best?} First, we notice that the requirements for the chained bounds are somewhat stronger. 
\begin{lemma}\label{lemma:chunch}
Under the assumptions \ref{ass}, let $\ee_0$ and $w_0$ be such that $\|w-w_0\|\leq\ee_0$, $\forall w\in \W$. Assume that $s\mapsto\nabla_w \Ell_s(w)$ has regularity $\Rr_\df(\xi)$ wrt $\Pp_S$, $\forall w\in\W$, and define $\hat \Ell_s(w) = \Ell_s(w)-\Ell_s(w_0)$ and $\hat \G=\E_{W\otimes S}[\hat \Ell_S(W)]-\E_{W, S}[\hat \Ell_S(W)]$. Then, $\hat \G = \G$, and $s\mapsto\hat \Ell_s(w)$ has regularity $\Rr_\df(\ee_0\xi)$, wrt $\Pp_S$ and $\forall w\in\W$.
\end{lemma}
Hence, whenever we derive a chained bound $|\G|\leq \xi\sum_{k=1}^\infty \ee_{k-1}\E_{\Pp_W}[\df(\Pp_S, \Pp_{S|W_k})]$ in our framework, we can always state an unchained counterpart in the form $|\G|\leq \ee_0\xi\,\E_{\Pp_W}[\df(\Pp_S, \Pp_{S|W})]$. Nevertheless, the next result shows that conditioning on $W_k$ instead of $W$ can often be helpful.
\begin{lemma}\label{lemma:wassdataprocineq}
Assume that $\mu\mapsto\df(\Pp_S, \mu)$ is convex. For any $\{\ee_k\}$-refining sequence of nets on $\W$, the sequence $\{\E_{\Pp_W}[\df(\Pp_S, \Pp_{S|W_k})]\}_{k\in\Nn}$ is non-decreasing and, $\forall k\in\Nn$, we have $$\E_{\Pp_W}[\df(\Pp_S, \Pp_{S|W_k})]\leq\E_{\Pp_W}[\df(\Pp_S, \Pp_{S|W})]\,.$$
\end{lemma}
$\KL(\nu\|\mu)$ is convex in both $\nu$ and $\mu$ \citep{renKL}, and the same holds for $\Wass(\mu, \nu)$ \citep{villani08}. Thus, $I(W_k;S)\leq I(W;S)$\footnote{This can also be seen as a trivial consequence of the data-processing inequality.} and $\E_{\Pp_W}[\Wass(\Pp_S,\Pp_{S|W_k})]\leq\E_{\Pp_W}[\Wass(\Pp_S,\Pp_{S|W})]$.

Lemma \ref{lemma:wassdataprocineq} alone is not enough to ensure that the chained bound is tighter than its unchained counterpart. However, if $\Pp_W$ is very concentrated on a tiny region of $\W$, so that $S$ is almost independent of $W_k$ up to a small scale (\textit{i.e.}, large $k$), then one can expect the chained result to be the tightest. We will clarify this intuition by means of two simple toy examples. Since \cite{asadi2018chaining} have already shown that the CMI bound can be much tighter than the MI one, here the focus is on the Wasserstein bounds.

\subsection{Comparison of the chained and unchained Wasserstein bounds}
In the following we denote by $\Bl$ the standard Wasserstein bound (Proposition \ref{prop:Wass}) and by $\Bg$ its chained counterpart (Proposition \ref{prop:Wassch}). For simplicity, we mainly focus on the case $m=1$, so that we can write $s=x$ and $\G=\E_{\Pp_{W\otimes X}}[\ell(W, X)]-\E_{\Pp_{W, X}}[\ell(W, X)]$.
\paragraph{Example 1}
\label{ex:unif}
Let $\W=\X=[-1,1]$, $\ell(w, x)=\frac{1}{2}(w-x)^2$, and $\ee_k=2^{-k}$, for $k\in\Nn$. We can find mappings $\pi_k$ that define an $\{\ee_k\}$-refining sequence of nets, with $\W_k = \{2^{1-k}j\;:\;j\in[-2^{k-1}:2^{k-1}]\}$, where $[a:b]=[a, b]\cap\mathbb Z$. Fix $k^\star\in\Nn$ and define $\theta=2^{-k^\star}$. Let $X$ be uniformly distributed on $(-\theta, \theta)$. We choose an algorithm that, given $x$, selects the $w$ minimising $\ell(w, x)$. This means that $\Pp_{W|X=x}=\delta_x$, where $\delta_x$ is the Dirac measure centred on $x$. Note that $\nabla_w\ell$ is $1$-Lipschitz and $\ell$ is $2$-Lipschitz (on $\X$, uniformly on $\W$). However, thanks to Lemma \ref{lemma:chunch} we know that we can consider the loss $\tilde\ell(w, x) =\ell(w,x)-\tfrac{x^2}{2}$, which leads to the same generalisation and is $1$-Lipschitz. 

In this simple example, we can compute exactly everything we need (see Appendix \ref{app:unif}):
$$|\G| = \frac{1}{3}\,\theta^2\simeq 0.33\,\theta^2\,;\qquad\frac{1}{2}\,\Bl=\Btl = \frac{2}{3}\,\theta\simeq 0.67\,\theta\,;\qquad\Bg = \frac{247}{105}\,\theta^2\simeq2.35\,\theta^2\,.$$
Note that, as $\theta$ decreases, $\Pp_W$ becomes more and more concentrated, since $W$ lies with probability $1$ in $(-\theta, \theta)$. In particular, $X$ and $W_k$ are independent for $k\leq k^\star=-\log_2\theta$, and so the first $k^\star$ terms in the chaining sum are null. For this reason, $\Bg$ captures the right behaviour $O(\theta^2)$ of $\G$ for $\theta\to 0$, which is not the case for $\Bl$ and $\Btl$.

Quite interestingly, it is possible to explicitly evaluate the CMI bound ($\Bcmi$) as well. We find $\Bcmi \simeq 3.50\,\theta$, meaning that in this example the chained MI bound fails to capture the right behaviour of $\G$ as $\theta\to 0$. We refer to Section \ref{sec:compWMI} for a few comments about this. On the other hand, the unchained MI bound is infinite, since $W$ is a deterministic function of $X$.

Finally, if we consider a larger random sample $S=\{X_1,\dots, X_m\}$, with $m>1$, we still have that the ratio $\B_{\nabla\Ell}/\B_\Ell$ (between the chained and unchained Wasserstein bounds) vanishes as $O(\theta)$ for $\theta\to0$. Again, this is a consequence of the fact that $S$ and $W_k$ are independent for $k\leq k^\star$, since $W$ is the empirical average $\sum_i X_i/m$ and lies in $(-\theta,\theta)$ with probability $1$.
\paragraph{Example 2}
\label{ex:gauss}
This toy model is inspired by Example 1 in \cite{asadi2018chaining}. Let $\W=\{w\in\R^2:\|w\|=1\}$\footnote{This is not a convex set. However, one can either suitably extend $\ell$ to the unit disk in $\R^2$, or easily check that the hypotheses of the extended framework of Theorem \ref{thm:genchapp} in Appendix \ref{app:extfram} are verified (see Section \ref{app:gauss}).} and $\X=\R^2$. Fix $a>0$ and let $X\sim\N((a,0), \Id)$, a normal distribution centered in $(a, 0)$, with the identity matrix as covariance. The algorithm aims at finding the direction of the mean of $X$ (that is $(1, 0)$), by choosing the $w$ that minimises the loss $\ell(w, x) = -w\cdot x$. Let $w_0=(1,0)$ and $\ee_0=4$. For $k\geq 1$, let $\W_k=\{w=(\cos\frac{2\pi j}{2^k}, \sin\frac{2\pi j}{2^k}): j\in[-2^{k-1}: 2^{k-1}-1]\}$ and $\ee_k=4/2^k$. We can then easily find projections $\pi_k$ that make $\{\W_k\}_{k\in\Nn}$ an $\{\ee_k\}$-sequence of refining nets. Both $\ell$ and $\nabla_w\ell$ are $1$-Lipschitz in $\X$, $\forall w\in\W$. Although it is hard to find the analytic expressions for $\Bl$ and $\Bg$, we can study their asymptotic behaviour for $a\to\infty$. In this limit, $\Pp_W$ becomes highly concentrated around $(0,1)$, as it tends towards a Dirac delta. So, for $a$ large enough we expect the chained bound to be the tightest. Indeed, we find
$$|\G|= \Theta(1/a)\,;\qquad\B_\ell = \Theta(1)\,;\qquad \Bg=O((\log a-\log\log a)/a)\,.\footnote{Here $f=O(g)$ stands for $\lim_{a\to\infty}|f(a)/g(a)|<\infty$, while by $f=\Theta(g)$ we mean that $f=O(g)$ and $g=O(f)$.}$$
Up to logarithmic factors, $\Bg$ can capture the correct behaviour of $|\G|$ as $a\to\infty$.

As a final remark, note that in this example the loss $\ell$ is not Lipschitz on $\W$, uniformly on $\X$, and so the \textit{forward-channel} Wasserstein bound from \cite{Wang2019Wass} does not apply.\footnote{Of course, $\ell(w, x) = -w\cdot x/\|x\|$ would bring the same algorithm and is $1$-Lipschitz in $w$. However this is just due to the radial symmetry. Changing  the problem slightly and considering for instance $\ell(w, x)=-w\cdot \psi(x)$, for a general $1$-Lipschitz map $\psi:\X\to\X$, would not allow to easily find an equivalent loss that is $1$-Lipschitz in $w$.}
\subsection{High concentration is not always enough} \label{sec:highD}
In both the previous examples, the chained bound was much tighter than its unchained counterpart when $\Pp_W$ was highly concentrated in a small neighbourhood $U$ of a single point $w_\star$. In particular, if $2\ee$ is the diameter of $U$, we can expect that just knowing that $W\in U$ is not informative up to a length-scale of order $\ee$. However, this can easily fail when $W$ concentrates around two far apart points (say $w_1$ and $w_2$). Indeed, if for small $k$ we already have that $\pi_k(w_1)\neq\pi_k(w_2)$, knowing that the chosen hypothesis is next to $w_1$ might bring a lot of information about $S$. On the other hand, one can still imagine situations in which there are multiple points around which $W$ concentrates, yet which one is the nearest to the chosen hypothesis is not informative about the sample. 

In Appendix \ref{app:highdim}, we discuss a high-dimensional version of \nameref{ex:unif}, where $W$ does not concentrate around a single point, but in a thin neighbourhood of a one-dimensional line. We show that when $\theta$ (the parameter describing the size of the support of $W$) has the right scaling with the dimension $d$ of $\W$, the ratio $\Bg/\Bl$ vanishes as $d\to\infty$.

\section{Comparison between MI and Wasserstein bounds}\label{sec:compWMI}
We conclude this paper with a few comments on the relation between the MI-based (Propositions \ref{prop:MI} and \ref{prop:MIchain}) and the Wasserstein-based bounds (Propositions \ref{prop:Wass} and \ref{prop:Wassch}). The problem comes down to comparing the $\KL$ divergence with the $1$-Wasserstein distance, a task closely related to transportation-cost inequalities (see \cite{CIT-064} for a pedagogical overview). Let $\mu$ be a probability measure on the Polish space $(\Z, \Sigma_\Z)$. For $\eta> 0$, $\mu$ is said to satisfy a $L^1$ transport-cost inequality with constant $\eta$ (in short $\mu\in T_1(\eta)$) if, for any $\nu\ll\mu$, $\Wass(\mu,\nu)\leq \sqrt{2\eta\,\KL(\nu\|\mu)}$. Hence, whenever $\Pp_S\in T_1(1)$ we are assured that each one of the two Wasserstein-based bounds is tighter than the corresponding MI-based one. For instance, this is the case when $\Pp_S$ is a multivariate normal whose covariance matrix is the identity \citep{talagrand1996}, as in \nameref{ex:gauss}. However, there is a price to pay: whenever the $L^1$ transport-inequality holds, then Lipschitzianity is stronger than subgaussianity. More precisely, \cite{BOBKOV19991} showed that $\mu\in T_1(1)$ if, and only if, for every $\xi$-Lipschitz function $f:\Z\to\R$, $f(Z)$ is $\xi$-subgaussian for $Z\sim\mu$.

It is worth noticing that, if the size of the support of $X$ is particularly small, the Wasserstein bounds can be much tighter than the MI ones. This is for instance the case in \nameref{ex:unif}, where the length-scale of the support of $X$ is given by $\theta$. There, the chained Wasserstein bound goes as $\theta^2$. A factor $\theta$ is brought by the chaining technique, which allows us to neglect the contributions of the larger length-scales, whilst the other factor $\theta$ is due to the use of the Wasserstein distance, which intrinsically takes into account the considered length-scale. In contrast, since the MI is scale-invariant, the CMI bound has only a linear dependence in $\theta$ coming from the chaining method.
\subsection{Scaling with the sample size}
It is worth mentioning the different roles that the factor $1/\sqrt m$ plays in the MI and the Wasserstein bounds. In the MI bound this scaling is linked to concentration properties, since it comes from the fact that the average of $m$ independent $\xi$-SG random variables is $(\xi/\sqrt m)$-SG. The requirement that $S$ is made of independent draws is hence essential in this case. On the other hand, in the Wasserstein bound the factor $1/\sqrt m$ has a merely geometric origin and follows from the relation \eqref{eq:defmetric} between the metrics $d_\X$ and $d_\Sc$. In particular, an alternative choice of $d_\Sc$ might yield a different factor in front of the bound, but also change the scaling with $m$ of the Wasserstein distance. A priori, it is not easy to say which $d_\Sc$ would bring the tightest bound. Once more, let us stress that the Wasserstein bound does not require that $\Pp_S = \Pp_X^{\otimes m}$. Indeed, $\Wass(\Pp_S, \Pp_{S|W})$ will take into account the dependencies between the training inputs, and we can expect it to scale poorly with $m$ if the different $X_i$ in $S$ are strongly correlated. However, even in the case of independent $X_i$, it is hard to say in general what is the exact dependence with $m$, for both $I(W;S)$ and $\Wass(\Pp_S, \Pp_{S|W})$.

As a final remark about the case $\Pp_S = \Pp_X^{\otimes m}$, just by looking at $\Pp_X$ it is sometimes possible to establish that both the standard and chained Wasserstein bounds are tighter than their MI counterparts, no matter the size of the training dataset and the choice of the algorithm. To this purpose, we can again exploit some classical results on the transport-cost inequalities \citep{CIT-064, gozleo10}. For a probability measure $\mu$, we say that $\mu\in T_2(1)$ if, for any $\nu\ll\mu$, $\Wass_2(\mu,\nu) \leq \sqrt{2\KL(\nu\|\mu)}$. It is known that $\mu\in T_2(1)$ implies that $\mu^{\otimes m}\in T_2(1)$, $\forall m\geq 1$. In particular, if $\Pp_X\in T_2(1)$, then we are ensured that $\Pp_S = \Pp_X^{\otimes m}\in T_2(1)$. Since $\Wass=\Wass_1\leq \Wass_2$, $\Pp_X\in T_2(1)$ actually implies $\Pp_S\in T_1(1)$, which (as we discussed the beginning of this section) means that each Wasserstein-based bound is tighter than the corresponding MI-based one. 

\section{Conclusion}
We introduced a general framework allowing us to derive new generalisation results leveraging on the chaining technique. By doing so, under suitable regularity conditions we established a duality between chained and unchained generalisation bounds. Although the chained bounds usually come at the price of stricter assumptions, sometimes they better capture the loss function's behaviour, especially in cases where the hypothesis distribution is highly concentrated. We hence believe that combining the chaining method with other information-theoretic techniques is a promising direction in order to tighten the bounds on the generalisation error. 

In this work we have mainly focused on the \textit{backward-channel} information-theoretic perspective, as we believe that it combines naturally with the chaining on the hypotheses' space. However, the chained PAC-Bayesian result that we presented is an example of a \textit{forward-channel} bound, as it considers the distribution of the hypotheses, conditioned on the sample. A future direction of study could be to extend our general framework to include \textit{forward-channel} bounds. We believe this should not present major technical difficulties and might bring new interesting results. 

Although information-theoretic bounds are usually hard to evaluate in practice, recent works have derived computable analytic bounds for specific algorithms, such as Langevin dynamics or stochastic gradient descent, by upper-bounding information-theoretic generalisation results. We believe that combining these ideas with the chaining technique is a venue worth exploring. 

\acks{Eugenio Clerico is partly supported by the UK Engineering and Physical Sciences Research Council (EPSRC) through the grant EP/R513295/1 (DTP scheme). Arnaud Doucet is partly supported by the EPSRC grant EP/R034710/1. He also acknowledges the support of the UK Defence Science and Technology Laboratory (DSTL) and EPSRC under grant EP/R013616/1. This is part of the collaboration between US DOD, UK MOD, and UK EPSRC, under the Multidisciplinary University Research Initiative.}


\newpage
\bibliography{bib}

\begin{thebibliography}{53}
\providecommand{\natexlab}[1]{#1}
\providecommand{\url}[1]{\texttt{#1}}
\expandafter\ifx\csname urlstyle\endcsname\relax
  \providecommand{\doi}[1]{doi: #1}\else
  \providecommand{\doi}{doi: \begingroup \urlstyle{rm}\Url}\fi

\bibitem[Alquier(2021)]{alquier2021user}
P.~Alquier.
\newblock User-friendly introduction to {PAC}-{B}ayes bounds.
\newblock \emph{arXiv:2110.11216}, 2021.

\bibitem[Aminian et~al.(2021)Aminian, Toni, and
  Rodrigues]{aminian2021informationtheoretic}
G.~Aminian, L.~Toni, and M.~R.~D. Rodrigues.
\newblock Information-theoretic bounds on the moments of the generalization
  error of learning algorithms.
\newblock \emph{arXiv:2102.02016}, 2021.

\bibitem[Anthony and Bartlett(2002)]{neuralnetAnthony}
M.~Anthony and P.~L. Bartlett.
\newblock \emph{Neural Network Learning - Theoretical Foundations}.
\newblock Cambridge University Press, 2002.

\bibitem[Asadi and Abbe(2020)]{Asadi2020ChainingMC}
A.~R. Asadi and E.~Abbe.
\newblock Chaining meets chain rule: Multilevel entropic regularization and
  training of neural nets.
\newblock \emph{JMLR}, 21, 2020.

\bibitem[Asadi et~al.(2018)Asadi, Abbe, and Verd{\'u}]{asadi2018chaining}
A.~R. Asadi, E.~Abbe, and S.~Verd{\'u}.
\newblock Chaining mutual information and tightening generalization bounds.
\newblock \emph{NeurIPS}, 2018.

\bibitem[Audibert and Bousquet(2004)]{genericPACBayes}
J-Y. Audibert and O.~Bousquet.
\newblock {PAC}-{B}ayesian generic chaining.
\newblock \emph{NeurIPS}, 2004.

\bibitem[Belkin et~al.(2018)Belkin, Ma, and Mandal]{belkin2018understand}
M.~Belkin, S.~Ma, and S.~Mandal.
\newblock To understand deep learning we need to understand kernel learning.
\newblock \emph{ICML}, 2018.

\bibitem[Bobkov and Götze(1999)]{BOBKOV19991}
S.~G. Bobkov and F.~Götze.
\newblock Exponential integrability and transportation cost related to
  logarithmic sobolev inequalities.
\newblock \emph{Journal of Functional Analysis}, 163\penalty0 (1), 1999.

\bibitem[Bousquet and Elisseeff(2002)]{bousquet2002}
O.~Bousquet and A.~Elisseeff.
\newblock Stability and generalization.
\newblock \emph{Journal of Machine Learning Research}, 2, 2002.

\bibitem[Bousquet et~al.(2004)Bousquet, Boucheron, and Lugosi]{Bousquet2004}
O.~Bousquet, S.~Boucheron, and G.~Lugosi.
\newblock \emph{Introduction to Statistical Learning Theory}.
\newblock Springer, 2004.

\bibitem[Bu et~al.(2019)Bu, Zou, and Veeravalli]{Bu2019}
Y.~Bu, S.~Zou, and V.~V. Veeravalli.
\newblock Tightening mutual information based bounds on generalization error.
\newblock \emph{ISIT}, 2019.

\bibitem[Catoni(2009)]{catonipreprint}
O.~Catoni.
\newblock A {PAC-B}ayesian approach to adaptive classification.
\newblock \emph{Preprint LPMA}, 840, 2009.

\bibitem[Cooper and Farid(2020)]{Cooper2020ATF}
E.~A. Cooper and H.~Farid.
\newblock A toolbox for the radial and angular marginalization of bivariate
  normal distributions.
\newblock \emph{arXiv:2005.09696}, 2020.

\bibitem[Donsker and Varadhan(1983)]{donskerVaradhan}
M.~D. Donsker and S.~R.~S. Varadhan.
\newblock Asymptotic evaluation of certain {M}arkov process expectations for
  large time. iv.
\newblock \emph{Communications on Pure and Applied Mathematics}, 36\penalty0
  (2), 1983.

\bibitem[Dudley(1967)]{DUDLEY}
R.~M. Dudley.
\newblock The sizes of compact subsets of {H}ilbert space and continuity of
  {G}aussian processes.
\newblock \emph{Journal of Functional Analysis}, 1\penalty0 (3), 1967.

\bibitem[Dudley(2002)]{dudley_2002}
R.~M. Dudley.
\newblock \emph{Real Analysis and Probability}.
\newblock Cambridge University Press, 2002.

\bibitem[Dwork and Roth(2014)]{DworkDP2014}
C.~Dwork and A.~Roth.
\newblock The algorithmic foundations of differential privacy.
\newblock \emph{Foundations and Trends in Theoretical Computer Science}, 9,
  2014.

\bibitem[Erven and Harremo{\"e}s(2014)]{renKL}
T.~Erven and P.~Harremo{\"e}s.
\newblock Rényi divergence and kullback-leibler divergence.
\newblock \emph{IEEE Transactions on Information Theory}, 60\penalty0 (7),
  2014.

\bibitem[Esposito et~al.(2021)Esposito, Gastpar, and
  Issa]{esposito2020generalization}
A.~R. Esposito, M.~Gastpar, and I.~Issa.
\newblock Generalization error bounds via {R}\'enyi-, $f$-divergences and
  maximal leakage.
\newblock \emph{IEEE Transactions on Information Theory}, 67\penalty0 (3),
  2021.

\bibitem[Gozlan and Léonard(2010)]{gozleo10}
N.~Gozlan and C.~Léonard.
\newblock Transport inequalities. {A} survey.
\newblock \emph{Markov Processes and Related Fields}, 16, 2010.

\bibitem[Guedj(2019)]{guedj2019primer}
B.~Guedj.
\newblock A primer on {PAC}-{B}ayesian learning.
\newblock \emph{Proceedings of the Second Congress of the French Mathematical
  Society}, 2019.

\bibitem[Guntuboyina et~al.(2014)Guntuboyina, Saha, and Schiebinger]{powerinf}
A.~Guntuboyina, S.~Saha, and G.~Schiebinger.
\newblock Sharp inequalities for $f$ -divergences.
\newblock \emph{IEEE Transactions on Information Theory}, 60\penalty0 (1),
  2014.

\bibitem[Haghifam et~al.(2020)Haghifam, Negrea, Khisti, Roy, and
  Dziugaite]{haghifam2020}
M.~Haghifam, J.~Negrea, A.~Khisti, D.~M. Roy, and G.~K. Dziugaite.
\newblock Sharpened generalization bounds based on conditional mutual
  information and an application to noisy, iterative algorithms.
\newblock \emph{NeurIPS}, 2020.

\bibitem[Hellstr{\"o}m and Durisi(2020)]{hellstrom2020}
F.~Hellstr{\"o}m and G.~Durisi.
\newblock Generalization bounds via information density and conditional
  information density.
\newblock \emph{IEEE Journal on Selected Areas in Information Theory},
  1\penalty0 (3), 2020.

\bibitem[Hoeffding(1963)]{Hoeffding}
W.~Hoeffding.
\newblock Probability inequalities for sums of bounded random variables.
\newblock \emph{Journal of the American Statistical Association}, 58\penalty0
  (301), 1963.

\bibitem[Kechris(1995)]{kechris95}
A.~S. Kechris.
\newblock \emph{Classical Descriptive Set Theory}.
\newblock Springer-Verlag, 1995.

\bibitem[Lehmann and Casella(1998)]{LehmCase98}
E.~L. Lehmann and G.~Casella.
\newblock \emph{Theory of Point Estimation}.
\newblock Springer-Verlag, New York, NY, USA, second edition, 1998.

\bibitem[Liero et~al.(2018)Liero, Mielke, and Savaré]{beppe<3}
M.~Liero, A.~Mielke, and G.~Savaré.
\newblock Optimal entropy-transport problems and a new
  {H}ellinger–{K}antorovich distance between positive measures.
\newblock \emph{Inventiones Mathematicae}, 211, 2018.

\bibitem[Lopez and Jog(2018)]{lopez2018wass}
A.~T. Lopez and V.~Jog.
\newblock Generalization error bounds using {W}asserstein distances.
\newblock \emph{IEEE Information Theory Workshop (ITW)}, 2018.

\bibitem[Marsaglia et~al.(1990)Marsaglia, Narasimhan, and Zaman]{marsaglia}
G.~Marsaglia, B.~Narasimhan, and A.~Zaman.
\newblock The distance between random points in rectangles.
\newblock \emph{Communications in Statistics - Theory and Methods}, 19, 1990.

\bibitem[McAllester(1998)]{McAllester98somepac-bayesian}
D.~A. McAllester.
\newblock Some {PAC}-{B}ayesian theorems.
\newblock \emph{COLT}, 1998.

\bibitem[McAllester(1999)]{mcallester}
D.~A. McAllester.
\newblock {PAC}-{B}ayesian model averaging.
\newblock \emph{COLT}, 1999.

\bibitem[Negrea et~al.(2019)Negrea, Haghifam, Dziugaite, Khisti, and
  Roy]{negrea19}
J.~Negrea, M.~Haghifam, G.~K. Dziugaite, A.~Khisti, and D.~M. Roy.
\newblock Information-theoretic generalization bounds for {SGLD} via
  data-dependent estimates.
\newblock \emph{NeurIPS}, 2019.

\bibitem[Neu et~al.(2021)Neu, Dziugaite, Haghifam, and Roy]{neu21}
G.~Neu, G.~K. Dziugaite, M.~Haghifam, and D.~M. Roy.
\newblock Information-theoretic generalization bounds for stochastic gradient
  descent.
\newblock \emph{COLT}, 2021.

\bibitem[Palomar and Verd\'u(2008)]{lautuminf}
D.~P. Palomar and S.~Verd\'u.
\newblock Lautum information.
\newblock \emph{IEEE Transactions on Information Theory}, 54\penalty0 (3),
  2008.

\bibitem[Raginsky and Sason(2013)]{CIT-064}
M.~Raginsky and I.~Sason.
\newblock Concentration of measure inequalities in information theory,
  communications, and coding.
\newblock \emph{Foundations and Trends in Communications and Information
  Theory}, 10\penalty0 (1-2), 2013.

\bibitem[Rodr{\'i}guez-G{\'a}lvez et~al.(2020)Rodr{\'i}guez-G{\'a}lvez, Bassi,
  Thobaben, and Skoglund]{Galvez2020OnRS}
B.~Rodr{\'i}guez-G{\'a}lvez, G.~Bassi, R.~Thobaben, and M.~Skoglund.
\newblock On random subset generalization error bounds and the stochastic
  gradient {L}angevin dynamics algorithm.
\newblock \emph{IEEE Information Theory Workshop (ITW)}, 2020.

\bibitem[Rodr{\'{\i}}guez-G{\'{a}}lvez
  et~al.(2021)Rodr{\'{\i}}guez-G{\'{a}}lvez, Bassi, Thobaben, and
  Skoglund]{borja2021tighter}
B.~Rodr{\'{\i}}guez-G{\'{a}}lvez, G.~Bassi, R.~Thobaben, and M.~Skoglund.
\newblock Tighter expected generalization error bounds via {W}asserstein
  distance.
\newblock \emph{arXiv:2101.09315}, 2021.

\bibitem[Russo and Zou(2019)]{russo2019much}
D.~Russo and J.~Zou.
\newblock How much does your data exploration overfit? controlling bias via
  information usage.
\newblock \emph{IEEE Transactions on Information Theory}, 66\penalty0 (1),
  2019.

\bibitem[Shalev-Shwartz and Ben-David(2014)]{shalevBook2014}
S.~Shalev-Shwartz and S.~Ben-David.
\newblock \emph{Understanding Machine Learning - From Theory to Algorithms.}
\newblock Cambridge University Press, 2014.

\bibitem[Steinke and Zakynthinou(2020)]{steinke2020}
T.~Steinke and L.~Zakynthinou.
\newblock Reasoning about generalization via conditional mutual information.
\newblock \emph{COLT}, 2020.

\bibitem[Talagrand(1996)]{talagrand1996}
M.~Talagrand.
\newblock Transportation cost for {G}aussian and other product measures.
\newblock \emph{Geometric and Functional Analysis}, 6\penalty0 (3), 1996.

\bibitem[Talagrand(2005)]{Talagrand2005MMbook}
M.~Talagrand.
\newblock \emph{The Generic Chaining: Upper and Lower Bounds of Stochastic
  Processes}.
\newblock Springer-Verlag, 2005.

\bibitem[Talagrand(2014)]{Talagrand2014MMbook}
M.~Talagrand.
\newblock \emph{Upper and Lower Bounds for Stochastic Processes: Modern Methods
  and Classical Problems}.
\newblock Springer-Verlag, 2014.

\bibitem[van Handel(2016)]{vanHandel}
R.~van Handel.
\newblock \emph{Probability in High Dimensions}, 2016.
\newblock [Online; accessed on 02/2022.]
  \url{https://web.math.princeton.edu/~rvan/APC550.pdf}.

\bibitem[Vapnik(2000)]{vapnik00}
V.~N. Vapnik.
\newblock \emph{{The Nature of Statistical Learning Theory}}.
\newblock Springer, 2000.

\bibitem[Vershynin(2018)]{vershynin_2018}
R.~Vershynin.
\newblock \emph{High-Dimensional Probability: An Introduction with Applications
  in Data Science}.
\newblock Cambridge University Press, 2018.

\bibitem[Villani(2009)]{villani08}
C.~Villani.
\newblock \emph{Optimal Transport -- Old and New}.
\newblock Springer, 2009.

\bibitem[Wang et~al.(2019)Wang, Diaz, Filho, and Calmon]{Wang2019Wass}
H.~Wang, M.~Diaz, J.~C. S.~Santos Filho, and F.~P. Calmon.
\newblock An information-theoretic view of generalization via {W}asserstein
  distance.
\newblock \emph{ISIT}, 2019.

\bibitem[Wyner(1978)]{WYNERCMI}
A.~D. Wyner.
\newblock A definition of conditional mutual information for arbitrary
  ensembles.
\newblock \emph{Information and Control}, 38\penalty0 (1), 1978.

\bibitem[Xu and Raginsky(2017)]{Xu2017InformationtheoreticAO}
A.~Xu and M.~Raginsky.
\newblock Information-theoretic analysis of generalization capability of
  learning algorithms.
\newblock \emph{NeurIPS}, 2017.

\bibitem[Zhang et~al.(2021)Zhang, Bengio, Hardt, Recht, and
  Vinyals]{zhang2017understanding}
C.~Zhang, S.~Bengio, M.~Hardt, B.~Recht, and O.~Vinyals.
\newblock Understanding deep learning (still) requires rethinking
  generalization.
\newblock \emph{Commun. ACM}, 64\penalty0 (3), 2021.

\bibitem[Zhou et~al.(2022)Zhou, Tian, and Liu]{zhou2022stochastic}
R.~Zhou, C.~Tian, and T.~Liu.
\newblock Stochastic chaining and strengthened information-theoretic
  generalization bounds.
\newblock \emph{arXiv:2201.12192}, 2022.

\end{thebibliography}
\newpage
\appendix
\section{Omitted proofs of Sections \ref{sec:fram} and \ref{sec:ex}}
Here $(\Z, d_\Z)$ is a separable complete metric space, with Borel $\sigma$-algebra $\Sigma_\Z$ induced by the metric. $\W\times\Z$ is endowed with the product $\sigma$-algebra $\Sigma_\W\otimes\Sigma_\Z$. $\PR$ denotes the space of probability measures on $\Z$ and is endowed with the $\sigma$-algebra induced by the topology of weak convergence.
\subsection{Proof of Lemma \ref{lemma:dfreg}}
\begin{manuallemma}{\ref{lemma:dfreg}}
Let $\df_1:(\mu, \nu) \mapsto \sqrt{2\KL(\nu\|\mu)}$ and $\df_2:(\mu, \nu) \mapsto \Wass(\mu, \nu)$. Consider a measurable map $f:\Z\to\R^q$ (with $q\geq 1$). If $f(Z)$ is $\xi$-SG for $Z\sim\mu\in\PR$, then $f$ has regularity $\Rr_{\df_1}(\xi)$ wrt $\mu$. If $f$ is $\xi$-Lipschitz on $\Z$, then $f$ has regularity $\Rr_{\df_2}(\xi)$, wrt any $\mu\in\PR$ such that $f\in L^1(\mu)$.
\end{manuallemma}
\begin{proof}
First, notice that Lemmas \ref{lemma:klmeas} and \ref{lemma:wassmeas} ensure that both $\df_1$ and $\df_2$ are measurable, as required by Definition \ref{def:Dreg}. 

Assume that $f(Z)$ is $\xi$-SG for $Z\sim\mu$. Then, by definition $f\in L^1(\mu)$. Fix $\nu$ such that $f\in L^1(\nu)$ and $\Supp(\nu)\subseteq\Supp(\mu)$. If $q=1$, the Donsker-Varadhan representation of $\KL$ \citep{donskerVaradhan} and subgaussianity yield
$$\KL(\nu\|\mu)\geq \sup_{\lambda\in\R}\lambda(\E_\nu[f(Z)]-\E_\mu[f(Z)])-\lambda^2\xi^2/2 = \frac{(\E_\mu[f(Z)]-\E_\nu[f(Z)])^2}{2\xi^2}\,,$$
from which the $\df_1$-regularity of $f$ follows immediately. The case of a generic $q>1$ is trivial, since $v\cdot f(Z)$ is $(\xi\|v\|)$-SG by Definition \ref{def:SG}, for all $v\in\R^q$.

Now, let $f\in L^1(\mu)$ be $\xi$-Lipschitz. If $q=1$, let $\pi$ be any coupling with marginals $\mu$ and $\nu$. We have that
$$|\E_\mu[f(Z)]-\E_\nu[f(Z)]|=|\E_{(Z, Z')\sim\pi}[f(Z)-f(Z')]|\leq \xi\,\E_{(Z, Z')\sim\pi}[d(Z, Z')]\,.$$
The $\df_2$-regularity can be established by taking the inf among all the couplings $\pi$ with marginals $\mu$ and $\nu$. The case $q>1$ follows from the fact that $z\mapsto v\cdot f(z)$ is $(\xi\|v\|)$-Lipschitz.
\end{proof}
\subsection{Proof of Theorem \ref{thm:genstd}} 
Theorem \ref{thm:genstd} is equivalent to the following result.
\begin{theorem}\label{thm:genstdA}
Consider a measurable map $F:\W\times\Z\to\R$, such that $z\mapsto F(w, z)$ has regularity $\Rr_\df(\xi)$ wrt $\Pp_Z$ and for all $w\in\W$. Then we have
$$|\E_{\Pp_{W\otimes  Z}}[F(W, Z)]-\E_{\Pp_{W, Z}}[F(W, Z)]|\leq \xi\,\E_{\Pp_W}[\df(\Pp_Z, \Pp_{Z|W})]\,.$$
\end{theorem}
\begin{proof}
First, note that $\Supp(\Pp_{Z|W=w})\subseteq\Supp(\Pp_Z)$ by Lemma \ref{lemma:supp} and $\E_{\Pp_{Z|W=w}}[|F(w, Z)|]<\infty$, $\Pp_W$-a.s, since $\E_{\Pp_{W, Z}}[|F(W, Z)|]<+\infty$. In particular, for $\Pp_W$-almost every $w\in\W$ we have that
$$|\E_{\Pp_Z}[F(w, Z)]-\E_{\Pp_{Z|W=w}}[F(w, Z)]|\leq \xi\,\df(\Pp_Z, \Pp_{Z|W=w})\,.$$
Then the conclusion follows by taking the expectation wrt $\Pp_W$ and using Jensen's inequality. 
\end{proof}
\subsection{Proof of Theorem \ref{thm:gench}}\label{app:proofch}
Theorem \ref{thm:gench} follows from the next result, which is a direct corollary of Theorem \ref{thm:genchapp} and Lemma \ref{lemma:superreg}, proved in Appendix \ref{app:extfram}.
\begin{theorem}\label{thm:genchA}
Let $\W$ be a compact convex subset of $\R^d$ with non-empty interior.  Consider a measurable map $F:\W\times\Z\to\R$, such that $w\mapsto F(w, z)$ is $C^1$, $\Pp_Z$-a.s. Assume that $\sup_{(w, z)\in\W\times\X}|F(w, z)|<+\infty$ and $\sup_{(w, z)\in\W\times\X}|\nabla_wF(w, z)|<+\infty$. If $z\mapsto \nabla_w F(w, z)$ has regularity $\Rr_\df(\xi)$ wrt $\Pp_Z$, $\forall w\in\W$, then we have that for any $\{\ee_k\}$-refining sequence of nets on $\W$
$$|\E_{\Pp_{W\otimes  Z}}[F(W, Z)]-\E_{\Pp_{W, Z}}[F(W, Z)]|\leq \xi\sum_{k=1}^\infty \ee_{k-1}\E_{\Pp_W}[\df(\Pp_Z, \Pp_{Z|W_k})]\,.$$
\end{theorem}
\begin{proof}
By Lemma \ref{lemma:superreg}, the regularity of $\nabla_w F$ implies that the map $z\mapsto (F(w, z)-F(w', z))$ has regularity $\Rr_\df(\xi\|w-w'\|)$, wrt $\Pp_Z$ and for all $w, w'\in\W$. We conclude by Theorem \ref{thm:genchapp}.
\end{proof}
\subsection{Proof of Lemma \ref{lem:equivSGP}}
\begin{manuallemma}{\ref{lem:equivSGP}}
Under the assumptions \ref{ass}, the stochastic process $(\ell(w, X))_{w\in \W}$ is $\xi$-SG if, and only if, $\nabla_w\ell(w, X)$ is a $\xi$-SG vector for all $w\in\W$.
\end{manuallemma}
\begin{proof}
First, notice that, without loss of generality, we can consider the case of a one-dimensional $\W\subseteq\R$. Indeed, if $\W$ is higher dimensional, for any two given points $w$ and $w'$, we can always restrict to a line connecting them, making the problem 1D. Moreover, letting $\bar\ell(w, x) = \ell(w, x) - \E_{\Pp_X}[\ell(w, X)]$ we have that the assumptions in \ref{ass} ensure that $\nabla_w\bar \ell = \nabla_w\ell -\E_{\Pp_X}[\nabla_w\ell]$. So, we just need to show that the lemma holds for $\bar\ell$. 

Now, let $\bar\ell$ be a $\xi$-SG process, so that for $\varepsilon\neq 0$ and $\lambda\in\R$
$$\E_{\Pp_X}[e^{\lambda(\bar\ell(w+\varepsilon, X)-\bar\ell(w, X))/\varepsilon}]\leq e^{\frac{\lambda^2}{2\varepsilon^2}\,\xi^2\varepsilon^2} = e^{\frac{\lambda^2\xi^2}{2}}\,.$$
In particular, by Fatou's lemma we have
$$\E_{\Pp_X}[e^{\lambda\partial_w\bar\ell(w, X)}] = \E_{\Pp_X}\left[\lim_{\varepsilon\to 0}e^{\lambda\frac{\bar\ell(w+\varepsilon, X)-\bar\ell(w, X)}{\varepsilon}}\right]\leq\liminf_{\varepsilon\to 0}\E_{\Pp_X}\left[e^{\lambda\frac{\bar\ell(w+\varepsilon, X)-\bar\ell(w, X)}{\varepsilon}}\right] \leq e^{\frac{\lambda^2\xi^2}{2}}\,.$$

For the reverse implication, assume that $\partial_w\bar\ell(w, X)$ is $\xi$-SG for all $w\in\W$. Fix $w, w'\in\W$. By the assumptions \ref{ass} we have that, $\Pp_X$-a.s. 
$$\bar\ell(w', x)-\bar\ell(w, x) = \int_w^{w'}\partial_w\bar\ell(u, x)\dd u\,.$$
Fix a positive integer $N$ and let $u_j = w + j(w'-w)/N$. We have
\begin{align*}
    \E_{\Pp_X}\left[e^{\lambda\frac{w'-w}{N}\sum_{j=1}^N\partial_w\bar\ell(u_j, X)}\right]&=\E_{\Pp_X}\left[\prod_{j=1}^Ne^{\lambda\frac{w'-w}{N}\partial_w\bar\ell(u_j, X)}\right]\\
    &\leq \prod_{j=1}^N\E_{\Pp_X}[e^{\lambda(w'-w)\partial_w\bar\ell(u_j, X)}]^{1/N} \leq e^{(\lambda^2\xi^2(w-w')^2)/2}\,.
\end{align*}
Now let $Y_N(x) = \frac{w'-w}{N}\sum_{j=1}^N\partial_w\bar\ell(u_j, x)$. Since $w\mapsto\ell(w, x)$ is $C^1$ ($\Pp_X$-a.s.) by \ref{ass}, we have $\Pp_X$-a.s. that
$$\lim_{N\to\infty}Y_N(x) = \int_w^{w'} \partial_w\bar\ell(u, x)\dd u = \ell(w', x)-\ell(w, x)\,.$$
We conclude that
$$\lim_{N\to\infty}\E_{\Pp_X}\left[e^{\lambda\frac{w'-w}{N}\sum_{j=1}^N\partial_w\bar\ell(u_j, x)}\right] = \E_{\Pp_X}\left[e^{\lambda(\ell(w', x)-\ell(w, x))}\right]\,,$$
since by \ref{ass} $\partial_w\bar\ell$ is bounded. 
\end{proof}

\section{Extended general framework}
\subsection{Weakening the assumptions for the chained bounds}\label{app:extfram}
The framework that we presented in the main text required the assumptions \ref{ass} for $\ell$ (or $F$ in the setting of Theorem \ref{thm:genchA}) for the chained bound. Actually a result equivalent to Theorem \ref{thm:gench} can be obtained with weaker assumptions, namely just requiring almost sure continuity and boundedness in expectation for $\ell$, and dropping the convexity hypothesis for $\W$. 
\begin{theorem}\label{thm:genchapp}
Let $\W$ be a compact subset of $\R^d$ and $\{\W_k\}$ a $\{\ee_k\}$-refining sequence of nets on $\W$. Consider a measurable map $F:\W\times\Z\to\R$, such that $w\mapsto F(w, z)$ is continuous on $\W$, $\Pp_Z$-a.s., and $\E_{\Pp_Z}[\sup_{w\in\W}|F(w, Z)|]<+\infty$. Moreover, assume that the function $z\mapsto F(w, z)-F(w', z)$ has regularity $\Rr_\df(\xi\|w-w'\|)$ wrt $\Pp_Z$, for every $(w, w')\in\W^2$. Then, we have
$$|\E_{\Pp_{W, Z}}[F(W, Z)]-\E_{\Pp_{W\otimes  Z}}[F(W, Z)]|\leq \xi\sum_{k=1}^\infty \ee_{k-1}\E_{\Pp_W}[\df(\Pp_Z, \Pp_{Z|W_k})]\,,$$
where $\E_{\Pp_W}[\df(\Pp_Z, \Pp_{Z|W_k})]=\int_\W\df(\Pp_Z, \Pp_{Z|W\in \pi_k^{-1}(w)})\dd\Pp_W(w)$.
\end{theorem}
\begin{proof}
First notice that $w\mapsto F(w, z)$ is uniformly continuous on $\W$, $\Pp_Z$-a.s., since $\W$ is compact. It follows that $z\mapsto \sup_{w\in\W}|F(w, z)-F(w_k, z)|\to 0$, $\Pp_Z$-a.s., and so, using the fact that this map is dominated by $z\mapsto 2\sup_{w\in\W}|F(w, z)|$, which is in $L^1(\Pp_Z)$ by hypothesis, we get that
$$\E_{\Pp_Z}\left[\sup_{w\in\W}|F(w, Z)-F(w_k, Z)|\right]\to 0\,,$$
as $k\to+\infty$, by dominated convergence. In particular, $\E_{\Pp_{W,Z}}[|F(W, Z)-F(W_k, Z)|]\to 0$ and $\E_{\Pp_{W\otimes  Z}}[|F(W, Z)-F(W_k, Z)|]\to 0$. Moreover, recalling that $\W_0=\{w_0\}$ we see that $\E_{\Pp_{W\otimes Z}}[F(W_0, Z)]-\E_{\Pp_{W, Z}}[F(W_0, Z)]=0$. It follows that 
\begin{align}\label{eq:telescope}
    \begin{split}
    \big|\E_{\Pp_{W\otimes Z}}&[F(W, Z)]-\E_{\Pp_{W, Z}}[F(W, Z)]\big|\\
    &\leq\sum_{k=1}^\infty \left|\E_{\Pp_{W\otimes Z}}[F(W_k, Z)-F(W_{k-1}, Z)]-\E_{\Pp_{W, Z}}[F(W_k, Z)-F(W_{k-1}, Z)]\right|\\
    &=\sum_{k=1}^\infty \left|\E_{\Pp_{W_k\otimes Z}}[F(W_k, Z)-F(W_{k-1}, Z)]-\E_{\Pp_{W_k, Z}}[F(W_k, Z)-F(W_{k-1}, Z)]\right|\,.
    \end{split}
\end{align}
Now, notice that $\Supp(\Pp_{Z|W_k=w_k})\subseteq\Supp(\Pp_Z)$ $\Pp_W$-a.s.\ by Lemma \ref{lemma:supp}. Moreover, by the fact that $\E_{\Pp_Z}[\sup_{w\in\W}|F(w, Z)|]<+\infty$ we have $\E_{\Pp_{Z|W_k=w_k}}[\sup_{w\in\W}|F(w, Z)|]<+\infty$, and so in particular $\E_{\Pp_{Z|W_k=w_k}}[|F(w_k, Z)-F(w_{k-1},Z)|]<+\infty$, for $\Pp_{W_k}$-almost every $w_k$. Thus, using the regularity of $F$ we find that
\begin{align}\label{eq:finqui}
    \begin{split}
    \big|\E_{\Pp_Z}[F(w_k, Z)-F(w_{k-1},Z)]-\E_{\Pp_{Z|W_k=w_k}}&[F(w_k, Z)-F(w_{k-1},Z)]\big|\\
    &\leq \xi\|w_k-w_{k-1}\|\df(\Pp_Z, \Pp_{Z|W_k=w_k})\,,
    \end{split}
\end{align}
for $\Pp_{W_k}$-almost every $w_k$. We can hence conclude by taking the expectation wrt $\Pp_W$ and using Jensen's inequality.
\end{proof}
It is easy to see that the current framework includes the one in the main text. 
\begin{lemma}\label{lemma:superreg}
Let $\W\subseteq\R^d$ be a convex set. Consider a measurable map $F:\W\times\Z\to\R$ with the following properties: $w\mapsto F(w, z)$ is $C^1$ $\Pp_Z$-a.s., $\sup_{(w, z)\in\W\times\Z}|F(w, z)|<+\infty$, and $\sup_{(w, z)\in\W\times\Z}\|\nabla_w F(w, z)\|<+\infty$. If $z\mapsto \nabla_w F(w, z)$ has regularity $\Rr_\df(\xi)$ wrt $\Pp_Z$, $\forall w\in\W$, then $z\mapsto (F(w, z)-F(w', z))$ has regularity $\Rr_\df(\xi\|w-w'\|)$ (wrt $\Pp_Z$ and $\forall w, w'\in\W$).
\end{lemma}
\begin{proof}
Fix a probability $\hat\Pp_Z$ on $\Z$ such that $\Supp(\hat\Pp_Z)\subseteq\Supp(\Pp_Z)$.  Now, notice that since $\W$ is convex, and $F$ is $C^1$, for $\Pp_Z$-almost every $z$ we have
$$F(w, z) - F(w', z) = \int_0^1 \nabla_wF(w_t, z)\cdot(w-w')\,\dd t\,,$$
where $w_t = w' + t(w-w')$. Since $F$ is uniformly bounded, we can use Fubini-Tonelli's theorem and Jensen's inequality to write
\begin{align*}
    |\E_{\Pp_Z}[F(w, Z)]-&\E_{\hat\Pp_Z}[F(w', Z)]| \\
    &\leq\int_0^1\left|\E_{\Pp_Z}[\nabla_wF(w_t,  Z)\cdot(w-w')] - \E_{\hat\Pp_Z}[\nabla_wF(w_t,  Z)\cdot(w-w')]\right|\,\dd t\,.
\end{align*}
Using the fact that $z\mapsto F(w_t, z)$ is in both $L^1(\Pp_Z)$ and $L^1(\hat\Pp_Z)$ since $F$ is bounded, we conclude by the regularity of $\nabla_wF$. 
\end{proof}
All the bounds of this paper can be restated in this more general framework. We will only give a direct proof of Proposition \ref{prop:MIchain0}.
\begin{manualprop}{\ref{prop:MIchain0}}
Let $\Pp_S=\Pp_X^{\otimes m}$ and $\W$ be a compact set, with an $\{\ee_k\}$-refining sequence of nets defined on it. Suppose that $w\mapsto\ell(w, x)$ is continuous, for $\Pp_X$-almost every $x$,\footnote{Note that in \cite{asadi2018chaining} the result is stated under a weaker assumption of separability of the process. To avoid introducing further definitions and technicalities in the proofs, we decided to focus on the case of a.s.\ continuity.}\ and that $\{\ell(w, X)\}_{w\in\W}$ is a $\xi$-SG stochastic process. Then we have
$$|\G|\leq \frac{\xi}{\sqrt m}\sum_{k=1}^\infty\ee_{k-1}\sqrt{2I(W_k;S)}\,.$$
\end{manualprop}
\begin{proof}
By standard arguments, $\{\Ell_s(w)\}_{s\in\Sc}$ is a $(\xi/\sqrt m)$-SG process. Hence, $\Ell_S(w)-\Ell_S(w')$ is $(\xi/\sqrt m)$-SG for every $w, w'\in\W$. By Lemma \ref{lemma:dfreg}, $s\mapsto \Ell_s(w)-\Ell_s(w')$ has regularity $\Rr_\df(\xi/\sqrt m)$ wrt $\Pp_S$ ($\forall w\in\W$), with $\df:(\mu, \nu)\mapsto \sqrt{2\KL(\nu\|\mu)}$. Finally, let $g(w, s) = \Ell_\X(w)-\Ell_s(w)$. Clearly $g$ has the same regularity of $\Ell$. It is not hard to show that $\E_{\Pp_S}[\sup_{w\in\W}|g(w, S)|]<+\infty$ (this is a straight consequence of Remark 8.1.5 in \cite{vershynin_2018}). We conclude by Theorem \ref{thm:genchapp} and Jensen's inequality.
\end{proof}
\subsection{Bounds for non-uniform \texorpdfstring{$\df$}{D}-regularity}\label{app:extdf}
As mentioned at the end of Section \ref{sec:fram}, the results given so far are stated under uniform regularity assumptions. The next two results show that this is not strictly necessary, and that slightly different bounds can be obtained relaxing these assumptions. 
\begin{theorem}\label{thm:genstdextdf}
Consider a non-negative measurable function $w\mapsto\xi_w$ such that $\|\xi_W\|_{L^p(\Pp_W)} = \xi$, for some $p\in[1,+\infty]$. Assume that a measurable map $F:\W\times\Z\to\R$ is such that $z\mapsto F(w, z)$ has regularity $\Rr_\df(\xi_w)$ wrt $\Pp_Z$ and for every $w\in\W$. Then we have
$$|\E_{\Pp_{W\otimes  Z}}[F(W, Z)]-\E_{\Pp_{W, Z}}[F(W, Z)]|\leq \xi\,\E_{\Pp_W}[\df(\Pp_Z, \Pp_{Z|W})^r]^{1/r}\,,$$
where $r$ is such that $1/p+1/r=1$ (with the convention $1/\infty=0$).
\end{theorem}
\begin{proof}
The proof is essentially the same as for Theorem \ref{thm:genstdA}, the only difference being that now we have
$$|\E_{\Pp_Z}[F(w, Z)]-\E_{\Pp_{Z|W=w}}[F(w, Z)]|\leq \xi_w\,\df(\Pp_Z, \Pp_{Z|W=w})\,,$$
whose expectation under $\Pp_W$ can be upperbounded via H\"older's inequality
\end{proof}
\begin{theorem}\label{thm:genchextdf}
Let $\W$ be a compact subset of $\R^d$ and $\{\pi_k(\W)\}$ a $\{\ee_k\}$-refining sequence of nets on $\W$. Consider a measurable map $F:\W\times\Z\to\R$, such that $w\mapsto F(w, z)$ is continuous on $\W$, $\Pp_Z$-a.s., and $\E_{\Pp_Z}[\sup_{w\in\W}|F(w, Z)|]<+\infty$. Fix $\xi\geq 0$ and consider a measurable map $w\mapsto\xi_w\geq 0$ such that $\|\xi_{W_k}\|_{L^p(\Pp_W)}\leq \xi$, for all $k\in\Nn$ and for some $p\in[1,+\infty]$. Assume that for every $(w, w')\in\W^2$ the function $z\mapsto F(w, z)-F(w', z)$ has regularity $\Rr_\df(\xi_w\|w-w'\|)$, wrt $\Pp_Z$. Then, we have
$$|\E_{\Pp_{W, Z}}[F(W, Z)]-\E_{\Pp_{W\otimes  Z}}[F(W, Z)]|\leq \xi\sum_{k=1}^\infty \ee_{k-1}\E_{\Pp_W}[\df(\Pp_Z, \Pp_{Z|W_k})^r]^{1/r}\,,$$
where $r$ is such that $1/p+1/r=1$ (with the convention $1/\infty=0$).
\end{theorem}
\begin{proof}
The proof is essentially analogous to the one of Theorem \ref{thm:genchapp}, but instead of \eqref{eq:finqui} now we have
\begin{align*}
    \big|\E_{\Pp_{Z|W_k=w_k}}[F(w_k, Z)-F(w_{k-1}, Z)]-\E_{\Pp_Z}&[F(w_k, Z)-F(w_{k-1}, Z)]\big|\\
    &\leq \xi_{w_k}\ee_{k-1}\df(\Pp_Z, \Pp_{Z|W_k=w_k})\,.
\end{align*}
The conclusion follows easily by H\"older's inequality. 
\end{proof}

\section{PAC-Bayesian bounds}\label{app:PAC}
The next result \citep{catonipreprint} is a classical PAC-Bayesian bound. For the sake of completeness we give here a standard proof.
\begin{manualprop}{\ref{prop:PAC}}
Assume that $\ell$ is bounded in $[-\xi, \xi]$. Let $\Pp^*_W$ be a fixed probability measure on $\W$, chosen independently of $S$. Fix $\delta\in(0,1)$ and $\lambda>0$. Then, with probability $\Pp_S=\Pp_X^{\otimes m}$ larger than $1-\delta$ on the draw of $S$, we have
$$\E_{\Pp_{W|S}}[g_S(W)]\leq \frac{\xi}{\sqrt{2m}}\left(\lambda + \frac{\KL(\Pp_{W|S}\|\Pp^*_W)+\log\frac{1}{\delta}}{\lambda}\right)\,.$$
\end{manualprop}
\begin{proof}
We define $\Pp^*_{W\otimes S}=\Pp^*_W\otimes\Pp_S$. Fix $\lambda>0$. Using the Donsker-Varadhan representation of the $\KL$ divergence \citep{donskerVaradhan}, we have that for all $s\in\Sc$
$$\E_{\Pp_{W|S=s}}[g_s(W)]\leq \frac{\xi}{\sqrt{2m}\lambda}\left(\KL(\Pp_{W|S=s}\|\Pp^*_{W}) + \log\E_{\Pp^*_{W}}[e^{\sqrt{2m}\lambda g_s(W)/\xi}]\right)\,.$$
By Markov's inequality, we have that
$$\Pp_S\left(\E_{\Pp^*_{W}}[e^{\sqrt{2m}\lambda g_S(W)/\xi}]\leq \frac{1}{\delta}\,\E_{\Pp^*_{W\otimes S}}[e^{\sqrt{2m}\lambda g_S(W)/\xi}]\right)\geq 1-\delta\,.$$
Now, for all $w\in\W$ we have that $\ell(w, X)\subset[-\xi,\xi]$ is $\xi$-SG. In particular $g_S(w)$ is $(\xi/\sqrt m)$-SG, as $\Pp_S=\Pp_X^{\otimes m}$. Since $\E_{\Pp_S}[g_S(w)]=0$ we have 
$$\log\E_{\Pp^*_{W\otimes S}}[e^{\sqrt{2m}\lambda g_S(W)/\xi}]\leq \lambda^2\,,$$
from which we conclude.
\end{proof}
Note that although Proposition \ref{prop:PAC} is valid for all $\lambda>0$, we cannot optimise the final bound wrt $\lambda$. Indeed, we have that such a choice of $\lambda$ would depend on $\KL(\Pp_{W|S}, \Pp^*_W)$ and hence on the particular sample used. A possible strategy to overcome this issue consists in selecting a few possible values $\lambda_1,\dots, \lambda_t$ for $\lambda$, before drawing the sample $S$. Then, by mean of a union bound, one can say that with probability $\Pp_S$ higher than $1-t\delta$ the generalisation is bounded by the best PAC-Bayesian bound among the $t$ ones obtained. 
\begin{manualprop}{\ref{prop:chPAC}}
Under the assumptions \ref{ass}, consider a $\{\ee_k\}$-refining sequence of nets on $\W$ and assume that $\nabla_w\ell$ is bounded in $[-\xi, \xi]$. Let $\Pp^*_W$ be a fixed probability measure on $\W$, chosen independently of $S$. Fix two sequences $\{\delta_k\}_{k\in\mathbb N}$ and $\{\lambda_k\}_{k\in\mathbb N}$, such that $\delta_k\in(0,1)$ and $\lambda_k>0$ for all $k$. Assume that $\sum_{k\in\Nn}\delta_k=\delta \in(0,1)$. Then, with probability $\Pp_S=\Pp_X^{\otimes m}$ larger than $1-\delta$ on the draw of $S$, we have
\begin{align*}
    \E_{\Pp_{W|S}}[g_S(W)]\leq\frac{\xi}{\sqrt{2m}}\left( 2\,\sqrt{\log\frac{1}{\delta_0}}+\sum_{k=1}^\infty\ee_{k-1} \left(\lambda_k + \frac{\KL(\Pp_{W_k|S}\|\Pp^*_{W_k})+\log\frac{1}{\delta_k}}{\lambda_k}\right)\right)\,.
\end{align*}
\end{manualprop}
\begin{proof}
By the assumptions in \ref{ass}, $w\mapsto\Ell_s(w)$ is uniformly continuous on $\W$, $\Pp_S$-a.s. In particular, $\sup_{w\in\W}|\Ell_s(w)-\Ell_s(w_k)|\to 0$ as $k\to\infty$, $\Pp_S$-a.s. As a consequence $\sup_{w\in\W}|\E_{\Pp_S}[\Ell_S(w)]-\E_{\Pp_S}[\Ell_S(w_k)]|\to 0$ ($\Pp_S$-a.s.) since the loss is uniformly bounded. It follows that, for $\Pp_S$-almost every $s$,
$$\lim_{k\to\infty}\E_{\Pp_{W|S=s}}[|g_s(W)-g_s(W_k)|] = 0\,.$$
Hence, recalling that $W_0=w_0$, we have that, $\Pp_S$-a.s.,
$$\E_{\Pp_{W|S=s}}[g_s(W)] = g_s(w_0) + \sum_{k=1}^\infty \E_{\Pp_{W_k|S=s}}[g_s(W_k)-g_s(W_{k-1})]\,.$$
On the one hand, by Hoeffding's inequality \citep{Hoeffding}, the first term in the RHS can be upper-bounded with high probability, as
$$\Pp_S\left(g_S(w_0)> \xi\sqrt{\tfrac{2}{m}\log\tfrac{1}{\delta_0}}\right)\leq \delta_0\,.$$
On the other hand, proceeding as in the proof of Proposition \ref{prop:PAC}, for each term in the telescopic sum we can write, for $\Pp_S$-almost every $s$,
\begin{align*}
   \E_{\Pp_{W_k|S=s}}&[g_s(W_k)-g_s(W_{k-1})]\\
   &\leq \frac{\ee_{k-1}\xi}{\sqrt{2m}\lambda_k}\left(\KL(\Pp_{W_k|S=s}\|\Pp^*_{W_k}) + \log\E_{\Pp^*_{W_k}}[e^{\sqrt{2m}\lambda_k(g_s(W_k)-g_s(W_{k-1}))/(\ee_{k-1}\xi)}]\right)\,.
\end{align*}
Now, $\nabla_w\ell\subset[-\xi,\xi]$, and hence $\nabla_w\ell(w, X)$ is $\xi$-SG, for all $w\in\W$. By Lemma \ref{lem:equivSGP}, we have that $\{\ell(w, X)\}_{w\in\W}$ is a $\xi$-SG process. In particular, $\{g_S(w)\}_{w\in\W}$ is a centred $(\xi/\sqrt m)$-SG process, as $\Pp_S = \Pp_X^{\otimes m}$. We have thus obtained that 
$$\log\E_{\Pp^*_{W_k\otimes S}}[e^{\sqrt{2m}\lambda_k(g_S(W_k)-g_S(W_{k-1}))/(\ee_{k-1}\xi)}]\leq \lambda_k^2\,.$$
By Markov's inequality we have that 
$$\Pp_S\left(\E_{\Pp_{W_k|S}}[g_s(W_k)-g_s(W_{k-1})]>\frac{\ee_{k-1}\xi}{\sqrt{2m}} \left(\lambda_k + \frac{\KL(\Pp_{W_k|S}\|\Pp^*_{W_k})+\log\frac{1}{\delta_k}}{\lambda_k}\right)\right)\leq \delta_k\,.$$
We conclude by a union bound. 
\end{proof}
As for the standard PAC-Bayesian result, here as well we cannot directly optimise the parameters $\lambda_k$. Clearly one can again proceed by fixing few possible values for each parameter and then use a union argument to select the best bound. However, in this case this might become particularly hard, due to the large number of parameters. A possible way to address this problem consists in doing some optimisation that does not rely on the value of $\KL(\Pp_{W_k|S}, \Pp^*_{W_k})$, to reduce the number of parameters. For instance, we can proceed in the following way. One might suppose that $\KL(\Pp_{W_k|S}, \Pp^*_{W_k})$ increases linearly with $k$. Note that this is for instance the case if the algorithm is deterministic and $\Pp^*_{W_k}$ is uniform. So, let us say that we believe that $\KL(\Pp_{W_k|S}, \Pp^*_{W_k})=\alpha k$, for some $\alpha>0$. Then we are allowed to optimise all the $\lambda_k$ in the chained PAC-Bayesian bound where $\KL(\Pp_{W_k|S}, \Pp^*_{W_k})$ is replaced by $\alpha k$. This leads to 
\begin{align*}
    \E_{\Pp_{W|S}}[g_S(W)]\leq\frac{\xi}{\sqrt{2m}}\left( 2\,\sqrt{\log\frac{1}{\delta_0}}+\sum_{k=1}^\infty\ee_{k-1} \,\frac{\KL(\Pp_{W_k|S}\|\Pp^*_{W_k})+\log\frac{1}{\delta_k}}{\sqrt{\alpha k+\log\frac{1}{\delta_k}}}\right)\,,
\end{align*}
which is a valid bound, holding with probability higher than $1-\delta$, for all $\alpha>0$. Now we have essentially replaced the $\lambda_k$ with a single parameter $\alpha$. Again we cannot optimise directly wrt $\alpha$, but we can proceed as for the unchained bound, finding a good $\alpha$ by means of a union bound. 

As a final remark note that one might want to optimise in terms of $\delta_k$ as well. This should be possible, but the constraint $\sum_k\delta_k=\delta$ and the non-convexity of the problem can make the minimisation quite hard in practice. Yet, one can probably resort to numerical methods. 
\section{Omitted proofs of Section \ref{sec:comp}}
\begin{manuallemma}{\ref{lemma:chunch}}
Under the assumptions \ref{ass}, let $\ee_0$ and $w_0$ be such that $\|w-w_0\|\leq\ee_0$, $\forall w\in \W$. Assume that $s\mapsto\nabla_w \Ell_s(w)$ has regularity $\Rr_\df(\xi)$ wrt $\Pp_S$, $\forall w\in\W$, and define $\hat \Ell_s(w) = \Ell_s(w)-\Ell_s(w_0)$ and $\hat \G=\E_{W\otimes S}[\hat \Ell_S(W)]-\E_{W, S}[\hat \Ell_S(W)]$. Then, $\hat \G = \G$, and $s\mapsto\hat \Ell_s(w)$ has regularity $\Rr_\df(\ee_0\xi)$, wrt $\Pp_S$ and $\forall w\in\W$.
\end{manuallemma}
\begin{proof}
The fact that $\E_{\Pp_{W, S}}[\Ell_S(w_0)] = \E_{\Pp_{W\otimes S}}[\Ell_S(w_0)]$ implies that $\hat\G=\G$. The regularity of $s\mapsto \hat\Ell_s(w)$ is a direct consequence of Lemma \ref{lemma:superreg}.
\end{proof}
\begin{manuallemma}{\ref{lemma:wassdataprocineq}}
Assume that $\nu\mapsto\df(\Pp_S, \nu)$ is convex. For any $\{\ee_k\}$-refining sequence of nets on $\W$, the sequence $\{\E_{\Pp_W}[\df(\Pp_S, \Pp_{S|W_k})]\}_{k\in\Nn}$ is non-decreasing and, $\forall k\in\Nn$, we have $$\E_{\Pp_W}[\df(\Pp_S, \Pp_{S|W_k})]\leq\E_{\Pp_W}[\df(\Pp_S, \Pp_{S|W})]\,.$$
\end{manuallemma}
\begin{proof}
Fix $k\geq 0$ and $w_k\in\W_k$ such that $\Pp_W(W_k=w_k)>0$. For any measurable set $U$ on $\Sc$ , we have
$$\Pp_{S|W_k=w_k}(U) = \int_{\W}\Pp_{S|W=w}(U)\dd\Pp_{W|W_k=w_k}(w)\,,$$
where $\dd\Pp_{W|W_k=w_k}(w) = \frac{\dd\Pp_W(w)}{\Pp_W(W_k=w_k)}$ if $w\in\pi_k^{-1}(w_k)$, and $0$ otherwise. Hence, we can write 
$$\df(\Pp_S, \Pp_{S|W_k=w_k}) = \df\left(\Pp_S, \int_{\W}\Pp_{S|W=w}(\cdot)\dd\Pp_{W|W_k=w_k}(w)\right)\,.$$
Since $\QQ\mapsto\df(\Pp_S, \QQ)$ is a convex function, we can use Jensen's inequality to obtain
$$\df(\Pp_S, \Pp_{S|W_k=w_k})\leq \int_{\W}\df(\Pp_S, \Pp_{S|W=w})\dd\Pp_{W|W_k=w_k}(w)\,.$$
By taking the expectation wrt $\Pp_{W_k}$ we conclude that 
$$\E_{\Pp_{W}}[\df(\Pp_S, \Pp_{S|W_k})]\leq \E_{\Pp_W}[\df(\Pp_S, \Pp_{S|W})]\,.$$
Now, for any $k'> k$, the same proof can be used to show that
$$\E_{\Pp_{W}}[\df(\Pp_S, \Pp_{S|W_k})]\leq \E_{\Pp_{W}}[\df(\Pp_S, \Pp_{S|W_{k'}})]\,,$$
by simply replacing $\W$ with $\W_{k'}$ and $\Pp_W$ with $\Pp_{W_{k'}}$. 
\end{proof}
\section{Toy Models}\label{app:toy}
\subsection{\nameref{ex:unif}}\label{app:unif}
Let $\W=\X=[-1,1]$, $\ell(w, x)=\frac{1}{2}(w-x)^2$, and $\ee_k=2^{-k}$, for $k\in\Nn$. We can find mappings $\pi_k$ that define a $\{\ee_k\}$-refining sequence of nets, with $\W_k = \{2^{1-k}j\;:\;j\in[-2^{k-1}:2^{k-1}]\}$, where $[a:b]=[a, b]\cap\mathbb Z$. Fix $k^\star\in\Nn$ and define $\theta=2^{-k^\star}$. Let $X$ be uniformly distributed on $(-\theta, \theta)$, that is $X\sim U_{(-\theta, \theta)}$. We choose an algorithm that, given $x$, selects the $w$ minimising $\ell(w, x)$. This means that $\Pp_{W|X=x}=\delta_x$, where $\delta_x$ is the Dirac measure on $x$. Note that $\nabla_w\ell$ is $1$-Lipschitz and $\ell$ is $2$-Lipschitz (on $\X$, uniformly on $\W$). However, thanks to Lemma \ref{lemma:chunch} we know that we can consider the loss $\tilde\ell(w, x) =\ell(w,x)-\tfrac{x^2}{2}$, which does not affect the algorithm, leads to the same generalisation, and is $1$-Lipschitz. The marginal distribution of $W$ is $W\sim U_{(-\theta, \theta)}$. Moreover, we have $\E_{\Pp_{W, X}}[\ell(W, X)]=0$ and $\E_{\Pp_X}[\ell(w, X)] = \frac{1}{2}\left(w^2+\frac{\theta^2}{3}\right)$. So,
$$\G =\E_{\Pp_{W\otimes X}}[\ell(W, X)]-\E_{\Pp_{W, X}}[\ell(W, X)] = \E_{\Pp_W}\left[\frac{1}{2}\left(W^2+\frac{\theta^2}{3}\right)\right] = \frac{\theta^2}{3}\,.$$

Recall that we denote as $\Bl$ the bound in Proposition \ref{prop:Wass} and as $\Bg$ the chained bound from Proposition \ref{prop:Wassch}. We denote as $\Btl$ the unchained bound obtained using $\tilde\ell$ instead of $\ell$. Clearly we have $\Btl=\Bl/2$. We will now evaluate $\Btl$ and $\Bg$. As a starting point, note that the $1$-Wasserstein distance between two uniforms measures, on the intervals $(A, B)$ and $(a, b)\subseteq (A, B)$, is given by
\begin{align*}
    \Wass(U_{(A, B)}, U_{(a, b)}) = \frac{(A-a)^2+(B-b)^2}{2((B-A)-(b-a))}\,.
\end{align*}
Note that choosing $a= b\in[A, B]$ in the RHS above gives the $1$-Wasserstein distance between a uniform distribution and a Dirac measure. Now, let $a=b=w$, $A=-\theta$ and $B=\theta$. We find that
$$\Wass(\Pp_X, \Pp_{X|W=w}) = \frac{\theta}{2}\left(1+\frac{w^2}{\theta^2}\right)\,.$$
It follows that
$$\Btl=\E_{\Pp_W}[\Wass(\Pp_X, \Pp_{X|W=w})] = \frac{2}{3}\,\theta \,.$$
Comparing $\G$ and $\Btl$, we realize that the standard Wasserstein bound becomes loose for small $\theta$.

Now, fix $k\in\Nn$. If $k\leq k^\star$, then $\pi_k(w)=w_0=0$ with probability $1$. In particular, we have that $W_k\indep X$ and hence $\Wass(\Pp_X, \Pp_{X|W_k})=0$. We will hence focus on the case $k>k^\star$. Let $k=k^\star+\kk$. Now, notice that $\pi_k$ defines $2^\kk+1$ intervals in $(-\theta, \theta)$. We will denote them as $I_j$, where $j\in[-2^{k'-1}:2^{k'-1}]$. We have $I_{-2^{\kk-1}} = (-\frac{1}{2^{k^\star}}, -\frac{2^{\kk}-1}{2^k})$ and $I_{2^{\kk-1}} = (\frac{2^\kk-1}{2^k}, \frac{1}{2^{k^\star}})$, while, for $j\in[-2^{\kk-1}+1:2^{\kk-1}-1]$, $I_j = (\tfrac{2j-1}{2^k}, \tfrac{2j+1}{2^k})$. Note that the two outer intervals will have probability $\Pp_W(W\in I_{-2^{\kk-1}})=\Pp_W(W\in I_{-2^{\kk-1}})= 2^{-(\kk+1)}$, while for the inner intervals, we have $\Pp_W(W\in I_j)=2^{-\kk}$. 

Now, for $j\in[-2^{\kk-1}+1:2^{\kk-1}-1]$ we define $a_j = \tfrac{2j-1}{2^k}$ and $b_j = \tfrac{2j+1}{2^k}$. Note that for all these inner intervals we have $b_j-a_j = 2^{1-k}$, $(b_j-a_j)/\theta = 2^{1-\kk}$, $a_j/\theta = (2j-1)/2^\kk$, and $b_j/\theta = (2j+1)/2^\kk$. So, the contribution brought by the inner intervals to $\E_{\Pp_W}[\Wass(\Pp_X, \Pp_{X|W_k})]$ is given by 
\begin{align*}
    E_1=\frac{\theta}{2^\kk}\sum_{j=-(2^{\kk-1}-1)}^{2^{\kk-1}-1} &\frac{\left(1+\frac{2j-1}{2^\kk}\right)^2+\left(-1+\frac{2j+1}{2^\kk}\right)^2}{4(1-2^{-\kk})} \\
    &=\frac{\theta}{6(1-2^{-\kk})}(4 -12\times 2^{-\kk} + 11\times 2^{-2\kk}-3\times 2^{-3\kk}) \,.
\end{align*}
On the other hand, the contribution of the two outer intervals ($j=\pm 2^{\kk-1}$) is given by 
$$E_2=2\times\frac{\theta}{2^{\kk+1}}\frac{1}{2}\frac{\left(2-\frac{1}{2^\kk}\right)^2}{\left(2-\frac{1}{2^\kk}\right)}  = \frac{\theta}{2^{\kk+1}}\left(2-\frac{1}{2^\kk}\right) = \theta\left(2^{-\kk}-\frac{1}{2}\times 2^{-2\kk}\right) \,.$$
We conclude that, for $\kk\geq 1$ and $k=k^\star+\kk$, we have
\begin{align*}
    \E_{\Pp_W}&[\Wass(\Pp_X, \Pp_{X|W_k})] = E_1+E_2\\
    &=\frac{\theta}{6(1-2^{-\kk})}(4 -12\times 2^{-\kk} + 11\times 2^{-2\kk}-3\times 2^{-3\kk}) + \theta\left(2^{-\kk}-\frac{1}{2}\times 2^{-2\kk}\right)\,.
\end{align*}
We can finally compute $\Bg$, as we have
\begin{align*}
    \Bg & = \sum_{k=1}^\infty \frac{1}{2^{k-1}}\E_{W_k}[\Wass(\Pp_X, \Pp_{X|W_k})] \\
        & = \frac{1}{2^{k^\star}}\sum_{k'=1}^\infty\frac{1}{2^{\kk-1}}\E_{W_{k^\star+k'}}[\Wass(\Pp_X, \Pp_{X|W_{k^\star+k'}})]  = \frac{247}{105}\theta^2 \simeq 2.35\,\theta^2 \,.
\end{align*}
Now, it is interesting to compare these results with the CMI bound. For this purpose, we need to compute $I(W_k;X)$ for a fixed $k \in \Nn$. Similar to the chained Wasserstein bound, for $k \leq k^\star$ we have that $I(W_k;X) = 0$ as $W_k \indep X$. Therefore, we focus on $k = k^\star + \kk$ where $\kk \geq 1$. First, notice that the $\KL$ divergence between two uniform measures, on the intervals $(A, B)$ and $(a, b)\subseteq (A, B)$, is given by
\begin{align*}
    \KL(U_{(a, b)} \| U_{(A, B)}) = \log \frac{B-A}{b-a}\,.
\end{align*}
As a consequence, we have that for the inner intervals $I_j$ (with $j\in [-2^{\kk-1}+1:2^{\kk-1}-1]$)
$$\KL(\Pp_{X|W_k\in I_j}\|\Pp_X) = \log(2^{k-k^\star}) = \kk\log 2\,,$$
while for the two outer intervals we have
$$\KL(\Pp_{X|W_k\in I_{-2^{\kk-1}}}\|\Pp_X) = \KL(\Pp_{X|W_k\in I_{-2^{\kk-1}}}\|\Pp_X) = \log(2^{k+1-k^\star}) = (\kk+1)\log 2\,.$$
Taking the expectation wrt $\Pp_W$ we obtain
\begin{align*}
    I(W_k;X) &= \E_{\Pp_W}[\KL(\Pp_{X|W_k}\|\Pp_{X})] = \sum_{j = -2^{\kk-1}}^{2^{\kk-1}}\Pp_W(W_k\in I_j)\KL(\Pp_{X|W_k\in I_j}\|\Pp_X)\\
    &=2\times 2^{-(\kk+1)}(\kk+1)\log 2 + (1-2\times 2^{-(\kk+1)})\kk\log 2 = (\kk + 2^{-\kk})\log2\,.
\end{align*}
Therefore, the CMI bound is given by
\begin{align*}
    \Bcmi &= \sum_{k=1}^\infty \frac{1}{2^{k-1}} \sqrt{2 I(W_k; X)}\\
    & = \frac{1}{2^{k^\star}}\sum_{\kk=1}^\infty\frac{1}{2^{\kk-1}} \sqrt{2(\kk + 2^{-\kk})\log2} \simeq 3.50\,\theta.
\end{align*}
For $\theta\to 0$ (\textit{i.e.}, $k^\star\to\infty$) $\Bg$ is much tighter than $\Bl$ and $\Bcmi$, as it captures the asymptotic behaviour of $\G=\theta^2/3$.

Finally, let us consider the case of a random sample $S=\{X_1,\dots,X_m\}$, for $m>1$. We denote as $\B_{\nabla\Ell}$ the chained Wasserstein bound, and as $\B_\Ell$ the unchained one. Minimising $\Ell_S$ leads to
$$W = \frac{1}{m}\sum_{i=1}^m X_i\,.$$
Since each $X_i$ lies in $(-\theta,\theta)$ with probability $1$, in particular we have that
$$\Pp_W(W\in(-\theta,\theta)) = 1\,.$$
So, for $k\leq k^\star$, $W_k$ is deterministic and hence $S\indep W_k$. We get 
$$\B_{\nabla\Ell} = \frac{1}{\sqrt m}\sum_{k=1}^\infty \frac{1}{2^{k-1}}\E_{\Pp_W}[\Wass(\Pp_S,\Pp_{S|W_k})] = \frac{1}{2^{k^\star}\sqrt m}\sum_{\kk=1}^\infty\frac{1}{2^{\kk-1}}\E_{\Pp_W}[\Wass(\Pp_S,\Pp_{S|W_k})] \leq 2\theta\B_\Ell\,,$$
where we used Lemma \ref{lemma:wassdataprocineq} and the fact that $\theta = 2^{-k^\star}$. We have thus seen that even for large samples we still have that for $\theta\to 0$
$$\frac{\B_{\nabla\Ell}}{\B_\Ell} = O(\theta)\,.$$

\subsubsection{Higher dimensional variant for a generic loss}\label{app:highdim}
We discuss now a higher dimensional version of the above toy model. Fix a positive integer integer $d\geq 1$. Let $\W=\X=[-1, 1]^d$. Fix an integer $k^\star\geq 1$ and define $\theta=2^{-k^\star}$. We will assume that the choice of $k^\star$ scales with $d$ so that $\theta=\Theta(d^{-\alpha})$ for some $\alpha>0$. Let $X$ be uniformly distributed on $R_d=(\theta,\theta)^{d-1}\times(-1,1)$. For $k\in\Nn$ we let $\ee_k = 2^{-k}\sqrt d$ (the rescaling $\sqrt d$ is necessary as now $\W$ has diameter $2\sqrt d$) and we consider a $\{\ee_k\}$-refining sequence of nets $\W_k = {\tilde \W_k}^{\otimes d}$, where $\tilde \W_k = \{2^{1-k}j\;:\;j\in[-2^{k-1}:2^{k-1}]\}$. We consider a generic loss function $\ell$ satisfying the assumptions in \ref{ass}, and such that $\nabla_w\ell$ is $1$-Lipschitz in $\X$, uniformly in $\W$. From Lemma \ref{lemma:chunch} we know that we can find a loss $\tilde\ell$ which is $\sqrt d$-Lipschitz (as $\ee_0=\sqrt d$), and in general we cannot assume the Lipschitz constant to be smaller. As in the 1D example, we assume that we have an algorithm that given $x$, selects $w=x$. This means that $\Pp_{W|X=x}=\delta_x$, where $\delta_x$ is the Dirac measure on $x$, and so the marginal distribution of $W$ is $U_{R_d}$. 

As we are interested in evaluating the Wasserstein bounds, we will need to compute quantities like $\Wass(\Pp_X, \Pp_{X|W=w})$ and $\Wass(\Pp_X, \Pp_{X|W_k=w_k})$. This can be a pretty hard task if we use the standard $2$-norm on $\R^d$ as the distance on $\X$. To give an idea of the challenge, note that already in dimension $d=2$ computing the expected distance between two uniform distributions on rectangles is far from being trivial \citep{marsaglia}. For this reason, everything is much easier to compute if we endow $\X$ with the distance given by the $1$-norm on $\R^d$, that is
$$\hat d_\X(x, x') = \sum_{i=1}^d |x_i-x_i'|\,,$$
where $x_i$ and $x_i'$ are the components of $x$ and $x'$. We will denote the Wasserstein distances computed in this way as $\hat\Wass$, and the bounds based on this distance as $\hat \B$. Note, however, that we always have that $\Wass\leq\hat\Wass$, where $\Wass$ is the Wasserstein distance with cost
$$d_\X(x, x') = \|x-x'\|\,,$$
as $d_\X(x, x')\leq \hat d_\X(x,x')$ for all $x, x'$. Moreover, when $x$ and $x'$ are in $R_d$, we have that $\hat d_\X(x, x')- d_\X(x, x')=O(\theta\sqrt{d-1})$. For this reason, since $\theta=\Theta(d^{-\alpha})$, we obtain that $\Blh-\Bl = O(d^{1-\alpha})$ and $\Bgh-\Bg = O(d^{1-\alpha})$.

Now, for the Wasserstein distatence between $\Pp_X$ and $\Pp_{X|W=w}$, thanks to the fact that we are using $\hat d_\X$, we have
$$\hat\Wass(\Pp_X, \Pp_{X|W = w}) = \sum_{i=1}^d \Wassun(\Pp_{X_i}, \Pp_{X_i|W = w})\,,$$
where $\Wassun$ is the Wasserstein distance wrt the 1D distance $d_{\X_i}(x_i, x'_i)= |x_i-x'_i|$. Taking the expectation wrt $\Pp_W$ we find
$$\Btlh = \sqrt{d}\,\E_{\Pp_W}[\hat\Wass(\Pp_X, \Pp_{X|W})] = \frac{2\sqrt d}{3}\,(1-(d-1)\theta) = \Theta(d^{1/2}+d^{3/2-\alpha})\,.$$
Since $\Btlh-\Btl=O(d^{1-\alpha})$, if follows that
$$\Btl = \Theta(d^{1/2}+d^{3/2-\alpha})\,.$$

We are now left with the task of estimating $\Bg$. Fix $w_k$ such that $\Pp_W(W_k=w_k)>0$. Now, we have that $\Pp_{X|W_k=w_k}$ is the uniform distribution on the rectangle $\pi^{-1}_k(\W)$. Up to sets of measure $0$, we can find $d$ intervals $(a_i, b_i)$ such that
$$\pi^{-1}_k(\W) = (a_1, b_1)\times \dots\times (a_d, b_d)\,.$$
We can choose a transport plan that is composed of $d$ steps. First we squeeze all the probability mass from $\X$ to $(a_1, b_1)\times(-1, 1)^{d-1}$. Then we squeeze the second component, and so on. In this way we find that
$$\hat \Wass(\Pp_X, \Pp_{X|W_k=w_k})\leq \sum_{i=1}^d \Wassun(\Pp_{X_i}, \Pp_{X_i|W_k=w_k})\,.$$
On the other hand, we have that
\begin{align*}
    \hat\Wass(\Pp_X, &\Pp_{X|W_k=w_k}) = \inf_{\pi\in\Pi(\Pp_X, \Pp_{X|W_k=w_k})}\E_{(X,X')\sim\pi}[\hat d_\X(X, X')]\\
    &= \inf_{\pi\in\Pi(\Pp_X, \Pp_{X|W_k=w_k})}\sum_{i=1}^d \E_{(X,X')\sim\pi}[|X_i- X_i'|]\geq \sum_{i=1}^d \Wassun(\Pp_{X_i}, \Pp_{X_i|W_k=w_k})\,.
\end{align*}
We conclude that
$$\hat\Wass(\Pp_X, \Pp_{X|W_k=w_k}) = \sum_{i=1}^d \Wassun(\Pp_{X_i}, \Pp_{X_i|W_k=w_k})\,.$$
We are now back at evaluating Wasserstein distances between uniform distributions on invervals. Proceeding as in the 1D version of the toy example we find
\begin{align*}
    \Bgh & = \sum_{k=1}^\infty \ee_{k-1} \E_{\Pp_W}[\hat\Wass(\Pp_S, \Pp_{S|W_k})] = \frac{247\sqrt d}{105}\,(1+(d-1)\theta^2) = \Theta(d^{1/2}+d^{3/2-2\alpha})\,.
\end{align*}
Again, since $\Bgh - \Bg = O(d^{1-\alpha})$ we have that if $\alpha\geq 1/2$
$$\Bg = \Theta(d^{1/2})\,.$$
In general, as $\alpha$ might be in $(0, 1/2)$, we can say that (since $\Bg\leq\Bgh$)
$$\Bg = O(d^{1/2}+d^{3/2-2\alpha})\,.$$
Now, we want to compare the two bounds. We have
$$\frac{\Bg}{\Btl} = O\left(\frac{1+d^{1-2\alpha}}{1+d^{1-\alpha}}\right).$$
If $\alpha\in(0,1)$, we have that this ratio vanishes for $d\to\infty$, meaning that the chained bounds becomes much tighter than its unchained counterpart. On the other hand, for $\alpha>1$ the ratio is of order $1$.

\subsection{\nameref{ex:gauss}}\label{app:gauss}
Let $\W=\{w\in\R^2:\|w\|=1\}$ and $\X=\R^2$. Fix $a>0$ and let $X\sim\N(\al, \Id)$, a multivariate normal distribution centered in $\al=(a, 0)$, with covariance matrix given by the identity. Let the loss be $\ell(w, x) = -w\cdot x$. As in \nameref{ex:unif}, the algorithm selects the $w$ minimising the loss. In practice, we are trying to find the direction of the mean of $X$, which is $(1, 0)$. Let $\ee_k=4/2^k$ (for $k\in\Nn$), $w_0=(1, 0)$, and $\W_k=\{w=(\cos\frac{2\pi j}{2^k}, \sin\frac{2\pi j}{2^k}\phi): j\in[-2^{k-1}: 2^{k-1}-1]\}$ for $k\geq 1$. We can easily define projections $\pi_k$ that make $\{\W_k\}_{k\in\Nn}$ a $\{\ee_k\}$-sequence of refining nets. With no difficulty one can verify that $\ell$ is $1$-Lipschitz in $\X$, $\forall w\in\W$. Since $\W$ is not convex, we want to use Theorem \ref{thm:genchapp} to give our chaining bound. It is easy to verify that $\ell$ satisfies the $\df$ regularity with $\df=\Wass$, as 
$$|(\ell(w, x)-\ell(w, x'))-(\ell(w', x)-\ell(w', x'))|\leq \|x-x'\|\|w-w'\|\,.$$
Since the values of $\G$, $\Bl$, and $\Bg$ depend on $a$, we will explicitly write them as functions of $a$. We will start by finding the exact expression of $|\G(a)|$. 

Denote as $\aaa$ the Cartesian axis on which $\al$ lies. For $v\in\R^2$, denote as $\alpha(v)$ be the angle between $v$ and $\aaa$. Since the learnt $w$ is parallel to $x$, we have that, with probability $1$, $\alpha(X)=\alpha(W)$. Thus, the distribution of $\alpha(W)$ is the distribution of the angle of an isotropic Gaussian centred in $\al$, whose density is given by \citep{Cooper2020ATF}
$$\rho_a(\alpha) = \frac{\phi(a)}{\sqrt{2\pi}}\left(1+\frac{a\cos\alpha\,\Phi(a\cos\alpha)}{\phi(a\cos\alpha)}\right)\,,$$
where $\phi(t)=\frac{1}{\sqrt 2\pi}e^{-t^2/2}$ and $\Phi(t)=\frac{1}{2}(1+\erf(t/\sqrt 2))$.

Now, we can actually give an explicit form for $|\G(a)|$. Indeed, we have
$$|\G(a)| = a\int_{-\pi}^{\pi]}(1-\cos\alpha)\rho_a(\alpha)\dd\alpha = a - \frac{\phi(a)}{\sqrt{2\pi}}\int_{-\pi}^\pi(a\cos\alpha)^2\frac{\Phi(a\cos\alpha)}{\phi(a\cos\alpha)}\dd \alpha\,.$$
Performing a change of variable we get
\begin{align*}
    \int_{-\pi}^\pi(a\cos\alpha)^2\frac{\Phi(a\cos\alpha)}{\phi(a\cos\alpha)}\dd \alpha = 2\int_{-a}^a&\frac{u^2}{\sqrt{a^2-u^2}}\frac{\Phi(u)}{\phi(u)}\dd u \\
    &= \frac{a^2 e^{a^2/4}\pi^{3/2}}{\sqrt 2}\left(I_0(a^2/a)+I_1(a^2/4)\right)\,,
\end{align*}
where $I_n(t)$ denotes the modified Bessel function of the first kind. So, we have
$$|\G(a)| = a\left(1-\frac{1}{2}\,\frac{I_0(a^2/a)+I_1(a^2/4)}{\sqrt\frac{2}{\pi}\frac{e^{a^2/4}}{a}}\right)\,.$$
We can now use the asymptotic expansions
\begin{align*}
    &I_0(a^2/4) = \sqrt\frac{2}{\pi}\frac{e^{a^2/4}}{a}\left(1 + \frac{1}{2 a^2} + O(a^{-4})\right)\,;\\
    &I_1(a^2/4) = \sqrt\frac{2}{\pi}\frac{e^{a^2/4}}{a}\left(1 - \frac{3}{2 a^2} + O(a^{-4})\right)\,,
\end{align*}
to get that
$$|\G(a)| = \frac{1}{2a} + O(a^{-3})\,.$$

Now, we want to show that, as $a\to\infty$, $\Bl$ is of order $1$. We start by computing a lower bound. For each $w$, let us consider a new set of Cartesian axes ($\uuu(w)$ and $\vvv(w)$), such that the angle between $\vvv(w)$ and $\aaa$ is $\alpha(w)$, and $\uuu(w)$ is the normal axis which contains the point $\al$. We choose the orientation of the axes so that in this reference framework we have $\al=(a\sin\alpha(w), 0)$. Since $X$, conditioned on $\W=w$, has support contained in the axis $\vvv(w)$, the Wasserstein distance $\W(\Pp_X, \Pp_{X|W=w})$ is lower-bounded by the transport cost of moving every point in $\R^2$ to the closest point on $\vvv(w)$. We thus have
\begin{align*}
    \Wass(\Pp_X, \Pp_{X|W=w})&\geq \frac{1}{2\pi}\int_{\R^2}|u|e^{-\frac{(u-a\sin\alpha(w))^2+v^2}{2}}\dd u\dd v \\
    &= \frac{1}{\sqrt{2\pi}}\int_\R |u|e^{-\frac{(u-a\sin\alpha(w))^2}{2}}\dd u \geq a|\sin\alpha(w)|\,.
\end{align*}
We can now explicitly compute a lower bound for $\Bl(a)$ by taking the expectation wrt $\Pp_W$. We get
$$\Bl(a)\geq\int_{-\pi}^\pi a|\sin\theta|\rho_a(\theta)\dd\theta = \sqrt{\frac{2}{\pi}}\erf\frac{a}{\sqrt 2}\,.$$
In particular, we have established that
$$\liminf_{a\to\infty}\Bl(a)\geq \sqrt{\frac{2}{\pi}}\,.$$

We can now look for an upper bound on $\Bl(a)$. Fixed $w$, we can consider the following transport plan from $\Pp_X$ to $\Pp_{X|W=w}$. First, we transport all the probability mass on $\vvv(w)$, then we arrange the mass on $\vvv(w)$ so as to reach the correct density. For the first step, notice that we are simply projecting $\Pp_X$ on $\vvv(w)$. It is not hard to realise that in this way the linear density obtained on $\vvv(w)$ is a centred standard normal distribution. The transport cost for this step is given by
$$\frac{1}{2\pi}\int_{\R^2}|u|e^{-\frac{(u-a\sin\alpha(w))^2+v^2}{2}}\dd u\dd v\leq 1+a|\sin\alpha(w)|\,.$$
Now let $V\sim\N(0, 1)$. The actual distribution of $\Pp_{X|W=w}$ on $\vvv(w)$ is actually given by $V$, conditioned on $V\geq -a\cos\alpha(w)$, as $-a\cos\alpha(w)$ is the coordinate on $\vvv(w)$ of the origin of the standard $\R^2$ Cartesian framework and so $\Pp_{X|W=w}$ has support $\{v\in\vvv(w):v\geq -a\cos\alpha(w)\}$. We can easily evaluate
$$\Wass(\Pp_V,\Pp_{V|V\geq -a\cos\alpha(w)})=\frac{\phi(a\cos\alpha(w))}{\Phi(a\cos\alpha(w))}\,.$$
So we have found that
$$\Wass(\Pp_X, \Pp_{X|W=w})\leq 1+a|\sin\alpha(w)| + \frac{\phi(a\cos\alpha(w))}{\Phi(a\cos\alpha(w))}\,.$$
Averaging on $w$ we get that 
$$\Bl(a) =\E_{\Pp_W}[\Wass(\Pp_X, \Pp_{X|W})]\leq 1 + \sqrt{\frac{2}{\pi}}\erf\frac{a}{\sqrt 2} + \frac{e^{-a^2}}{\Phi(-a)}\,,$$
and so
$$\limsup_{a\to\infty}\Bl(a)\leq 1 + \sqrt{\frac{2}{\pi}}\,.$$
In particular, we have found that $\Bl(a) = \Theta(1)$, for $a\to \infty$. 

We are now left with the task of evaluating $\Bg(a)$. Recall that, for each $k\geq 1$, we have $\W_k=\{w=(\cos\frac{2\pi j}{2^k}, \sin\frac{2\pi j}{2^k}): j\in[-2^{k-1}: 2^{k-1}-1]\}$ and $w_0=(1,0)$. Denote as $\U_k$ the partition on $\W$ induced by $\pi_k$, that is
$$\U_k=\{U=\pi_k^{-1}(w):\,w\in\W_k\}\,.$$
We can certainly suppose that each $U\in\U_k$ is the circular arc enclosed by two adjacent elements of $\W_k$. Now, let $\bar \U_k = \{U\in\U_k:\,(1, 0)\neq U\}$ and define $\theta_k = \pi/2^k$. Then, we have that, up to points of null measure, $\{W_k=(1,0)\}=\{|\alpha(W)|\leq\theta_k\}$. As a consequence
\begin{align}\label{eq:longeqex}
    \begin{split}
    \E_{\Pp_W}[&\Wass(\Pp_X, \Pp_{X|W_k})] = \textstyle\sum_{U\in\U_k}\E_{\Pp_W}[\Wass(\Pp_X, \Pp_{X|W\in U})\one_U(W)]\\ 
    &=\Pp_W(|\alpha(W)|\leq\theta_k)\Wass(\Pp_X, \Pp_{X|W_k=w_0})+\textstyle\sum_{U\in\overline\U_k}\Pp_W(W\in U)\Wass(\Pp_X, \Pp_{X|W\in U})\,,
    \end{split}
\end{align}
where $\one_U$ is the indicator function of the event $U$. We need now to upper-bound the terms of this sum. 

Let us define $Z=X-\al$. Clearly $Z\sim\N(0, \Id)$. Let $\rho$ be the density of $Z$, a centered standard multivariate normal, and $\tilde\rho$ be the density of $Z$ conditioned on $|\alpha(W)|\leq\theta_k$. We have that
$$\tilde\rho(z) = \begin{cases}0&\text{ if $|\alpha|>\theta$;}\\\rho(z)/\Pp_W(|\alpha(W)|\leq\theta_k)&\text{ otherwise.}\end{cases}$$
Let $\zeta=\|Z\|$ and note that $\zeta\sim\chi_2$, the Rayleigh distribution. We notice that $$\Wass(\Pp_X, \Pp_{X|W_k=w_0}) = \Wass(\Pp_Z, \Pp_{Z||\alpha(W)|\leq \theta_k})\,.$$ 
We can upper-bound this quantity by the transport cost of moving the mass $\Pp_W(|\alpha(W)|>\theta_k)$ away from $\{|\alpha(W)|>\theta_k\}$, bringing it all on $\al$, and finally redistributing it in the slice $\{|\alpha(W)|\leq\theta_k\}$, proportionally to $\tilde\rho$. We hence have
$$\Wass(\Pp_Z, \Pp_{Z||\alpha(W)|\leq\theta_k})\leq \Pp_W(|\alpha(W)|>\theta_k)(\Wass(\Pp_Z, \delta_\al)+\Wass(\delta_\al, \Pp_{Z||\alpha(W)|\leq\theta_k}))\,.$$
We can evaluate
$$\Wass(\Pp_Z, \delta_\al)=\int_{\R^2}\|z\|\rho(z)\dd z = \E_\zeta[\zeta] = \sqrt{\frac{\pi}{2}}\,.$$
On the other hand,
\begin{align*}
    \Wass(\delta_\al, \Pp_{Z||\alpha(W)|\leq\theta_k}) &=\int_{\R^2}\|z\|\tilde\rho(z)\dd z \\
    &\leq \frac{1}{\Pp_W(|\alpha(W)|\leq\theta_k)}\int_{\R^2}\|z\|\rho(z)\dd z = \frac{\sqrt{\pi/2}}{\Pp_W(|\alpha(W)|\leq\theta_k)}\,.
\end{align*}
Now notice that
$$\Pp_W(|\alpha(W)|\leq\theta_k)\geq\Pp_\zeta(\zeta\leq a\sin\theta_k) = F_\zeta(a\sin\theta_k)\,,$$
where $F_\zeta:u\mapsto 1-e^{u^2/2}$ is the cdf of $\zeta$.
As a consequence we eventually find
\begin{align*}
    \Wass(\Pp_X, &\Pp_{X|W_k=w_0})\\
    &\leq\Pp_W(|\alpha(W)|>\theta_k)\left(1+\frac{1}{\Pp_W(|\alpha(W)|\leq\theta_k)}\right)\sqrt{\frac{\pi}{2}}\leq \left(1+\frac{1}{F_\zeta(a\sin\theta_k)}\right)\sqrt{\frac{\pi}{2}}\,.
\end{align*}

Now, for $U\in\bar\U_k$, we have that $\{W\in U\}\subseteq \{|\alpha(W)|\geq\theta_k\}$. We can upper-bound $\Wass(\Pp_X, \Pp_{X|W\in U})=\Wass(\Pp_Z, \Pp_{Z|W\in U})$ as
$$\Wass(\Pp_Z, \Pp_{Z|W\in U})\leq\Wass(\Pp_Z, \delta_\al)+\Wass(\delta_\al, \Pp_{Z|W\in U})\,.$$
We have already computed $\Wass(\Pp_Z, \delta_\al)=\sqrt{\pi/2}$. For the other term we have
\begin{align*}
    \Wass(\delta_\al, \Pp_{Z|W\in U})&=\frac{1}{\Pp_W(W\in U)}\int_{(z+\al)/\|z+\al\|\in U}\|z\|\rho(z)\dd z\\
    &\leq \frac{1}{\Pp_W(W\in U)}\int_{\|z\|>a\sin\theta_k}\|z\|\rho(z)\dd z = \frac{1-F_\zeta(a\sin\theta_k)}{\Pp(W\in U)}\,.
\end{align*}
We have thus obtained that
$$\Wass(\Pp_X, \Pp_{X|W\in U})\leq \sqrt{\frac{\pi}{2}}+\frac{1-F_\zeta(a\sin\theta_k)}{\Pp(W\in U)}\,.$$

Going back to \eqref{eq:longeqex}, we can now write
\begin{equation}\label{eq:finally}\E_{\Pp_W}[\Wass(\Pp_X, \Pp_{X|W_k})]\leq(1-F_\zeta(a\sin\theta_k))\left(\left(2 +\frac{1}{F_\zeta(a\sin\theta_k)}\right)\sqrt{\frac{\pi}{2}} + 2^k-1\right)\,,\end{equation}
where we used that $\overline\U_k$ has $2^k-1$ elements and that $\sum_{u\in\bar\U_k}\Pp_W(W\in U)\leq (1-F_\zeta(a\sin\theta_k))$. Now, by plugging into \eqref{eq:finally} the explicit expressions of $F_\zeta$ and $\theta_k$ we obtain
$$\E_{\Pp_W}[\Wass(\Pp_X, \Pp_{X|W_k})] \leq e^{-\frac{1}{2}a^2\sin^2(\pi/2^k)}\left(2^k-1+\left(2+\frac{1}{1-e^{-\frac{1}{2}a^2\sin^2(\pi/2^k)}}\right)\sqrt\frac{\pi}{2}\right) = \B_k(a)\,.$$

Fix $k^\star>1$. By Lemma \ref{lemma:wassdataprocineq}, we have that for all $k\leq k^\star$, $\E_{\Pp_W}[\Wass(\Pp_X, \Pp_{X|W_k})]\leq \B_{k^\star}(a)$, and for $k> k^\star$ we have $\E_{\Pp_W}[\Wass(\Pp_X, \Pp_{X|W_k})]\leq \Bl(a)$. So we have that
$$\Bg(a)\leq \sum_{k=1}^{k^\star}\ee_{k-1}\B_{k^\star}(a) + \sum_{k=k^\star}^\infty\ee_{k}\Bl(a) \leq 8\,\B_{k^\star}(a)+4\times 2^{-k^\star}\Bl(a)\,.$$

Now the idea is that we want to choose $k^\star= k^\star_a$ as a function of $a$, in a way that makes the bound vanish for $a\to+\infty$. Note that if 
\begin{equation}\label{eq:kastar}
    a\geq \frac{2\log 2\sqrt{k^\star_a}}{\sin(\pi/2^{k^\star_a})}\,,
\end{equation}
then
$$\B_{k^\star_a}(a) \leq 2^{-k^\star_a} + 2^{-2k^\star_a}\left(2+\frac{1}{1-2^{-2k^\star_a}}\right)\sqrt{\frac{\pi}{2}}\,.$$
Notice we can choose $a\mapsto k^\star_a$ such that \eqref{eq:kastar} holds and $a = O(2^{-k^\star_a}\sqrt{k^\star_a})$, for $a\to+\infty$, which implies
$$2^{-k^\star_a}= O\left(\frac{\log a -\log\log a}{a}\right)\,.$$
This proves the asymptotic behaviour for large $a$
$$\Bg(a)=O\left(\frac{\log a -\log\log a}{a}\right)\,.$$
In particular, up to logarithmic factors, the chained bound can capture the correct behaviour of $\G(a)$.

\section{Technicalities}\label{app:tech}
\begin{lemma}
The mapping $w\mapsto \Pp_{Z|W=w}$ is measurable.
\end{lemma}
\begin{proof} Recall that $\Sigma_{\PR}$ is the $\sigma$-algebra on $\PR$ induced by the weak topology. $\Sigma_\PR$ is generated by the maps $\phi_U:\PR\to[0,1]$, given by $\mu\mapsto \phi_U(\mu)=\mu(U)$, for $U$ ranging in $\Sigma_\Z$ (cf.\ Theorem 17.24 in \cite{kechris95}). By definition of regular conditional probability, for every $U\in\Sigma_\Z$ the map $w\mapsto\Pp_{Z|W=w}(U)$ is measurable. Hence $w\mapsto \Pp_{Z|W=w}$ is a measurable map $\W\to\PR$ wrt $\Sigma_\PR$.
\end{proof}
\begin{definition}
Let $f:(0,+\infty)\to\R$ be a convex lower semi-continuous map such that $f(1)=0$ and $\lim_{x\to +\infty}f(x)/x=+\infty$. For $\PP, \QQ\in\PR$ we define the $f$-divergence 
$$D_f(\QQ\|\PP)=\begin{cases}\E_{\PP}[f(\tfrac{\dd\QQ}{\dd\PP})]&\text{if $\QQ\ll\PP$;}\\+\infty&\text{otherwise.}\end{cases}$$
\end{definition}
Examples of $f$ divergences are the $\KL$ divergence ($f:u\mapsto u\log u$) and the $p$-power divergence ($f:u\mapsto u^p-1$).
\begin{lemma}\label{lemma:klmeas}
$\df:\PR\times\PR\to[0,+\infty]$, defined by $\df(\mu, \nu) = D_f(\nu\|\mu)$, is measurable.
\end{lemma}
\begin{proof}
The measurability follows from the fact $D_f$ is weakly lower semi-continuous (see Corollary 2.9 and Remark 2.1 in \cite{beppe<3}). 
\end{proof}
\begin{lemma}\label{lemma:wassmeas}
$\df:\PR\times\PR\to[0,+\infty]$, defined by $\df(\mu, \nu) = \Wass(\mu,\nu)$, is measurable.
\end{lemma}
\begin{proof}
The measurability follows from the weak lower semi-continuity of $\Wass$ (see \cite{villani08}, Remark 6.12).
\end{proof}
\begin{lemma}\label{lemma:supp}
$\Supp(\Pp_{Z|W=w}) \subseteq \Supp(\Pp_Z)$, $\Pp_W$-a.s.
\end{lemma}
\begin{proof}
We start by recalling that given a measure $\mu\in\PR$, $\Supp(\mu)$ is the smallest closed subset $K$ of $\Z$ such that $\mu(K)=1$. Let $U\subseteq\W$ be defined as 
$$U=\{w\in\W: \Pp_{Z|W=w}(\Supp(\Pp_Z))<1\}\,.$$
First, we notice that $U$ is measurable. Indeed, $\Supp(\Pp_Z)$ is closed, and hence measurable, so $w\mapsto \Pp_{Z|W=w}(\Supp(\Pp_Z))$ is a measurable map, by definition of regular conditional probability.
Now note that
\begin{align*}
    1=\Pp_Z(\Supp(\Pp_Z)) &= \int_\W\Pp_{Z|W=w}(\Supp(\Pp_Z))\,\dd\Pp_W(w)\\
    &\leq 1 - \Pp_W(U)+\int_U\Pp_{Z|W=w}(\Supp(\Pp_Z))\,\dd\Pp_W(w)\,.
\end{align*}
As a consequence, we must have that $\int_U\Pp_{Z|W=w}(\Supp(\Pp_Z))\,\dd\Pp_W(w)\geq\Pp_W(U)$. However, by definition $\Pp_{Z|W=w}(\Supp(\Pp_Z))<1$ for $w\in U$, and so we necessarily have $\Pp_W(U)=0$. We conclude by noticing that $\Supp(\Pp_{Z|W=w})\supset\Supp(\Pp_Z)$ if, and only if, $w\in U$. 
\end{proof}
\section{Explicit bounds}\label{app:bounds}
In this section we present several bounds that can be derived via the framework of Section \ref{sec:fram}. To our knowledge, all the chaining bounds that we present here are new, the only exception being the one in Proposition \ref{prop:MIdis}, which was recently established in \cite{zhou2022stochastic}. However, most of the unchained counterparts were already derived in the literature. The reader can find the bibliographic references in Table \ref{table:boundsapp}. Henceforth, all the chained bounds that we state are valid for any $\{\ee_k\}$-sequence of refining nets on $\W$.
\subsection{A few examples of \texorpdfstring{$\df$}{D}-regularity}
\begin{definition}[Power divergence]
Let $p>1$. Given two probabilities $\PP$ and $\QQ$ on $\Z$, we define the $p$-power divergence 
$$D^{(p)}(\QQ\|\PP) = \begin{cases}\E_\PP\left[\left(\tfrac{\dd\QQ}{\dd\PP}\right)^p\right]-1&\text{if $\QQ\ll\PP$;}\\+\infty&\text{otherwise.}\end{cases}$$
For $p=2$, we denote $D^{(2)}(\QQ\|\PP)$ as $\chi^2(\QQ\|\PP)$.  
\end{definition}
\begin{lemma}\label{lemma:power}
Fix $p>1$ and let $r = p/(p-1)$. Let $\df:(\mu, \nu)\mapsto (D^{(p)}(\nu\|\mu)+1)^{1/p}$ and $f:\Z\to\R^q$ be measurable. Assume that $f\in L^1(\mu)$ and write $f_\mu=\E_\mu[f(Z)]$. If $\E_\mu[\|f(Z)-f_\mu\|^r]^{1/r}\leq\xi$, then $f$ has regularity $\Rr_\df(\xi)$ wrt $\mu$. 
\end{lemma}
\begin{proof}
Notice that $\df$ is measurable by Lemma \ref{lemma:klmeas}. First, we consider the case $q=1$. Fix $\nu\in\PR$ such that $\Supp(\nu)\subseteq\Supp(\mu)$ and $f\in L^1(\nu)$. If $\nu$ is not abslutely continuous wrt $\mu$, than the claim is trivially true, so assume $\nu\ll\mu$. Define $f_\mu = \E_\mu[f(Z)]$. We have
\begin{align*}
    |\E_\mu[f(Z)]-\E_\nu[f(Z)]|\leq\E_{\nu}[|f(Z)-f_\mu|]&=\int_\Z |f(z)-f_\mu|\,\tfrac{\dd\nu}{\dd\mu}(z)\dd\mu(z)\\
    &\leq \E_\mu[|f(Z)-f_\mu|^r]^{1/r}(D^{(p)}(\nu\|\mu)+1)^{1/p}\,,
\end{align*}
by H\"older's inequality.

The case $q>1$ follows form the one-dimesional case, since $\E_\mu[|(f(Z)-f_\mu)\cdot v|^r]^{1/r}\leq \|v\|\E_\mu[\|(f(Z)-f_\mu)\|^r]^{1/r}$ for all $v\in\R^q$.
\end{proof}
\begin{corollary}\label{cor:power}
Fix $p>1$ and let $r = p/(p-1)$. Let $\df:(\mu, \nu)\mapsto (D^{(p)}(\nu\|\mu)+1)^{1/p}$ and $f:\Z\to\R^q$ be measurable. Assume that $f(Z)$ is $\xi$-SG for $Z\sim\mu$. Then $f$ has regularity $\Rr_\df(e^{1/e}\sqrt r\,\xi)$ wrt $\mu$. 
\end{corollary}
\begin{proof}
Simply use that $\E_\mu[\|f(Z)-\E_{Z'\sim\mu}[f(Z')]\|^r]^{1/r}\leq e^{1/e}\sqrt r\,\xi$ if $f(Z)$ is $\xi$-SG to conclude by Lemma \ref{lemma:power}. 
\end{proof}
\begin{lemma}\label{lemma:chi2}
Let $\df:(\mu, \nu)\mapsto \sqrt{\chi^2(\nu\|\mu)}$. Let $f:\Z\to\R^q$ be measurable. Assume that $\|\mathbb{C}_\mu[f(Z)]\|\leq\xi^2$, where $\mathbb C_\mu[f(Z)]$ is the covariance matrix of $f(Z)$ for $Z\sim\mu$. Then, $f$ has regularity $\Rr_{\df}(\xi)$. 
\end{lemma}
\begin{proof}
For $q=1$, the claim is a direct consequence of the HCR bound \citep{LehmCase98}. The case $q>1$ follows easily. 
\end{proof}
\begin{definition}[Total variation]
The total variation of two probability measures $\PP, \QQ\in\PR$ is defined as
$$\TV(\PP,\QQ) = \sup_{U\in\Sigma_\Z}|\PP(U)-\QQ(U)|\,.$$
\end{definition}
\begin{lemma}\label{lem:TV}
Let $\df:(\mu, \nu)\mapsto 2\TV(\mu, \nu)$. Let $f:\Z\to\R^q$ be a measurable map, bounded in $[-\xi,\xi]$. Then $f$ has regularity $\Rr_\df(\xi)$ wrt any $\mu\in\PR$. 
\end{lemma}
\begin{proof}
First, we need to show that $\nu\mapsto\TV(\mu, \nu)$ is measurable. We have that for all $U\in\Sigma_\Z$, the map $\nu\mapsto |\mu(U)-\nu(U)|$ is continuous in the weak topology. In particular, taking the supremum wrt $U$ we get a weakly lower semicontinuous map, which implies the measurability. Now, notice that asking $f\subseteq[-\xi, \xi]$ is equivalent to ask for $f$ to be $2\xi$-Lipschitz wrt the discrete metric on $\Z$. We can then proceed as in Lemma \ref{lemma:dfreg} using the fact that the total variation coincides with the $1$-Wasserstein distance when the transport cost is the discrete metric \citep{villani08}. 
\end{proof}
\begin{corollary}\label{cor:revkl}
Let $\df:(\mu, \nu)\mapsto \sqrt{2\KL(\mu\|\nu)}$. Let $f:\Z\to\R^q$ be a measurable map, bounded in $[-\xi,\xi]$. Then $f$ has regularity $\Rr_\df(\xi)$. 
\end{corollary}
\begin{proof}
The measurability of $\df$ is a obvious consequence of Lemma \ref{lemma:klmeas}. Then, the claim follows directly from Lemma \ref{lem:TV} by Pinsker's inequality; see e.g. \cite{vanHandel}.
\end{proof}
\subsection{Some simple bounds based on the \texorpdfstring{$\df$}{D}-regularity}
\begin{definition}[Power information]
Consider two coupled random variables $Z, Z'$ on $(\Z, \Sigma_\Z)$. For $p>1$ we define their $p$-power information \citep{powerinf} as 
$$I^{(p)}(Z;Z') = D^{(p)}(\Pp_{Z, Z'}\|\Pp_{Z\otimes Z'})\,.$$
\end{definition}
\begin{proposition}\label{prop:power}
Fix $p>1$, let $r=p/(p-1)$ and suppose that $\Pp_S=\Pp_X^{\otimes m}$. On the one hand, if $\ell(w, X)$ is $\xi$-SG for $X\sim\Pp_X$, for all $w\in\W$, then
$$|\G|\leq \frac{e^{1/e}\sqrt r\,\xi}{\sqrt m}\,(I^{(p)}(S;W)+1)^{1/p}\,.$$
On the other hand, under the assumptions \ref{ass} if $\nabla_w(\ell, X)$ is $\xi$-SG for $X\sim\Pp_X$, for all $w\in\W$, then
$$|\G|\leq \frac{e^{1/e}\sqrt r\,\xi}{\sqrt m}\sum_{k=1}^\infty\ee_{k-1}(I^{(p)}(S;W_k)+1)^{1/p}\,.$$
\end{proposition}
\begin{proof}
First notice that the $\xi$-subgaussianity of $\ell$ (respectively $\nabla_w\ell$) implies that of $\Ell$ (respectively $\nabla_w\Ell$) is $(\xi/\sqrt m)$-SG. Then, the first claim follows by Corollary \ref{cor:power}, Theorem \ref{thm:genstd}, and Jensen's inequality, while the second one by Corollary \ref{cor:power}, Theorem \ref{thm:gench}, and Jensen's inequality.
\end{proof}
\begin{proposition}\label{prop:chi2}
Suppose that $\Pp_S=\Pp_X^{\otimes m}$. On the one hand, if $\V_{\Pp_X}[\ell(w, X)]\leq\xi^2$, for all $w\in\W$, then
$$|\G|\leq \frac{\xi}{\sqrt m}\,\E_{\Pp_W}\left[\sqrt{\chi^2(\Pp_{S|W}\|\Pp_S)}\right]\,.$$
On the other hand, under the assumptions \ref{ass} if $\|\mathbb C_{\Pp_X}[\nabla_w\ell(w, X)]\|\leq\xi^2$, for all $w\in\W$, then
$$|\G|\leq \frac{\xi}{\sqrt m}\sum_{k=1}^\infty\ee_{k-1}\E_{\Pp_W}\left[\sqrt{\chi^2(\Pp_{S|W_k}\|\Pp_S)}\right]\,.$$
\end{proposition}
\begin{proof}
The claims follow combining Lemma \ref{lemma:chi2} with Theorems \ref{thm:genstd} and \ref{thm:gench}. Note that the variance of $\Ell$ is re-scaled by a factor $1/\sqrt m$ wrt the one of $\ell$, as $\Pp_S=\Pp_X^{\otimes m}$. The same is true for the covariance of $\nabla_w\Ell$.
\end{proof}
\subsection{Individual-sample bounds}\label{sec:appdis}
Recall that $S=\{X_1,\dots, X_m\}$. In this section we will consider a probability measure $\Pp_S$ on $(\Sc,\Sigma_\Sc)$ such that the marginals $\Pp_{X_i}=\Pp_X$ for all $i\in[1:m]$, but we do not require that the draws are independent. Note moreover that $W$ might depend in a different way on each $X_i$, so that we can have that $\Pp_{W, X_i}\neq\Pp_{W, X_j}$, for $i\neq j$. Now, we specialise Theorems \ref{thm:genstd} and \ref{thm:gench} to obtain individual-sample bounds, such as those from \cite{Bu2019}. 
\begin{proposition}\label{prop:gendis}
Assume that $x\mapsto\ell(w, x)$ has regularity $\Rr_\df(\xi)$ wrt $\Pp_{X}$, $\forall w\in\W$. Then we have
$$|\G|\leq \frac{\xi}{m}\sum_{i=1}^m\,\E_{\Pp_W}[\df(\Pp_{X}, \Pp_{X_i|W})]\,.$$
\end{proposition}
\begin{proof}
Just write
$$\G = \frac{1}{m}\sum_{i=1}^m(\E_{\Pp_{W\otimes X}}[\ell(W, X)]-\E_{\Pp_{W, X_i}}[\ell(W, X_i)])\,.$$
and then conclude by applying Theorem \ref{thm:genstdA} to bound each term of the sum.
\end{proof}
\begin{proposition}\label{prop:gendisch}
Assume \ref{ass} and suppose that $x\mapsto\ell(w, x)$ has regularity $\Rr_\df(\xi)$ wrt $\Pp_{X}$, $\forall w\in\W$. Then we have
$$|\G|\leq \frac{\xi}{m}\sum_{i=1}^m\sum_{k=1}^\infty\ee_{k-1}\E_{\Pp_W}[\df(\Pp_{X}, \Pp_{X_i|W_k})]\,.$$
\end{proposition}
\begin{proof}
Just write
$$\G = \frac{1}{m}\sum_{i=1}^m(\E_{\Pp_{W\otimes X}}[\ell(W, X)]-\E_{\Pp_{W, X_i}}[\ell(W, X_i)])\,.$$
and then conclude by applying Theorem \ref{thm:genchA} to bound each term of the sum.
\end{proof}
We can now state several individual-sample generalisation bounds. For the sake of brevity, we will omit the proofs, as they are all direct applications of Propositions \ref{prop:gendis} and \ref{prop:gendisch}, and of the previously established results of $\df$-regularity.
\begin{proposition}\label{prop:MIdis}
On the one hand, if $\ell(w, X)$ is $\xi$-SG uniformly on $\W$, then
$$|\G|\leq \frac{\xi}{m}\sum_{i=1}^m \sqrt{2I(W; X_i)}\,.$$
On the other hand, if $\nabla_w\ell(w, X)$ is $\xi$-SG uniformly on $\W$, then
$$|\G|\leq \frac{\xi}{m}\sum_{i=1}^m\sum_{k=1}^\infty\ee_{k-1}\sqrt{2I(W_k;X_i)}\,.$$
\end{proposition}
\begin{proposition}\label{prop:Wassdis}
On the one hand, if $x\mapsto\ell(w, x)$ is $\xi$-Lipschitz uniformly on $\W$, then
$$|\G|\leq \frac{\xi}{m}\sum_{i=1}^mE_{\Pp_W}[\Wass(\Pp_{X}, \Pp_{X_i|W})]\,.$$
On the other hand, assume \ref{ass}. if $x\mapsto\nabla_w\ell(w, x)$ is $\xi$-Lipschitz uniformly on $\W$, then
$$|\G|\leq \frac{\xi}{m}\sum_{i=1}^m\sum_{k=1}^\infty\ee_{k-1}E_{\Pp_{W}}[\Wass(\Pp_{X}, \Pp_{X_i|W_k})]\,.$$
\end{proposition}
\begin{proposition}\label{prop:powerdis}
Fix $p>1$ and let $r=p/(p-1)$. Write $\bar\ell(w)$ for $\E_{\Pp_{X}}[\ell(w, X)]$ and $\overline{\nabla_w\ell}(w)$ for $\E_{\Pp_{X}}[\nabla_w\ell(w, X)]$. On the one hand, if, for all $w\in\W$, $\E_{\Pp_{X}}[|\ell(w, X)-\bar\ell(w)|^r]\leq \xi^r$, then
$$|\G|\leq \frac{\xi}{m}\sum_{i=1}^m\,(I^{(p)}(W; X_i)+1)^{1/p}\,.$$
On the other hand, assume \ref{ass}. If $\E_{\Pp_{X}}[\|\nabla_w\ell(w, X)-\overline{\nabla_w\ell}(w)\|^r]\leq\xi^r$, for all $w\in\W$, then
$$|\G|\leq \frac{\xi}{m}\sum_{i=1}^m\sum_{k=1}^\infty\ee_{k-1}(I^{(p)}(W_k;X_i)+1)^{1/p}\,.$$
\end{proposition}
\begin{proposition}\label{prop:chi2dis}
On the one hand, if, for all $w\in\W$, $\V_{\Pp_{X}}[\ell(w, X)]\leq \xi^2$, then
$$|\G|\leq \frac{\xi}{m}\sum_{i=1}^m \,\E_{\Pp_W}\left[\sqrt{\chi^2(\Pp_{X_i|W}\|\Pp_{X})}\right]\,.$$
On the other hand, assume \ref{ass}. If $\|\mathbb C_{\Pp_{X}}[\nabla_w\ell(w, X)]\|\leq\xi^2$, for all $w\in\W$, then
$$|\G|\leq \frac{\xi}{m}\sum_{i=1}^m\sum_{k=1}^\infty\ee_{k-1}\E_{\Pp_W}\left[\sqrt{\chi^2(\Pp_{X_i|W_k}\|\Pp_{X})}\right]\,.$$
\end{proposition}
\begin{proposition}\label{prop:TVdis}
On the one hand, if $|\ell(w, x)|\leq \xi$ for all $w\in\W$ and all $x\in\X$, then
$$|\G|\leq \frac{2\xi}{m}\sum_{i=1}^m\,\E_{\Pp_W}\left[\TV(\Pp_{X}, \Pp_{X_i|W})\right]\,.$$
On the other hand, assume \ref{ass}. If $\|\nabla_w\ell(w, x)\|\leq \xi$ for all $w\in\W$ and all $x\in\X$, then
$$|\G|\leq \frac{2\xi}{m}\sum_{i=1}^m\sum_{k=1}^\infty\ee_{k-1}\E_{\Pp_W}\left[\TV(\Pp_{X}, \Pp_{X_i|W_k})\right]\,.$$
\end{proposition}
\begin{definition}[Lautum information]
Consider two coupled random variables $Z, Z'$ on $(\Z, \Sigma_\Z)$. We define their lautum information \citep{lautuminf} as
$$\LI(Z;Z') = \KL(\Pp_{Z\otimes Z'}\|\Pp_{Z, Z'})\,.$$
\end{definition}
\begin{proposition}\label{prop:LIdis}
On the one hand, if $|\ell(w, x)|\leq \xi$ for all $w\in\W$ and all $x\in\X$, then
$$|\G|\leq \frac{\xi}{m}\sum_{i=1}^m\sqrt{2\LI(W;X_i)}\,.$$
On the other hand, assume \ref{ass}. If $\|\nabla_w\ell(w, x)\|\leq \xi$ for all $w\in\W$ and all $x\in\X$, then
$$|\G|\leq \frac{\xi}{m}\sum_{i=1}^m\sum_{k=1}^\infty\ee_{k-1}\sqrt{2\LI(W_k;X_i)}\,.$$
\end{proposition}

\subsection{Bounds based on random sub-sampling from a super-sample}
We can derive in our framework bounds in the same spirit of the conditional MI bound from \cite{steinke2020}.

Let $s^\star=( x^\star_1,\dots,  x^\star_m)$ denote a $(2m)$-sample, made of $m$ pairs $x^\star_i = (x_{i, 0}, x_{i, 1})$. The training sample is in the form $s=(x_1,\dots, x_m)$. The choice of $s$, given $s^\star$ is determined by a variable $u\in\{0,1\}^n$, in the sense that $x_i= x^\star_{i, u_i}$, where $u_i$ determine which one of the two components of $ x^\star_i$ is chosen as $x_i$. In practice we can write $s = s^\star_u$, with $u\in\{0, 1\}^n$. We let $\bar u=1-u$ (the difference being component-wise), and $\bar s=s^\star_{\bar u}$. We denote as $ S^\star$ the random super-sample and we assume that each $ X^\star_i\in S^\star$ has marginal distribution $\Pp_{ X^\star} = \Pp_X^{\otimes 2}$. Morover, we let $\Pp_{\bar U}=\Pp_{U}\sim \mathrm{Bernoulli}(\frac{1}{2})^{\otimes m}$, and we assume that $U\indep  S^\star$. Note that this implies that if the super-sample is made of independent pairs ($\Pp_{ S^\star} = \Pp_{ X^\star}^{\otimes m}$) then all the $X_i\in S$ are independent. 
\begin{proposition}\label{prop:genrand}
Let $\Pp_{ S^\star} = \Pp_{ X^\star}^{\otimes m}$. Assume that $s\mapsto \Ell_s(w)$ has regularity $\Rr_\df(\xi)$, wrt $\Pp_{S| S^\star=s^\star}$, for $\Pp_{ S^\star}$-almost every $s^\star$ and $\forall w\in\W$. Then, we have that
$$|\G|\leq \xi\,\E_{\Pp_{ W, S^\star}}[\df(\Pp_{S|W,  S^\star}, \Pp_{S| S^\star})+\df(\Pp_{\bar S|W,  S^\star}, \Pp_{\bar S| S^\star})]\,.$$
\end{proposition}
\begin{proof}
Let $\hat g(w, s^\star, u)=\Ell_{s^\star_{\bar u}}(w)-\Ell_{s^\star_u}(w)$. Now, recalling that $S =  S^\star_U$ and $\bar S =  S^\star_{\bar U}$, we have that $\Pp_{S| S^\star}$ is the law
of $ S^\star_U$ and $\Pp_{\bar S| S^\star}$ is the law
of $ S^\star_{\bar U}$, both under $\Pp_{U}$ and given $ S^\star$. Since $\Pp_{ S^\star}=\Pp_{ X^\star}^{\otimes m}=\Pp_X^{\otimes 2m}$, then $S\indep\bar S$. In particular $\bar{S}\indep W$, and hence $\Pp_{\bar S|W} = \Pp_{\bar S} = \Pp_S$, so that
$$\E_{\Pp_{W,  S^\star, U}}[\Ell_{ S^\star_{\bar U}}(W)]=\E_{\Pp_{W, \bar S}}[\Ell_{\bar S}(W)] = \E_{\Pp_{W\otimes S}}[\Ell_S(W)]\,.$$
It follows that $\G = \E_{\Pp_{W,  S^\star, U}}[\hat g(W,  S^\star, U)]$. Moreover, it is shown in \cite{borja2021tighter} (cf.\ proof of Theorem 3 therein) that
\begin{align*}
    \E_{\Pp_{W,  S^\star, U}}[\hat g(W,  S^\star, U)] = \E_{\Pp_{ S^\star}}\big[\E_{\Pp_{W\otimes U| S^\star}}&[\Ell_{ S^\star_U}(W)]-\E_{\Pp_{W, U| S^\star}}[\Ell_{ S^\star_U}(W)]]\\
    &-\E_{\Pp_{ S^\star}}[\E_{\Pp_{W, U| S^\star}}[\Ell_{ S^\star_{\bar U}}(W)]-\E_{\Pp_{W\otimes U| S^\star}}[\Ell_{ S^\star_{\bar U}}(W)]\big]\,.
\end{align*}
We hence have
\begin{align*}
    |\G|\leq \E_{\Pp_{ S^\star}}\Big[\big|\E_{\Pp_{W\otimes U| S^\star}}[\Ell_{ S^\star_U}(W)]&-\E_{\Pp_{W, U| S^\star}}[\Ell_{ S^\star_U}(W)]\big| \\
    &+ \big|\E_{\Pp_{W\otimes U| S^\star}}[\Ell_{ S^\star_{\bar U}}(W)]-\E_{\Pp_{W, U| S^\star}}[\Ell_{ S^\star_{\bar U}}(W)]\big|\Big]\,,
\end{align*}
which can be rewritten as
\begin{equation}\label{eq:supsam}
    |\G|\leq \E_{\Pp_{ W, S^\star}}\big[|\E_{\Pp_{S| S^\star}}[\Ell_S(W)]-\E_{\Pp_{S| W, S^\star}}[\Ell_S(W)]|+|\E_{\Pp_{\bar S| S^\star}}[\Ell_{\bar S}(W)]-\E_{\Pp_{\bar S| W, S^\star}}[\Ell_{\bar S}(W)]|\big]\,.
\end{equation}
Now, notice that, since $\Pp_U=\Pp_{\bar U}$, we have $\Pp_{S| S^\star}=\Pp_{\bar S| S^\star}$. In particular, $s\mapsto \Ell_s(w)$ has regularity $\Rr_\df(\xi)$ wrt $\Pp_{\bar S| S^\star=s^\star}$ as well ($\forall w\in\W$ and $\Pp_{ S^\star}$-a.s.). From \eqref{eq:supsam} and Theorem \ref{thm:genstd}, we have that
$$|\G|\leq \xi\,\E_{\Pp_{ W, S^\star}}[\df(\Pp_{S|W,  S^\star}, \Pp_{S| S^\star})+\df(\Pp_{\bar S|W,  S^\star}, \Pp_{\bar S| S^\star})]\,,$$
as requested.
\end{proof}
\begin{proposition}\label{prop:genrandch}
Let $\Pp_{ S^\star} = \Pp_{ X^\star}^{\otimes m}$. Assume \ref{ass} and suppose that $s\mapsto \nabla_w\Ell_s(w)$ has regularity $\Rr_\df(\xi)$, wrt $\Pp_{S| S^\star=s^\star}$, for $\Pp_{ S^\star}$-almost every $s^\star$ and $\forall w\in\W$. Then, we have that
$$|\G|\leq \xi\sum_{k=1}^\infty\ee_{k-1}\E_{\Pp_{ W, S^\star}}[\df(\Pp_{S|W_k,  S^\star}, \Pp_{S| S^\star})+\df(\Pp_{\bar S|W_k,  S^\star}, \Pp_{\bar S| S^\star})]\,.$$
\end{proposition}
\begin{proof}
We proceed just as in the proof on Proposition \ref{prop:genrand} until the last step, where we use Theorem \ref{thm:gench}, instead of Theorem \ref{thm:genstd}, to conclude.
\end{proof}
We give now some explicit example of bounds that can be obtained via the above two propositions. 
\begin{definition}[Conditional mutual information, power information, and lautum information]
Let $(Z, Z', W)$ be a random variable on $(\Z\times\Z\times\W, \Sigma_\Z\otimes\Sigma_\Z\otimes\Sigma_W)$. We define the conditional MI \citep{WYNERCMI} as 
$$I(Z; Z'|W) = \E_{\Pp_W}[\KL(\Pp_{Z, Z'|W}\|\Pp_{Z\otimes Z'|W})]\,.$$
For $p>1$, we define the conditional $p$-power information as
$$I^{(p)}(Z; Z'|W) = \E_{\Pp_W}[D^{(p)}(\Pp_{Z, Z'|W}\|\Pp_{Z\otimes Z'|W})]\,.$$
Finally, we define the conditional Lautum information \citep{lautuminf} as
$$\LI(Z;Z|W) = \E_{\Pp_W}[\KL(\Pp_{Z\otimes Z'|W}\|\Pp_{Z, Z'|W})]\,.$$
\end{definition}
\begin{proposition}\label{prop:CMI}
Let $\Pp_{ S^\star} = \Pp_{ X^\star}^{\otimes m}$. On the one hand, assume that $|\ell(w, x)|\leq \xi$, for all $w\in\W$ and all $x\in X$. Then, we have that
$$|\G|\leq 2\xi\,\sqrt{\frac{2I(W;S| S^\star)}{m}}\,.$$
On the other hand, assume \ref{ass} and suppose that $\|\nabla_w\ell(w, x)\|\leq \xi$, for all $w\in\W$ and all $x\in\X$. Then we have
$$|\G|\leq 2\xi\sum_{k=1}^\infty\ee_{k-1}\sqrt{\frac{2I(W_k;S| S^\star)}{m}}\,.$$
\end{proposition}
\begin{proof}
Assume that $|\ell|\leq \xi$. Note that $\ell(w, X)$ is $\xi$-SG, for all $w\in\W$, for $X\sim\Pp_{X| X^\star= x^\star}$ (for all $ x^\star$). As the elements of $S$ are independent (even when conditioning on $ S^\star$ since $U\indep  S^\star$), we have that, $\forall w\in\W$ and $\forall s^\star\in \Sc^2$, $\Ell_S(w)$ is $(\xi/\sqrt m)$-SG for $S\sim\Pp_{S| S^\star=s^\star}$. We can then conclude by Lemma \ref{lemma:dfreg} and Proposition \ref{prop:genrand}, using the fact that $I(W;S| S^\star) = I(W;\bar S| S^\star)$, as $\bar s$ is fully determined by $s$ (given $s^\star$). The proof for the chained bound is analogous.
\end{proof}
The proofs for the next propositions are essentially analogous of the one of Proposition \ref{prop:CMI} and hence are omitted.
\begin{proposition}\label{prop:Wassrand}
Let $\Pp_{ S^\star} = \Pp_{ X^\star}^{\otimes m}$ and assume that $d_\X$ and $d_\Sc$ are related by \eqref{eq:defmetric}. On the one hand, suppose that $x\mapsto \ell(w, x)$ is $\xi$-Lipschitz, for all $w\in\W$. Then, we have that
$$|\G|\leq \frac{\xi}{\sqrt{m}}\,\E_{\Pp_{W,S^\star}}[\Wass(\Pp_{S|S^\star}, \Pp_{S|W, S^\star}) + \Wass(\Pp_{\bar S|S^\star}, \Pp_{\bar S| W, S^\star})]\,.$$
On the other hand, assume \ref{ass} and suppose that $x\mapsto\nabla_w\ell(w, x)\|\leq \xi$, for all $w\in\W$ and all $x\in\X$. Then we have
$$|\G|\leq \frac{\xi}{\sqrt{m}}\sum_{k=1}^\infty\ee_{k-1}\E_{\Pp_{W,S^\star}}[\Wass(\Pp_{S|S^\star}, \Pp_{S|W_k, S^\star}) + \Wass(\Pp_{\bar S|S^\star}, \Pp_{\bar S| W_k, S^\star})]\,.$$
\end{proposition}
\begin{proposition}\label{prop:powerrand}
Fix $p>1$, let $r=p/(p-1)$ and suppose that $\Pp_{ S^\star} = \Pp_{ X^\star}^{\otimes m}$. On the one hand, assume that $|\ell(w, x)|\leq \xi$, for all $w\in\W$ and all $x\in X$. Then, we have that
$$|\G|\leq \frac{2e^{1/e}\sqrt r\,\xi}{\sqrt m}\,(I^{(p)}(W;S|S^\star)+1)^{1/p}\,.$$
On the other hand, assume \ref{ass} and suppose that $\|\nabla_w\ell(w, x)\|\leq \xi$, for all $w\in\W$ and all $x\in\X$. Then we have
$$|\G|\leq \frac{2e^{1/e}\sqrt r\,\xi}{\sqrt m}\sum_{k=1}^\infty\ee_{k-1}(I^{(p)}(W_k;S|S^\star)+1)^{1/p}\,.$$
\end{proposition}
\begin{proposition}\label{prop:chi2rand}
Suppose that $\Pp_{ S^\star} = \Pp_{ X^\star}^{\otimes m}$. On the one hand, if $|\ell(w, x)|\leq \xi$, for all $w\in\W$ and all $x\in X$, then
$$|\G|\leq \frac{2\xi}{\sqrt m}\,\E_{\Pp_W, S^\star}\left[\sqrt{\chi^2(\Pp_{S|W, S^\star}\|\Pp_{S|S^\star})} + \sqrt{\chi^2(\Pp_{\bar S|W, S^\star}\|\Pp_{\bar S|S^\star})}\right]\,.$$
On the other hand, under the assumptions \ref{ass} if $\|\nabla_w\ell(w, x)\|\leq \xi$, for all $w\in\W$ and all $x\in\X$, then
$$|\G|\leq \frac{2\xi}{\sqrt m}\sum_{k=1}^\infty\ee_{k-1}\E_{\Pp_{W, S^\star}}\left[\sqrt{\chi^2(\Pp_{S|W_k, S^\star}\|\Pp_{S|S^\star})}+\sqrt{\chi^2(\Pp_{\bar S|W_k, S^\star}\|\Pp_{\bar S|S^\star})}\right]\,.$$
\end{proposition}
One issue with this random sub-sampling approach is that in order to controll $\Ell_s$ wrt $\Pp_{S| S^\star=s^\star}$, almost uniformly in $s^\star$, one needs essentially to control the random binary variables $\ell(w,  X^\star)$ under $\Pp_{X| X^\star=( x^\star_0,  x^\star_1)}$ (that is $ X^\star =  x^\star_0$ with probability $1/2$, and $ x^\star_1$ with probability $1/2$). This can be easily done in the case of the Wasserstein distance, as the Lipschitzianity guarantees $\Wass$-regularity wrt any measure. However for the subgaussianity things are more complicated, and one essentially needs to ask that $\ell$ is bounded. 

It is however possible to restate Proposition \ref{prop:genrand} (and Proposition \ref{prop:genrandch}) without asking that the same regularity holds $\Pp_{ S^\star}$-a.s. The proof of both results follow closely the ones of Propositions \ref{prop:genrand} and \ref{prop:genrandch}, the only difference being a final application of H\"older's inequality.
\begin{proposition}\label{prop:genrandts}
Let $\Pp_{ S^\star} = \Pp_{ X^\star}^{\otimes m}$. Let $p\in[1,+\infty]$ and $r=p/(p-1)$ (with the convention that $1/0=+\infty$). Assume that $s\mapsto \Ell_s(w)$ has regularity $\Rr_\df(\xi_{s^\star})$, wrt $\Pp_{S| S^\star=s^\star}$, for $\Pp_{ S^\star}$-almost every $s^\star$ and $\forall w\in\W$, where $\|\xi_{ S^\star}\|_{L^p(\Pp_{ S^\star})}=\xi$. Then, we have that
$$|\G|\leq \xi\,\E_{\Pp_{W,  S^\star}}[|\df(\Pp_{S|W,  S^\star}, \Pp_{S| S^\star})+\df(\Pp_{\bar S|W,  S^\star}, \Pp_{\bar S| S^\star})|^r]^{1/r}\,.$$
\end{proposition}
\begin{proposition}\label{prop:genrandchts}
Let $\Pp_{ S^\star} = \Pp_{ X^\star}^{\otimes m}$. Let $p\in[1,+\infty]$ and $r=p/(p-1)$ (with the convention that $1/0=+\infty$). Assume \ref{ass} and suppose that $s\mapsto \nabla_w\Ell_s(w)$ has regularity $\Rr_\df(\xi_{s^\star})$, wrt $\Pp_{S| S^\star=s^\star}$, for $\Pp_{ S^\star}$-almost every $s^\star$ and $\forall w\in\W$, where $\|\xi_{ S^\star}\|_{L^p(\Pp_{ S^\star})}=\xi$. Then, we have that
$$|\G|\leq \xi\sum_{k=1}^\infty\ee_{k-1}\E_{\Pp_{W, S^\star}}[|\df(\Pp_{S|W_k,  S^\star}, \Pp_{S| S^\star})+\df(\Pp_{\bar S|W_k,  S^\star}, \Pp_{\bar S| S^\star})|^r]^{1/r}\,.$$
\end{proposition}
\subsection{Individual-sample bounds based on random sub-sampling}
We can merge together the ideas of the last two sections.
\begin{proposition}\label{prop:genranddis}
Assume that $x\mapsto \ell(w, x)$ has regularity $\Rr_\df(\xi)$, wrt $\Pp_{X| X^\star=x^\star}$, for $\Pp_{X^\star}$-almost every $x^\star$ and $\forall w\in\W$. Then, we have that
$$|\G|\leq \frac{\xi}{m}\sum_{i=1}^m\,\E_{\Pp_{W, X_i^\star}}[\df(\Pp_{X_i|W,  X_i^\star}, \Pp_{X_i| X_i^\star})+\df(\Pp_{\bar X_i|W,  X_i^\star}, \Pp_{\bar X_i| X_i^\star})]\,.$$
\end{proposition}
\begin{proof}
Note that $\Pp_{X|X^\star=x^\star}=\Pp_{X_i|X_i^\star=x^\star}$. Proceeding as in the proof of Proposition \ref{prop:genrand}, we can show that, for $i\in[1:m]$,
\begin{align*}
    |\E_{\Pp_{W\otimes X_i}}[\ell(W, X_i)] &- \E_{\Pp_{W, X_i}}[\ell(W, X_i)]|\\
    &\leq \E_{\Pp_{W, X_i^\star}}[\df(\Pp_{X_i|W,  X_i^\star}, \Pp_{X_i| X_i^\star})+\df(\Pp_{\bar X_i|W,  X_i^\star}, \Pp_{\bar X_i| X_i^\star})]\,.
\end{align*}
We can immediately conclude by writing $\G$ as in the proof of Proposition \ref{prop:gendis}.
\end{proof}
\begin{proposition}\label{prop:genranddisch}
Assume \ref{ass} and suppose that $x\mapsto \nabla_w\ell(w,x)$ has regularity $\Rr_\df(\xi)$, wrt $\Pp_{X| X^\star=x^\star}$, for $\Pp_{ X^\star}$-almost every $x^\star$ and $\forall w\in\W$. Then, we have that
$$|\G|\leq \frac{\xi}{m}\sum_{i=1}^m\sum_{k=1}^\infty\ee_{k-1}\E_{\Pp_{X_i^\star, W}}[\df(\Pp_{X_i|W_k,  X_i^\star}, \Pp_{X_i| X_i^\star})+\df(\Pp_{\bar X_i|W_k,  X_i^\star}, \Pp_{\bar X_i| X_i^\star})]\,.$$
\end{proposition}
\begin{proof}
We proceed as for proving Proposition \ref{prop:genranddis}, but following the proof Proposition \ref{prop:genrandch} instead of \ref{prop:genrand}.
\end{proof}
Clearly one can generalise the two results above by using the same observations as in Propositions \ref{prop:genrandts} and \ref{prop:genrandchts}.

We can now restate all the individual-sample bounds from Section \ref{sec:appdis} in the random sub-sampling framework. 
\begin{proposition}\label{prop:MIranddis}
On the one hand, if $|\ell(w, x)|\leq\xi$, uniformly on $\W$ and $\X$, then
$$|\G|\leq \frac{2\xi}{m}\sum_{i=1}^m \sqrt{2I(W; X_i|X_i^\star)}\,.$$
On the other hand, if $|\nabla_w\ell(w, x)|\leq\xi$, uniformly on $\W$ and $\X$, then
$$|\G|\leq \frac{2\xi}{m}\sum_{i=1}^m\sum_{k=1}^\infty\ee_{k-1}\sqrt{2I(W_k;X_i|X_i^\star)}\,.$$
\end{proposition}
\begin{proposition}\label{prop:Wassranddis}
On the one hand, if $x\mapsto\ell(w, x)$ is $\xi$-Lipschitz uniformly on $\W$, then
$$|\G|\leq \frac{\xi}{m}\sum_{i=1}^m\E_{\Pp_W, X_i^\star}[\Wass(\Pp_{X_i|X_i^\star}, \Pp_{X_i|W, X_i^\star})+\Wass(\Pp_{\bar X_i|X_i^\star}, \Pp_{\bar X_i|W, X_i^\star})]\,.$$
On the other hand, assume \ref{ass}. if $x\mapsto\nabla_w\ell(w, x)$ is $\xi$-Lipschitz uniformly on $\W$, then
$$|\G|\leq \frac{\xi}{m}\sum_{i=1}^m\sum_{k=1}^\infty \ee_{k-1}\E_{\Pp_W, X_i^\star}[\Wass(\Pp_{X_i|X_i^\star}, \Pp_{X_i|W_k, X_i^\star})+\Wass(\Pp_{\bar X_i|X_i^\star}, \Pp_{\bar X_i|W_k, X_i^\star})]\,.$$
\end{proposition}
\begin{proposition}\label{prop:powerranddis}
Fix $p>1$ and let $r=p/(p-1)$. On the one hand, if $|\ell(w, x)|\leq\xi$, uniformly on $\W$ and $\X$, then
$$|\G|\leq \frac{2\xi}{m}\sum_{i=1}^m\,(I^{(p)}(W; X_i|X_i^\star)+1)^{1/p}\,.$$
On the other hand, assume \ref{ass}. If $|\nabla_w\ell(w, x)|\leq\xi$, uniformly on $\W$ and $\X$, then
$$|\G|\leq \frac{2\xi}{m}\sum_{i=1}^m\sum_{k=1}^\infty\ee_{k-1}(I^{(p)}(W_k;X_i|X_i^\star)+1)^{1/p}\,.$$
\end{proposition}
\begin{proposition}\label{prop:chi2randdis}
On the one hand, if $|\ell(w, x)|\leq\xi$, uniformly on $\W$ and $\X$, then
$$|\G|\leq \frac{\xi}{m}\sum_{i=1}^m \,\E_{\Pp_{W, X_i^\star}}\left[\sqrt{\chi^2(\Pp_{X_i|W, X_i^\star}\|\Pp_{X_i|X_i^\star})}+\sqrt{\chi^2(\Pp_{\bar X_i|W, X_i^\star}\|\Pp_{\bar X_i|X_i^\star})}\right]\,.$$
On the other hand, assume \ref{ass}. If $|\nabla_w\ell(w, x)|\leq\xi$, uniformly on $\W$ and $\X$, then
$$|\G|\leq \frac{\xi}{m}\sum_{i=1}^m\sum_{k=1}^\infty\ee_{k-1}\E_{\Pp_{W, X_i^\star}}\left[\sqrt{\chi^2(\Pp_{X_i|W_k, X_i^\star}\|\Pp_{X_i|X_i^\star})}+\sqrt{\chi^2(\Pp_{\bar X_i|W_k, X_i^\star}\|\Pp_{\bar X_i|X_i^\star})}\right]\,.$$
\end{proposition}
\begin{proposition}\label{prop:TVranddis}
On the one hand, if $|\ell(w, x)|\leq\xi$, uniformly on $\W$ and $\X$, then
$$|\G|\leq \frac{2\xi}{m}\sum_{i=1}^m\,\E_{\Pp_{W, X_i^\star}}\left[\TV(\Pp_{X_i|X_i^\star}, \Pp_{X_i|W, X_i^\star})+\TV(\Pp_{\bar X_i|X_i^\star}, \Pp_{\bar X_i|W, X_i^\star})\right]\,.$$
On the other hand, assume \ref{ass}. If $|\nabla_w\ell(w, x)|\leq\xi$, uniformly on $\W$ and $\X$, then
$$|\G|\leq \frac{2\xi}{m}\sum_{i=1}^m\sum_{k=1}^\infty\ee_{k-1}\left[\TV(\Pp_{X_i|X_i^\star}, \Pp_{X_i|W_k, X_i^\star})+\TV(\Pp_{\bar X_i|X_i^\star}, \Pp_{\bar X_i|W_k, X_i^\star})\right]\,.$$
\end{proposition}
\begin{proposition}\label{prop:LIranddis}
On the one hand, if $|\ell(w, x)|\leq\xi$, uniformly on $\W$ and $\X$, then
$$|\G|\leq \frac{2\xi}{m}\sum_{i=1}^m\sqrt{2\LI(W;X_i|X_i^\star)}\,.$$
On the other hand, assume \ref{ass}. If $|\nabla_w\ell(w, x)|\leq\xi$, uniformly on $\W$ and $\X$, then
$$|\G|\leq \frac{2\xi}{m}\sum_{i=1}^m\sum_{k=1}^\infty\ee_{k-1}\sqrt{2\LI(W_k;X_i|X_i^\star)}\,.$$
\end{proposition}
\subsection{Summary table}
Several explicit bounds that can be derived within our general framework of Section \ref{sec:fram} are reported in Table \ref{table:boundsapp}. The first column states the regularity condition required on the loss. However, we refer to the corresponding propositions for the detailed assumptions of each bound. All bounds are stated for $\xi=1$. The last columns give the literature references for each bound, to the best of our knowledge. However, this bibliography should be taken as a mere guideline, as there might possibly be missing references. Those bounds that we could not find in the literature are marked as ``New''. 
\newpage
\begin{table}[t]
\centering
\caption{Some bounds that can be derived with the framework from Section \ref{sec:fram}}
\label{table:boundsapp}
\resizebox{\linewidth}{!}{\begin{tabular}{lccc}
    \\
    \\
    \toprule[1pt]
    Assumption ($\forall w\in\W$) & Bound & Prop & Ref\\
    
    \midrule\midrule
    
    $\ell(w, X)$ $1$-SG & $\sqrt{2 I(W;S)/m}$ & \ref{prop:MI} & \cite{russo2019much} \\
    $\nabla_w\ell(w, X)$ $1$-SG & $\sum_k\varepsilon_{k-1}\sqrt{2 I(W_k;S)/m}$ & \ref{prop:MIchain} & \cite{asadi2018chaining}\\
    
    \midrule 
    
    $\ell(w, \cdot)$ $1$-Lipschitz& $\E_{\Pp_W}[\Wass(\Pp_S, \Pp_{S|W})]/\sqrt m$ & \ref{prop:Wass} & \cite{lopez2018wass}\\
    $\nabla_w\ell(w, \cdot)$ $1$-Lipschitz& $\sum_k\varepsilon_{k-1}\E_{\Pp_W}[\Wass(\Pp_S, \Pp_{S|W_k})]/\sqrt m$ & \ref{prop:Wassch} & New\\
    
    \midrule
    
    $\ell(w, X)$ $1$-SG & $e^{1/e}\sqrt{p}(I^{(p)}(W;S)+1)^{1/p}/\sqrt{m(p-1)}$ & \ref{prop:power} & \cite{aminian2021informationtheoretic} \\
    $\nabla_w\ell(w, X)$ $1$-SG & $e^{1/e}\sqrt{p}\sum_k\varepsilon_{k-1}(I^{(p)}(W_k;S)+1)^{1/p}/\sqrt{m(p-1)}$ & \ref{prop:power} & New\\
    
    \midrule 
    
    $\V_{\Pp_X}[\ell(w, X)]\leq 1$ & $\E_{\Pp_W}[\chi^2(\Pp_{S|W}\|\Pp_S)^{1/2}]/\sqrt m$ & \ref{prop:chi2} & \cite{borja2021tighter} \\
    $\|\mathbb C_{\Pp_X}[\nabla_w\ell(w, X)]\|\leq 1$ & $\sum_k\varepsilon_{k-1}\E_{\Pp_W}[\chi^2(\Pp_{S|W}\|\Pp_S)^{1/2}]/\sqrt m$ & \ref{prop:chi2} & New\\
    
    \midrule
    \midrule
    
    $\ell(w, X)$ $1$-SG & $\sum_i\sqrt{2 I(W;X_i)}/m$ & \ref{prop:MIdis} & \cite{Bu2019} \\
    $\nabla_w\ell(w, X)$ $1$-SG & $\sum_i\sum_k\varepsilon_{k-1}\sqrt{2 I(W_k;X_i)}/m$ & \ref{prop:MIdis} & \cite{zhou2022stochastic}\\
    
    \midrule 
    
    $\ell(w, \cdot)$ $1$-Lipschitz& $\sum_i\E_{\Pp_W}[\Wass(\Pp_X, \Pp_{X_i|W})]/m$ & \ref{prop:Wassdis} & \cite{borja2021tighter}\\
    $\nabla_w\ell(w, \cdot)$ $1$-Lipschitz& $\sum_i\sum_k\varepsilon_{k-1}\E_{\Pp_W}[\Wass(\Pp_X, \Pp_{X_i|W_k})]/ m$ & \ref{prop:Wassdis} & New\\

    \midrule

    $\E_{\Pp_X}[|\ell(w, X)-\bar\ell(w)|^{p/(p-1)}]\leq1$ & $\sum_i(I^{(p)}(W; X_i)+1)^{1/p}/m$ & \ref{prop:powerdis} & New\\
    $\E_{\Pp_X}[\|\nabla_w\ell(w, X_i)-\overline{\nabla_w\ell}(w)\|^{p/(p-1)}]\leq1$ & $\sum_i\sum_k\varepsilon_{k-1}(I^{(p)}(W_k; X_i)+1)^{1/p}/ m$ & \ref{prop:powerdis} & New\\
    
    \midrule

    $\V_{\Pp_X}[\ell(w, X)]\leq 1$ & $\sum_i\E_{\Pp_W}[\chi^2(\Pp_{X_i|W}\|\Pp_X)^{1/2}]/m$ & \ref{prop:chi2dis} & New\\
    $\|\mathbb C_{\Pp_X}[\nabla_w\ell(w, X)]\|\leq1$ & $\sum_i\sum_k\varepsilon_{k-1}\E_{\Pp_W}[\chi^2(\Pp_{X_i|W_k}\|\Pp_X)^{1/2}]/ m$ & \ref{prop:chi2dis} & New\\

    \midrule

    $|\ell|\leq 1$ & $\sum_i\E_{\Pp_W}[\TV(\Pp_X, \Pp_{X_i|W})]/m$ & \ref{prop:TVdis} & \cite{borja2021tighter}\\
    $\|\nabla_w\ell\|\leq 1$ & $\sum_i\sum_k\varepsilon_{k-1}\E_{\Pp_W}[\TV(\Pp_X, \Pp_{X_i|W_k})]/m$ & \ref{prop:TVdis} & New\\

    \midrule

    $|\ell|\leq 1$ & $\sum_i\sqrt{2 \LI(W;X_i)}/m$ & \ref{prop:LIdis} & \cite{borja2021tighter}\\
    $\|\nabla_w\ell\|\leq 1$ & $\sum_i\sum_k\varepsilon_{k-1}\sqrt{2 \LI(W_k;X_i)}/m$ & \ref{prop:LIdis} & New\\
    \midrule
    \midrule
    
    $|\ell|\leq 1$ & $2\sqrt{2 I(W;S|S^\star)/m}$ & \ref{prop:CMI} & \cite{steinke2020} \\
    $\|\nabla_w\ell\|\leq 1$ & $2\sum_k\varepsilon_{k-1}\sqrt{2 I(W_k;S|S^\star)/m}$ & \ref{prop:CMI} & New\\
    
    \midrule 
    
    $\ell(w, \cdot)$ $1$-Lipschitz& $\E_{\Pp_{W,S^\star}}[\Wass(\Pp_{S|S^\star}, \Pp_{S|W, S^\star}) + \dots\tablefootnote{Here and in the following, ``$\,\dots$'' should be read as: ``Take the same expression on the left and replace $\Pp_{S|W, S^\star}$ with $\Pp_{\bar S|W, S^\star}$ (or $\Pp_{X_i|W, X_i^\star}$ with $\Pp_{\bar X_i|W, X_i^\star}$).''.}]/\sqrt m$ & \ref{prop:Wassrand} & \cite{borja2021tighter}\\
    $\nabla_w\ell(w, \cdot)$ $1$-Lipschitz& $\sum_k\varepsilon_{k-1}\E_{\Pp_{W,S^\star}}[\Wass(\Pp_{S|S^\star}, \Pp_{S|W_k, S^\star}) + \dots]/\sqrt m$ & \ref{prop:Wassrand} & New\\
    
    \midrule
    
    $|\ell|\leq 1$ & $2e^{1/e}\sqrt{p}(I^{(p)}(W;S|S^\star)+1)^{1/p}/\sqrt{m(p-1)}$ & \ref{prop:powerrand} & New \\
    $\|\nabla_w\ell\|\leq 1$ & $2e^{1/e}\sqrt{p}\sum_k\varepsilon_{k-1}(I^{(p)}(W_k;S|S^\star)+1)^{1/p}/\sqrt{m(p-1)}$ & \ref{prop:powerrand} & New\\
    
    \midrule 
    
    $|\ell|\leq 1$ & $2\E_{\Pp_{W, S^\star}}[\chi^2(\Pp_{S|W, S^\star}\|\Pp_{S|S^\star})^{1/2}+\dots]/\sqrt m$ & \ref{prop:chi2rand} & New \\
    $\|\nabla_w\ell\|\leq 1$ & $2\sum_k\varepsilon_{k-1}\E_{\Pp_{W,S^\star}}[\chi^2(\Pp_{S|W, S^\star}\|\Pp_{S|S^\star})^{1/2}+\dots]/\sqrt m$ & \ref{prop:chi2rand} & New\\
    
    \midrule 
    \midrule
    
    $|\ell|\leq 1$ & $2\sum_i\sqrt{2 I(W;X_i|X_i^\star)}/m$ & \ref{prop:MIranddis} & \cite{haghifam2020} \\
    $\|\nabla_w\ell\|\leq 1$ & $2\sum_i\sum_k\varepsilon_{k-1}\sqrt{2 I(W_k;X_i|X_i^\star)}/m$ & \ref{prop:MIranddis} & New\\
    
    \midrule 
    
    $\ell(w, \cdot)$ $1$-Lipschitz& $\sum_i\E_{\Pp_{W, X_i^\star}}[\Wass(\Pp_{X_i|X_i^\star}, \Pp_{X_i|W, X_i^\star})+\dots]/m$ & \ref{prop:Wassranddis} & \cite{borja2021tighter}\\
    $\nabla_w\ell(w, \cdot)$ $1$-Lipschitz& $\sum_i\sum_k\varepsilon_{k-1}\E_{\Pp_{W, X_i^\star}}[\Wass(\Pp_{X_i|X_i^\star}, \Pp_{X_i|W_k, X_i^\star})+\dots]/ m$ & \ref{prop:Wassranddis} & New\\

    \midrule

    $|\ell|\leq 1$ & $2\sum_i(I^{(p)}(W; X_i|X_i^\star)+1)^{1/p}/m$ & \ref{prop:powerranddis} & New\\
    $\|\nabla_w\ell\|\leq 1$ & $2\sum_i\sum_k\varepsilon_{k-1}(I^{(p)}(W_k; X_i|X_i^\star)+1)^{1/p}/ m$ & \ref{prop:powerranddis} & New\\
    
    \midrule

    $|\ell|\leq 1$ & $\sum_i\E_{\Pp_{W, X_i^\star}}[\chi^2(\Pp_{X_i|W, X_i^\star}\|\Pp_{X_i|X_i^\star})^{1/2}+\dots]/m$ & \ref{prop:chi2randdis} & New\\
    $\|\nabla_w\ell\|\leq 1$ & $\sum_i\sum_k\varepsilon_{k-1}\E_{\Pp_{W, X_i^\star}}[\chi^2(\Pp_{X_i|W_k, X_i^\star}\|\Pp_{X_i|X_i^\star})^{1/2}+\dots]/ m$ & \ref{prop:chi2randdis} & New\\

    \midrule

    $|\ell|\leq 1$ & 2$\sum_i\E_{\Pp_{W, X_i^\star}}[\TV(\Pp_{X_i|X_i^\star}, \Pp_{X_i|W, X_i^\star})+\dots]/m$ & \ref{prop:TVranddis} & \cite{borja2021tighter}\\
    $\|\nabla_w\ell\|\leq 1$ & 2$\sum_i\sum_k\varepsilon_{k-1}\E_{\Pp_{W, X_i^\star}}[\TV(\Pp_{X_i|X_i^\star}, \Pp_{X_i|W_k, X_i^\star})+\dots]/m$ & \ref{prop:TVranddis} & New\\

    \midrule

    $|\ell|\leq 1$ & $2\sum_i\sqrt{2 \LI(W;X_i|X_i^\star)}/m$ & \ref{prop:LIranddis} & New\\
    $\|\nabla_w\ell\|\leq 1$ & $2\sum_i\sum_k\varepsilon_{k-1}\sqrt{2 \LI(W_k;X_i|X_i^\star)}/m$ & \ref{prop:LIranddis} & New\\
    \bottomrule[1pt]
\end{tabular}}
\end{table}
\end{document}